\documentclass[final,12pt]{colt2025} 

\title[A Polynomial-time Algorithm for Online Sparse Linear Regression]
{A Polynomial-time Algorithm for Online Sparse Linear Regression \\with Improved Regret Bound under Weaker Conditions}
\usepackage{times}
\usepackage{enumerate}
\usepackage{color}
\usepackage{bm}
\usepackage{booktabs} 
\usepackage{multirow}
\usepackage{makecell}

\newtheorem{Mylemma}{Lemma}
\newtheorem{MyCoro}{Corollary}
\newtheorem{MyRemark}{Remark}
\newtheorem{assumption}{Assumption}
\newtheorem*{question}{Question}

\coltauthor{%
 \Name{Junfan Li} \Email{lijunfan@hit.edu.cn}\\
 \addr Harbin Institute of Technology (Shenzhen)
 \AND
 \Name{Shizhong Liao} \Email{szliao@tju.edu.cn}\\
 \addr Tianjin University
 \AND
 \Name{Zenglin Xu} \thanks{Corresponding author.}\Email{zenglin@gmail.com}\\
 \addr Fudan University and Shanghai Academy of AI for Science
 \AND
 \Name{Liqiang Nie} $^\ast$ \Email{nieliqiang@gmail.com}\\
 \addr Harbin Institute of Technology (Shenzhen)
}

\begin{document}

\maketitle

\begin{abstract}%
  In this paper,
  we study the problem of online sparse linear regression (OSLR)
  where the algorithms are restricted to accessing only $k$ out of $d$ attributes per instance for prediction,
  which was proved to be NP-hard.
  Previous work gave polynomial-time algorithms
  assuming the data matrix satisfies the linear independence of features,
  the compatibility condition, or the restricted isometry property.
  We introduce a new polynomial-time algorithm,
  which significantly improves previous regret bounds \citep{Ito2017Efficient}
  under the compatibility condition that is weaker than the other two assumptions.
  The improvements benefit from a tighter convergence rate of the $\ell_1$-norm error of our estimators.
  Our algorithm leverages the well-studied Dantzig Selector,
  but importantly with several novel techniques, including
  an algorithm-dependent sampling scheme for estimating the covariance matrix,
  an adaptive parameter tuning scheme,
  and a batching online Newton step with careful initializations.
  We also give novel and non-trivial analyses,
  including an induction method for analyzing the $\ell_1$-norm error,
  careful analyses on the covariance of non-independent random variables,
  and a decomposition on the regret.
  We further extend our algorithm to OSLR with additional observations
  where the algorithms can observe additional $k_0$ attributes after each prediction,
  and improve previous regret bounds \citep{Kale2017Adaptive,Ito2017Efficient}.
\end{abstract}


\section{Introduction}

    For most online prediction tasks,
    algorithms are assumed to observe all of the attributes of an instance $\mathbf{x}_t\in\mathbb{R}^d$
    at each round $t=1,2,\ldots,T$.
    However,
    the assumption is hard to be satisfied in some real-world scenarios due to various constraints,
    such as computational constraint, human labors and privacy constraint
    \citep{Hazan2012Linear,Jain2012Differentially,Zolghadr2013Online,Murata2018Sample}.
    In the task of medical diagnosis of a disease \citep{Cesa-Bianchi2010Efficient},
    $\mathbf{x}_t$ contains the results of a large number of medical tests.
    However, many patients can only pay the cost for several medical tests.
    Thus $\mathbf{x}_t$ must be sparse.
    Another example is personalized recommendation \citep{Jain2012Differentially}.
    Due to the privacy constraint,
    search engines are not allowed to use sensitive attributes of users,
    such as gender, age, job and so on.
    $\mathbf{x}_t$ is also sparse.
    The algorithms typically make predictions using only limited attributes.

    \citet{Kale2014Open} first formulated online prediction problems with limited access to attributes
    as online sparse linear regression (OSLR).
    At each round $t$,
    an adversary gives an instance $\mathbf{x}_t$ to a learner.
    The learner can only observe $k$, $k<d$, attributes of $\mathbf{x}_t$ at most,
    and outputs a prediction $\hat{y}_t$.
    Then the adversary gives the true output $y_t$.
    The learner suffers a loss $(\hat{y}_t-y_t)^2$.
    The learner aims to minimize her cumulative losses over $T$ rounds.
    Typically, we compare the cumulative losses of the learner with that of any $k$-sparse competitor
    and define the regret as follows,
    \begin{equation}
    \label{eq:AAAI25:definition_regret}
        \forall \mathbf{w}\in \left\{\mathbf{v}\in\mathbb{R}^d:
        \Vert\mathbf{v}\Vert_0\leq k\right\},\quad
        \mathrm{Reg}(\mathbf{w})
        =\sum^T_{t=1}\left(\hat{y}_t-y_t\right)^2
        -\sum^T_{t=1}\left(\mathbf{w}^\top\mathbf{x}_t-y_t\right)^2.
    \end{equation}
    The primal goal is to develop an algorithm that runs in time $O(\mathrm{poly}(T,k,d))$ per-iteration,
    and enjoys a (an expected) regret of order $O(\mathrm{poly}(d,k)\cdot o(T))$.
    There is also a relaxation of OSLR denoted by $(k,k_0,d)$-OSLR,
    in which the algorithms are allowed to observe additional $k_0=O(k\ln{d})$ attributes
    after each prediction.

    The offline problems related to OSLR and $(k,k_0,d)$-OSLR
    concern the approximation of the $k$-sparse solution in linear systems
    that has been proven to be NP-hard by a reduction from set cover problem
    \citep{Natarajan1995Sparse,Foster2015Variable}.
    By a similar approach,
    it was shown that
    both OSLR and $(k,k_0,d)$-OSLR are also NP-hard \citep{Foster2016Online}.
    Specifically,
    there is no algorithm running in time $O(\mathrm{poly}(T,k,d))$ per-iteration
    and having an expected regret of $O(\mathrm{poly}(d)\cdot T^{1-\delta})$
    for any constant $\delta>0$ unless $\mathbf{NP}\subseteq \mathbf{BPP}$.
    They also proposed an inefficient algorithm for $(k,k_0,d)$-OSLR
    that enjoys an expected regret of $O(\frac{d^2}{k^2_0}\sqrt{kT\ln{d}})$
    at a $O((k+k_0)\binom{d}{k})$ space and per-round time complexity.
    Given the computational hardness result,
    it is necessary to further restrict the problem with additional regularity assumptions.
    Certain assumptions on the data matrix are adequate for
    approximating the sparse solutions of linear systems,
    such as the restricted isometry property (RIP) \citep{Candes2005Decoding},
    the compatibility condition \citep{Sara2007The},
    the restricted eigenvalues condition \citep{Bickel2009Simultaneous},
    the restricted strong convexity (RSC) and smoothness of the loss function
    \citep{Murata2018Sample}, and so on.
    \cite{Kale2017Adaptive} proposed the first algorithm for $(k,k_0,d)$-OSLR
    (under a different name POSLR)
    based on the Dantzig Selector \citep{Candes2007The}
    that runs in time $O(\mathrm{poly}(d))$ per-iteration and enjoys a high-probability regret of
    $O(\frac{k^2d^3\ln{T}}{k^3_0}\ln\frac{Td}{\delta})$ under RIP.
    \cite{Ito2017Efficient} independently proposed algorithms for OSLR and $(k,k_0,d)$-OSLR,
    all of which run in time $O(d)$ per-iteration
    and enjoy an expected regret of $O(\mathrm{poly}(d,k)+\mathrm{poly}(d,k)\sqrt{T})$
    under the linear independence of features or compatibility condition.

    However, previous computationally efficient algorithms
    either exhibit undesirable reliance on problem-dependent parameters like
    $T$, $d$, $\min_{i\in S}\vert w^\ast_i\vert<1$,
    and numerical factor,
    or require stringent assumptions,
    in which
    $\mathbf{w}^\ast=(w^\ast_1,\ldots,w^\ast_d)$ is a $k$-sparse vector with support set $S$
    such that $y_t=\langle \mathbf{x}_t,\mathbf{w}^\ast\rangle+\eta_t$
    where $\eta_t$ is a noise.
    Specifically,
    for $(k,k_0,d)$-OSLR,
    the regret bound of the algorithm \citep{Kale2017Adaptive} depends on $O(d^3)$ under the stringent RIP,
    and the regret bound of the algorithm \citep{Ito2017Efficient}
    depends on $O(\sqrt{T})$ under the stronger linear independent features condition.
    For OSLR,
    the regret bounds of the two algorithms \citep{Ito2017Efficient}
    depend on $d^8$
    \footnote{There are typos in original paper \citep{Ito2017Efficient}.
    The correct regret bounds are $O(\frac{d^8}{k^8})$,
    not $O(\frac{d^{16}}{k^{16}})$ as erroneously stated on Page 2 of the original paper.},
    $\min_{i\in S}\vert w^\ast_i\vert^{-7}$
    or a large constant factor $128^2\cdot 36^8$
    under the linear independent features or the compatibility condition.
    Table \ref{tab:ICML2025:OSLR} and Table \ref{tab:ICML2025:POSLR}
    show previous regret bounds.
    It is natural to ask
    \begin{question}
        Whether there are polynomial-time algorithms for OSLR with regret bounds
        that exhibit improved dependence on $T$, $d$, $\min_{i\in S}\vert w^\ast_i\vert$ and numerical factor,
        and for $(k,k_0,d)$-OSLR with regret bounds
        that exhibit improved dependence on $T$ and $d$ under mild assumptions?
    \end{question}

\subsection{Main Results}

    In this paper,
    we will answer the question affirmatively.
    We first propose a new polynomial-time algorithm for OSLR, named DS-OSLRC,
    and then extend the algorithm to $(k,k_0,d)$-OSLR, named DS-POSLRC.
    Our algorithms enjoy much better regret bounds under the realizable assumption,
    that is, there is a true $k$-sparse vector,
    and the compatibility condition that is less restrictive than RIP
    and the linear independent features condition.
    Our main results are summarized in Table \ref{tab:ICML2025:OSLR} and Table \ref{tab:ICML2025:POSLR}.
    \begin{itemize}
      \item DS-OSLRC outputs an estimator $\hat{\mathbf{w}}_s$
            satisfying $\Vert \hat{\mathbf{w}}_s-\mathbf{w}^\ast\Vert_1=\tilde{O}(\sqrt{kd/s})$
            for sufficiently large values of $s$,
            improving previous convergence rates \citep{Ito2017Efficient,Kale2017Adaptive}.
      \item For OSLR, DS-OSLRC enjoys a $4\sqrt{T}+\frac{102^4k^2d^2}{\delta^8_Sh(\mathbf{w}^\ast)^2}\ln^2\frac{dT}{\delta}
          +O(1)$ high-probability regret bound,
          in which $h(\mathbf{w}^\ast)=\min_{i\in S}\vert w^\ast_i\vert<1$
          and $O(1)$ hides some lower order terms.
          The averaging per-round time complexity is $O(\frac{\mathrm{LP}_d}{\sqrt{T}}+k^2)$,
          in which $\mathrm{LP}_d=O(\mathrm{poly}(d))$.
          We significantly improve the regret bounds by \cite{Ito2017Efficient}
          in terms of $d$, $T$, $h(\mathbf{w}^\ast)$ and numerical factor,
          under weaker or the same assumptions.
          Besides,
          the regret bounds in \cite{Ito2017Efficient} hold in expectation,
          which are weaker than our high-probability regret bound.
      \item For $(k,k_0,d)$-OSLR,
            DS-POSLRC enjoys a
        $O(\frac{k^2d^2}{\delta^4_Sk^2_0}\ln\frac{dT}{\delta}
        +\frac{k^2d}{\delta^4_Sk_0}\ln(T)\ln\frac{dT}{\delta})$ high-probability regret bound.
          The per-round time complexity is $O(\mathrm{LP}_d)$.
          We improve the regret bound by \cite{Kale2017Adaptive}
          by a factor of $O(\min\{\frac{d^2}{k^2_0},\frac{d}{k_0}\ln{T}\})$,
          and significantly improve the regret bound by \cite{Ito2017Efficient} in terms of $T$,
          under weaker assumptions.
    \end{itemize}

    \begin{table}[!t]
      \centering
      \begin{tabular}{l|r|r|l}
      \toprule
        Algorithm       & Regret bound w.r.t. $\mathbf{w}^\ast$& Per-round time & Assumptions\\
        \hline
        \cite{Ito2017Efficient}
        & $8\sqrt{kT}+\frac{8192^2\cdot d^8}{\sigma^8_dk^7h(\mathbf{w}^\ast)^7}+O(1)$ & $O(d)$
        &(a), (b), {\color{magenta} (c)}, (g), (i)\\
        \cite{Ito2017Efficient}
        & $8\sqrt{kT}+\frac{128^2\cdot 36^8\cdot d^8}{\delta^8_Sk^3h(\mathbf{w}^\ast)^7}+O(1)$ &$O(d)$
        &(a), (b), {\color{magenta} (d)}, (g), (i)\\
        {\color{magenta}DS-OSLRC} (Ours)
        &{\color{magenta}$4\sqrt{T}+\frac{102^4k^2d^2}{\delta^8_Sh(\mathbf{w}^\ast)^2}\ln^2\frac{dT}{\delta}+O(1)$}
        &{\color{magenta}$O\left(\frac{\mathrm{LP}_d}{\sqrt{T}}+k^2\right)$}&
        {(a), (b), {\color{magenta} (d)}, (f)}\\
      \bottomrule
      \end{tabular}
      \caption{Comparison of the algorithms for OSLR.
      $\delta\in(0,1)$ is a probability parameter,
      $h(\mathbf{w}^\ast)=\min_{i\in S}\vert w^\ast_i\vert< 1$,
      $\sigma_d\leq \delta_S$.
      $\mathrm{LP}_d=O(\mathrm{poly}(d))$ is the time complexity of solving a linear programming
      with $d$ constraints and $d$ variables.
      Assumptions: (a) realizable, (b) sparsity, (c) linear independence of features,
      (d) compatibility condition,
      (g) bounded noise,
      (f) Gaussian noise, and (i) i.i.d. instances.
      \textbf{Assumption (c) is stronger than (d).}}
      \label{tab:ICML2025:OSLR}
    \end{table}

    \begin{table}[!t]
      \centering
      \begin{tabular}{l|r|r|l|r}
      \toprule
        Algorithm       & Regret bound w.r.t. $\mathbf{w}^\ast$& Per-round time & Assumptions & $k_0$\\
      \toprule
        \cite{Foster2016Online}  & $O\left(\frac{d^2}{k^2_0}\sqrt{kT\ln{d}}\right)$ & $O(d^k)$ & -&$\geq 2$\\
        \cite{Ito2017Efficient}
        & $O\left(\frac{d}{\sigma^2_dk_0}\sqrt{T}\right)$ & $O(d)$ &(a), (b), {\color{magenta}(c)}, (g), (i)&$\geq 2$\\
        \cite{Kale2017Adaptive}  & $O\left(\frac{k^2d^3}{k^3_0}\ln{(T)}\ln\frac{dT}{\delta}\right)$
        & $O(\frac{\mathrm{LP}_d}{T/\log{T}}+d)$ & (a), (b), {\color{magenta}(e)}, (f)&$\geq 2$\\
        {\color{magenta}DS-POSLRC}
        & {\color{magenta}$O\left(\frac{k^2d}{\delta^4_Sk_0}(\frac{d}{k_0}+\ln{T})\ln\frac{dT}{\delta}\right)$}
        & {\color{magenta}$O\left(\mathrm{LP}_d\right)$}
        & (a), (b), {\color{magenta}(d)}, (f)&$\geq 3$\\
      \bottomrule
      \end{tabular}
      \caption{Comparison of the algorithms for $(k,k_0,d)$-OSLR.
      Assumption (e) is RIP,
      and the others follow Table \ref{tab:ICML2025:OSLR}.
      \textbf{Assumption (c) is stronger than (e). (e) is stronger than (d)}.}
      \label{tab:ICML2025:POSLR}
    \end{table}

\subsection{Technical Challenges and Contributions}

    DS-OSLRC builds upon the well-studied Dantzig Selector \citep{Candes2007The}.
    Despite this foundation,
    achieving such significant improvements on the regret bounds
    necessitates more novel methodologies in algorithm design and regret analyses.
    We explain as follows.

    The improvements on problem-dependent parameters, such as $d,T,h(\mathbf{w}^\ast)$,
    stem from a new algorithm-dependent sampling scheme,
    and a batching online Newton step (ONS) \citep{Hazan2007Logarithmic} with careful initializations.
    At some exploration round $s>1$,
    let $\hat{\mathbf{w}}_s$ be the estimator of $\mathbf{w}^\ast$.
    Previous algorithms \citep{Kale2017Adaptive,Ito2017Efficient}
    uniformly sample $k$ or $k_0$ attributes from $\mathbf{x}_s$
    to estimate $\mathbf{x}_s$ and $\mathbf{x}_s\mathbf{x}^\top_s$.
    Whereas our sampling scheme uses $\hat{\mathbf{w}}_{s-1}$ to construct sampling probability
    and simultaneously combines uniformly sampling without replacement.
    What's more,
    it requires novel analyses on $\Vert \Delta_s(S)\Vert_1$,
    in which $\Delta_s(S)=\hat{\mathbf{w}}_s(S)-\mathbf{w}^\ast$
    and $\hat{\mathbf{w}}_s(S)$ extracts the elements of $\hat{\mathbf{w}}_s$ restricted to $S\subseteq [d]$.
    We will prove that $\Vert \Delta_s(S)\Vert_1$ satisfies the following inequality
    \begin{equation}
    \label{eq:COLT2025:high_level_upper_bound}
    \begin{split}
        &\frac{\delta^2_S\Vert \Delta_s(S)\Vert_1}{k}\leq c_1\frac{(d-1)(d-2)}{(k-1)(k-2)s}\ln\frac{d}{\delta}+
        c_2\sqrt{\frac{d-1}{s(k-1)} \ln\frac{d}{\delta}} +\mu_s+\\
        &\quad\qquad\frac{c_3}{s}\sqrt{\frac{(d-1)(d-2)}{(k-1)(k-2)}\sum^s_{\tau=1}\Vert\Delta_{\tau-1}(S)\Vert^2_1\ln\frac{d}{\delta}}
        +c_4\sqrt{\frac{(d-1)(d-2)}{(k-1)(k-2)s}\ln\frac{d^2}{\delta}}\Vert\Delta_s(S)\Vert_1,
    \end{split}
    \end{equation}
    where $\delta_S,c_1,c_2,c_3,c_4$ are constants
    and $\mu_s$ is an estimator of $\sqrt{\sum^s_{\tau=1}\Vert\Delta_{\tau-1}(S)\Vert^2_1\ln\frac{d}{\delta}}$.
    We will use a novel induction method to solve the inequality
    and prove that $\Vert \Delta_s(S)\Vert_1=\tilde{O}(\sqrt{kd/s})$ for sufficiently large $s$,
    making it possible to refine the dependence on $d$ and $h(\mathbf{w}^\ast)$.
    We further refine the dependence on $T$
    by using ONS to update parameters throughout each epoch $\mathcal{T}_s\subseteq [T]$.
    Through novel analyses,
    DS-OSLRC enjoys a regret of
    $O(\sqrt{T}+\frac{k^2d^2}{h(\mathbf{w}^\ast)^2}\ln^2{(dT)})$.

    The algorithm-dependent sampling scheme also poses a technical challenge on the algorithm.
    To be specific,
    at each exploration round $s$,
    there is a threshold $\gamma_s>0$ related the Dantzig Selector,
    in which $\gamma_s$ depends on $\sum^s_{\tau=1}\Vert\Delta_{\tau-1}(S)\Vert^2_1$.
    As $\mathbf{w}^\ast$ is unknown,
    it is impossible to use the optimal $\gamma_s$.
    A similar issue exists in the online linearized LASSO algorithm \citep{Yang2023Online},
    in which the regularization parameter depends on $\Vert\Delta_{s-1}\Vert_1$.
    In fact, it is enough to construct $\hat{\gamma}_s\geq c\gamma_s$,
    where $c$ is a constant.
    To this end,
    we use a ``guess-then-verification'' technique.
    Specifically,
    if $\gamma_s$ is known,
    then we can set an exact value for $\mu_s$ in \eqref{eq:COLT2025:high_level_upper_bound}
    and solve the inequality.
    In this way,
    we obtain the ideal convergence rate that allows us to infer $\hat{\gamma}_s$.
    Then we use an induction method to verify the correctness,
    and solve \eqref{eq:COLT2025:high_level_upper_bound} for obtaining the real convergence rate.

    The improvement on the constant factor also
    stems from a tighter convergence rate of $\Vert\Delta_s(S)\Vert_1$,
    and requires novel analyses,
    including careful analyses on the covariance of non-independent random variables,
    and a decomposition on the regret.
    Since DS-OSLRC samples $k$ attributes without replacement,
    the observed attributes are not independent.
    It is necessary to control the covariances.
    Regarding the regret analysis,
    we decompose the regret in each epoch $\mathcal{T}_s$ into two components,
    $$
        \underbrace{\sum_{t\in \mathcal{T}_s}
        \left[\ell(\langle\hat{\mathbf{w}}_s(S_s),\mathbf{x}_t\rangle,y_t)
        -\ell(\mathbf{w}_s,\mathbf{x}_t\rangle,y_t)\right]}_{\mathrm{R}_1}
        +\underbrace{\sum_{t\in \mathcal{T}_s}
        \left[\ell(\mathbf{w}_s,\mathbf{x}_t\rangle,y_t)
        -\ell(\langle\mathbf{w}^\ast,\mathbf{x}_t\rangle,y_t)\right]}_{\mathrm{R}_2},
    $$
    in which $\mathrm{Supp}(\mathbf{w}_s)\subset S_s$.
    $\mathrm{R}_1$ can be bounded by the regret of ONS.
    Moreover, the constant factor associated with $\mathrm{R}_1$ is quite small.
    $\mathrm{R}_2$ depends on $\Vert\mathbf{w}_s-\mathbf{w}^\ast\Vert_1$.
    We note that $\mathbf{w}_s$ can be selected arbitrarily.
    We will use $\mathbf{w}_s=\mathbf{w}^\ast(S_s\cap S)$,
    rather than the obvious but sub-optimal $\hat{\mathbf{w}}_s(S_s)$.
    It can be proved $\Vert \mathbf{w}^\ast(S_s\cap S)-\mathbf{w}^\ast\Vert_1
    \leq \Vert \hat{\mathbf{w}}_s(S_s)-\mathbf{w}^\ast\Vert_1$,
    thereby further reducing the numerical factor.

\subsection{Related Work}

    \cite{Yang2023Online} formulated another variant of online sparse linear regression
    where the algorithms observe all attributes of each $\mathbf{x}_s$, $s=1,\ldots,T$,
    and aim to compute a sequence of estimators $\{\hat{\mathbf{w}}_s\}^T_{s=1}$ approximating $\mathbf{w}^\ast$,
    and proposed an online linearized LASSO algorithm, named OLin-LASSO.
    The algorithm was proved to achieve the optimal $\tilde{O}(\sqrt{k/s})$ $\ell_2$-norm error bound
    under the RSC condition \citep{Murata2018Sample}.
    RSC is stronger than the restricted eigenvalues
    and the compatibility condition, but is weaker than RIP.
    The DA-GL algorithm \citep{Yang2010Online}, which uses dual averaging to solve group LASSO,
    can also be applied to solve this variant.
    Without additional assumptions,
    DA-GL runs in time $O(d)$ per-iteration
    and returns a solution whose $\ell_2$-norm error is bounded by a constant.
    The I-LAMM algorithm \citep{Fan2018I-LAMM}
    first computes a good initialization $\hat{\mathbf{w}}_0$ on a batch of initial data,
    and then iteratively computes a sequence of better estimators.
    I-LAMM is also optimal under the restricted eigenvalues condition.
    The three algorithms are clearly unsuitable for OSLR and $(k,k_0,d)$-OSLR,
    as they require the complete information of each instance.

    Another related problem is linear regression with limited observations (LRLO)
    \citep{Cesa-Bianchi2010Efficient,Cesa-Bianchi2011Efficient,Hazan2012Linear}
    where the algorithms only observe $k$ attributes of each instance at training phase,
    but aim to learn a linear model with low excess risk.
    \cite{Murata2018Sample} studied a harder variant of LRLO
    where the algorithms only observe $k$ attributes of each instance at both training and test phase,
    and proposed an exploration-then-exploitation algorithm
    with improved sample complexity under RSC.
    Since the algorithms are not required to give predictions for training instances,
    the two problems are inherently less complex than OSLR.
    In other words,
    any algorithm for OSLR is applicable for the two problems.
    However, the reverse is not true.


\section{Notations and Preliminaries}

\subsection{Notations}

    For any vector $\mathbf{x}\in\mathbb{R}^d$,
    denote by $x_i$ its $i$-th coordinate.
    For any matrix $\mathbf{X}\in\mathbb{R}^{m\times n}$,
    denote by $X_{i,j}$ the element on the $i$-th row and $j$-th column.
    For any $d\in\mathbb{N}$,
    let $[d] = \{1,2,\ldots,d\}$.
    We use the notation $\Vert\cdot\Vert_p$ where $p\in\{0,1,2,\infty\}$,
    to denote the $p$-norm in $\mathbb{R}^d$.
    For any matrix $\mathbf{X}\in\mathbb{R}^{m\times n}$,
    let $\Vert \mathbf{X}\Vert_{\infty}=\max_{i\in[m]}\sum^n_{j=1}\vert X_{i,j}\vert$ be the infinity norm.
    Let $\left\{\left(\mathbf{x}_t,y_t\right)\right\}_{t\in[T]}$ be a sequence of examples,
    where $\mathbf{x}_t\in\mathcal{X}=\{\mathbf{x}\in\mathbb{R}^d:\Vert \mathbf{x}\Vert_{\infty}\leq 1\}$
    is an instance.
    Let $\mathbb{I}_{E}$ be the indictor function.
    If the event $E$ occurs, then $\mathbb{I}_{E}=1$.
    Otherwise, $\mathbb{I}_{E}=0$.
    For any subset $S\subset[d]$,
    let $S^c=[d]\setminus S$.
    For any $\mathbf{x}\in\mathbb{R}^d$ and $S\subset[d]$,
    we define $\mathbf{x}(S):=
    \left(x_1\cdot\mathbb{I}_{1\in S},x_2\cdot\mathbb{I}_{2\in S},...,x_d\cdot\mathbb{I}_{d\in S}\right)$.
    Let $\mathcal{N}(0,\sigma^2)$ be a normal distribution with variance $\sigma^2$,
    $\mathbf{a}_d$ be the $d$-dimensional vector of $a$'s,
    and $\mathbf{0}_{m\times n}$ be a matrix whose elements are all zeros.
    Let $\mathrm{Supp}(\mathbf{w})=\{i\in[d]:w_i\neq 0\}$ be the support set of a vector $\mathbf{w}$.

\subsection{Sufficient Conditions for Computationally Efficient Algorithms for OSLR}

    \begin{assumption}[Realizable]
    \label{ass:ICML2025:Realizable assumption}
        For each $(\mathbf{x}_t,y_t)$, $t=1,2,...,T$,
        there exists a $\mathbf{w}^\ast$ such that
        $$
            y_t = \langle \mathbf{w}^\ast, \mathbf{x}_t\rangle+\eta_t,
        $$
        in which $\eta_1,\eta_2,\ldots,\eta_T$ are independent noises from $\mathcal{N}(0,\sigma^2)$,
        $\Vert \mathbf{x}_t\Vert_{\infty}\leq 1$
        and $\Vert \mathbf{w}^\ast\Vert_1\leq 1$.
    \end{assumption}
    The normal distribution implies,
    with probability at least $1-\delta$,
    $\vert y_t\vert\leq 1+\sigma\sqrt{2\ln\frac{1}{\delta}}$
    for a fixed $t$.

    \begin{assumption}[Sparsity]
    \label{ass:COLT2025:sparsity}
        $\mathbf{w}^\ast$ is $k$-sparse,
        i.e., $\Vert \mathbf{w}^\ast\Vert_0\leq k$.
    \end{assumption}

    \begin{assumption}[$(\delta_S,S,\alpha)$-compatibility condition \citep{Geer2009On}]
    \label{ass:ICML25:restricted_eigenvalue}
        Let $S\subset[d]$ be an arbitrary subset and
        $\Omega_S=\{\mathbf{w}\in\mathbb{R}^d:\mathbf{w}\neq \mathbf{0}_d,
        \Vert\mathbf{w}(S^c)\Vert_1\leq \alpha \Vert\mathbf{w}(S)\Vert_1\}$.
        The data matrix $\mathbf{X}\in\mathbb{R}^{d\times T}$
        satisfies the $(\delta_S,S,\alpha)$-compatibility condition,
        if there is a constant $\delta_S>0$ such that
        $$
            \min_{\mathbf{w}\in\Omega_S}\frac{\vert S\vert\cdot \Vert \mathbf{X}^\top\mathbf{w}\Vert^2_2}{T\cdot \Vert\mathbf{w}(S)\Vert^2_1}
            =\delta^2_S.
        $$
    \end{assumption}

    The $\left(\delta_S,S,\alpha\right)$-compatibility condition requires
    that the covariance matrix $\mathbf{X}\mathbf{X}^\top$
    enjoys a kind of ``restricted'' positive definiteness on the restricted set $\Omega_S$,
    and the $\ell_1$-norm and the norm defined by $\mathbf{X}\mathbf{X}^\top$ to be compatible.
    As the equivalence of $\ell_1$-norm and $\ell_2$-norm of a vector,
    the compatibility condition is essentially equivalent to the restricted eigenvalues condition.
    For the sake of fair comparison with the results by \citep{Ito2017Efficient},
    we adopt the compatibility condition.
    We can obtain similar or better results under the restricted eigenvalues condition
    by slightly adjusting the analysis.
    The compatibility condition is weaker than the three popular assumptions adopted for solving OSLR,
    including the linear independence of features, RIP, and RSC.

    In this work,
    we only consider the case $3\leq k\leq d-3$.
    For $k\leq 2$ and $k\geq d-2$,
    it is easy to establish an algorithm running in time $O(d)$ or $O(d^2)$.
    Specifically,
    we reduce OSLR to an instance of the adversarial multi-armed bandits problem \citep{Auer2002The}
    by enumerating all of the $N=\binom{d}{k}$ combinations of features
    and corresponding each combination as an arm.
    At each round $t\in[T]$,
    we maintain a parameter vector $\mathbf{w}_{t,i}$ for each combination $i\in[N]$.
    First,
    we select a $\mathbf{w}_{t,I_t}$, $I_t\in[N]$, and make a prediction $\hat{y}_t$.
    After observing $y_t$,
    $\mathbf{w}_{t,I_t}$ can be updated by online gradient descent \citep{Zinkevich2003Online}.
    The technical challenge is that $\mathbf{w}_{t,i}$ can not be updated for all $i\neq I_t$.
    We must carefully design a master algorithm for choosing $I_t$
    and an elaborate parameter updating strategy.
    The same technical challenge exists in the problem of corralling a band of bandit algorithms
    \citep{Agarwal2017Corralling,Foster2020Adapting,Luo2022Corralling}.
    Luckily,
    by the master algorithm proposed in \citep{Foster2020Adapting},
    it is possible to obtain an expected regret of $O\left(\sqrt{NT}\right)$.
    A previous algorithm \citep{Ito2018Online} uses a batching technique
    to combine any bandit algorithms and online learning algorithms, like ONS,
    but only provides a regret of $\tilde{O}\left(N^{\frac{1}{3}}T^{\frac{2}{3}}\right)$.

\section{Algorithm}

    For clarity,
    we first provide a comprehensive overview of the key components and workflow of the proposed algorithm.

\subsection{Overview of Algorithm}

    We first introduce some notations.
    Let $\mathcal{T}=\left\{s^2:s=1,\ldots,\lfloor\sqrt{T}\rfloor\right\}$
    and $\mathcal{I}_s=\left\{1,2^2,\ldots,s^2\right\}$.
    For any $s\in \left\{1,\ldots,\lfloor\sqrt{T}\rfloor\right\}$, let
    $\mathcal{T}_s=\left\{s^2+1,s^2+2,\ldots,(s+1)^2-1\right\}$.
    If $s=\lfloor\sqrt{T}\rfloor$,
    then we define $(s+1)^2-1:=T$.
    It is obvious that $[T]=\mathcal{T}\cup^{\lfloor\sqrt{T}\rfloor}_{s=1}\mathcal{T}_s$.

    Overall,
    our algorithm alternately executes an exploration round and an exploitation epoch.
    Since the algorithms only access $k$ attributes of each instance,
    it is impossible to conduct an exploration process per-round.
    Under certain assumptions,
    both LASSO \citep{Tibshirani1996Regression,Bunea2007Sparsity}
    and the Dantzig Selector \citep{Candes2007The}
    can learn an estimator denoted by $\hat{\mathbf{w}}$ satisfying
    $\Vert \hat{\mathbf{w}}-\mathbf{w}^\ast\Vert_p=O\left(1/\sqrt{T}\right)$
    using $T$ examples where $p\in\{1,2\}$.
    Therefore,
    it is enough to conduct an exploration process across several rounds.
    Specifically,
    our algorithm will execute exploration at each round $t\in\mathcal{T}$.
    At any round $t\in\mathcal{T}$,
    our algorithm will sample $k$ attributes
    and solve the Dantzig Selector, i.e., \eqref{eq:ICML25:new_Dantzig_Selector},
    to obtain an estimator of $\mathbf{w}^{\ast}$ denoted by $\hat{\mathbf{w}}_s$
    where $s=\sqrt{t}$.
    The key technical contribution is the development of a novel algorithm-dependent sampling scheme.

    Given $\hat{\mathbf{w}}_s$,
    we further construct an estimator of the true support set $S$, denoted by $S_s$.
    Then our algorithm transitions to the exploitation epoch $\mathcal{T}_s$
    where our algorithm always observes the attributes indexed by $S_s$,
    and runs ONS initialized with $\hat{\mathbf{w}}_s(S_s)$.
    Although our algorithm uses ONS to learn $\mathbf{w}^{\ast}$,
    we still term the epoch ``exploitation''
    because it fully relies on $S_s$ and $\hat{\mathbf{w}}_s(S_s)$.
    We will propose a careful initialization scheme for ONS
    when our algorithm enters the next exploitation epoch $\mathcal{T}_{s+1}$.

    However,
    solving \eqref{eq:ICML25:new_Dantzig_Selector} requires the prior of $\mathbf{w}^{\ast}$.
    To be specific,
    there is a parameter $\gamma_s$ depending on $\Vert \Delta_{\tau}(S)\Vert_1,\tau<s$,
    where $ \Delta_{\tau}(S):=\hat{\mathbf{w}}_{\tau}(S)-\mathbf{w}^{\ast}$.
    To tune $\gamma_s$ adaptively,
    we first assume that $\Vert\Delta_{\tau}(S)\Vert_1$, $\tau<s$, are known.
    Then we can prove that $\Vert \Delta_s(S)\Vert_1$ satisfies the following inequality
    \begin{equation}
    \label{eq:COLT2025:high_level_upper_bound_ideal}
    \begin{split}
        &\frac{\delta^2_S\Vert \Delta_s(S)\Vert_1}{k}\leq c_1\frac{(d-1)(d-2)}{(k-1)(k-2)s}\ln\frac{d}{\delta}+
        c_2\sqrt{\frac{d-1}{s(k-1)} \ln\frac{d}{\delta}} +\sqrt{\sum^s_{\tau=1}\Vert\Delta_{\tau-1}(S)\Vert^2_1\ln\frac{d}{\delta}}+\\
        &\quad\qquad\frac{c_3}{s}\sqrt{\frac{(d-1)(d-2)}{(k-1)(k-2)}\sum^s_{\tau=1}\Vert\Delta_{\tau-1}(S)\Vert^2_1\ln\frac{d}{\delta}}
        +c_4\sqrt{\frac{(d-1)(d-2)}{(k-1)(k-2)s}\ln\frac{d^2}{\delta}}\Vert\Delta_s(S)\Vert_1.
    \end{split}
    \end{equation}
    We can solve \eqref{eq:COLT2025:high_level_upper_bound_ideal} by an induction method,
    yielding an ideal convergence rate on $\Vert \Delta_s(S)\Vert_1$
    that nearly matches the final convergence rate in Lemma \ref{lemma:estimator_error:DS-OSLRC}.
    We term it ``an ideal convergence rate''
    as it requires the prior of $\mathbf{w}^{\ast}$.
    Then we use the ideal convergence rate to estimate
    $\Vert \Delta_{\tau}(S)\Vert_1$ and $\gamma_s$,
    which only depend on $\tau<s$ and some known constants.
    For simplicity,
    we will not give the ideal convergence rate,
    but implicitly use it to define the estimator of $\gamma_s$
    in Section \ref{sec:COLT2025:adaptive_parameter_tuning}.

\subsection{Exploration}

    We use the Dantzig Selector to learn a sequence of estimators,
    as it can obtain a smaller constant factor in the convergence rate compared to LASSO.
    For the sake of clarity,
    we first introduce the Dantzig Selector in the full information setting.
    For a fixed $t\in \mathcal{T}$,
    let $\mathbf{X}_{\mathcal{I}_s}=(\mathbf{x}_1,\mathbf{x}_{2^2},\ldots,\mathbf{x}_{s^2})
    \in\mathbb{R}^{d\times s}$
    and $\mathbf{Y}_{\mathcal{I}_s}=(y_1,y_{2^2},\ldots,y_{s^2})^\top\in\mathbb{R}^s$,
    in which $s=\sqrt{t}$.
    The Dantzig Selector can return an estimator of $\mathbf{w}^\ast$,
    denoted by $\hat{\mathbf{w}}_s$,
    by solving the following constrained problem.
    $$
        \min_{\mathbf{w}\in\mathbb{R}^d}\Vert\mathbf{w}\Vert_1,\quad
        \mathrm{s.t.}~\left\Vert \frac{1}{s}\mathbf{X}_{\mathcal{I}_s}
        \mathbf{Y}_{\mathcal{I}_s}-\frac{1}{s}\mathbf{X}_{\mathcal{I}_s}\mathbf{X}^\top_{\mathcal{I}_s}\mathbf{w}
        \right\Vert_{\infty}\leq \gamma_s.
    $$
    Unluckily,
    the algorithms for OSLR can not directly observe $\mathbf{X}_{\mathcal{I}s}$.
    Therefore,
    it is necessary to construct estimators of $\mathbf{X}_{\mathcal{I}_s}$
    and $\mathbf{X}_{\mathcal{I}_s}\mathbf{X}^\top_{\mathcal{I}_s}$.
    Previous algorithms \citep{Ito2017Efficient,Kale2017Adaptive,Murata2018Sample}
    uniformly sample $k$ attributes from ${\mathbf{x}}_{\tau^2}$
    for constructing unbiased estimators of ${\mathbf{x}}_{\tau^2}$ and ${\mathbf{x}}_{\tau^2}{\mathbf{x}}^\top_{\tau^2}$
    for all $\tau=1,2,\ldots,s$.
    Nevertheless,
    employing such a basic sampling scheme degrades the convergence rate concerning its reliance on the dimension $d$.
    To address this issue,
    we will define an algorithm-dependent sampling scheme.
    A better sampling probability should be defined as a function of $\mathbf{w}^\ast$,
    such as $\mathbf{q}^{\ast}=(q^{\ast}_1,\ldots,q^{\ast}_d)$ in which
    $q^{\ast}_i=\frac{\vert w^\ast_i\vert}{\Vert \mathbf{w}^\ast\Vert_1}$, $i\in[d]$,
    implying $x_{\tau^2,i}$ will not be selected for any $i\notin S$.
    However, it is infeasible to construct $\mathbf{q}^{\ast}$ as $\mathbf{w}^\ast$ is unknown.
    Benefit from the sequential nature of online learning,
    we can utilize the estimator in the last round, denoted by $\hat{\mathbf{w}}_{s-1}$,
    as a reliable approximation of $\mathbf{w}^\ast$
    and construct a sampling distribution
    \footnote{A similar but fundamentally different idea was adopted by the OLin-LASSO Algorithm \citep{Yang2023Online},
     in which $\hat{\mathbf{w}}_{s-1}$ was used to approximate the square loss function at round $s$.},
    denoted by $\mathbf{q}_s=(q_{s,1},\ldots,q_{s,d})$,
    in which $q_{s,i}=\frac{\vert \hat{w}_{s-1,i}\vert}{\Vert \hat{\mathbf{w}}_{s-1}\Vert_1}$.
    It is still far from optimality to sample from $\mathbf{q}_s$
    due to $\hat{\mathbf{w}}_{s-1}\neq \mathbf{w}^\ast$.
    The error between $\hat{w}_{s-1,i}$ and ${w}^\ast_i$, $i\in[d]$,
    makes the estimators suffering large variances.
    Next we explain our sampling scheme.

    At any round $t\in\mathcal{T}$,
    we first sample a feature $x_{t,I_1}$ from $\{x_{t,1},\ldots,x_{t,d}\}$
    following $\mathbf{q}_s$,
    and then uniformly sample $k-1$ features denoted by $x_{t,I_2},\ldots,x_{t,I_k}$
    from $\{x_{t,1},\ldots,x_{t,d}\}\setminus \{x_{t,I_1}\}$ without replacement,
    in which $s=\sqrt{t}$.
    For simplicity,
    let $B_s=\{I_1,\ldots,I_k\}$,
    $\mathbb{P}[i\in B_s]$ be probability that the $i$-th feature is observed
    and $\mathbb{P}[i,j\in B_s]$
    be the probability that both the $i$-th feature and the $j$-th feature are observed.
    Initializing $\hat{\mathbf{w}}_0=\frac{1}{d}\mathbf{1}_d$.
    Next we construct unbiased estimators of ${\mathbf{x}}_t$ and ${\mathbf{x}}_t{\mathbf{x}}^\top_t$.
    Let $\hat{\mathbf{x}}_t=(\hat{x}_{t,1},\hat{x}_{t,2},\ldots,\hat{x}_{t,d})^\top\in\mathbb{R}^d$
    and $\mathbf{h}_t\in\mathbb{R}^{d\times d}$ satisfy
    \begin{equation}
    \label{eq:COLT2025:estimators}
    \begin{split}
        \forall i\in[d],\qquad\hat{x}_{t,i}=&\frac{x_{t,i}}{\mathbb{P}[i\in B_s]}\cdot\mathbb{I}_{i\in B_s},\quad
        h_t[i,i]=\frac{x^2_{t,i}}{\mathbb{P}[i\in B_s]}\cdot\mathbb{I}_{i\in B_s},\\
        \forall i\neq j\in[d],\quad h_t[i,j]=&
        \frac{x_{t,i}x_{t,j}}{\mathbb{P}[i,j\in B_s]}\cdot\mathbb{I}_{i,j\in B_s}.
    \end{split}
    \end{equation}
    Then we define unbiased estimators of $\mathbf{X}_{\mathcal{I}_s}$
    and $\mathbf{X}_{\mathcal{I}_s}\mathbf{X}^\top_{\mathcal{I}_s}$ as follows,
    $$
        \hat{\mathbf{X}}_{\mathcal{I}_s}=
        \left[\hat{\mathbf{x}}_1,\hat{\mathbf{x}}_{2^2},\hat{\mathbf{x}}_{3^2}\ldots,\hat{\mathbf{x}}_{s^2}\right]
        \in\mathbb{R}^{d\times s},\quad
        \mathbf{H}_{\mathcal{I}_s}=\sum^s_{\tau=1}\mathbf{h}_{\tau^2}\in\mathbb{R}^{d\times d}.
    $$
    Let $g_{d,k}=\frac{(d-1)(d-2)}{(k-1)(k-2)}$.
    Now we define the following time-variant Dantzig Selector ($\mathrm{DS}(\gamma_s)$)
    \begin{equation}
    \label{eq:ICML25:new_Dantzig_Selector}
    \mathrm{DS}(\gamma_s):
        \min_{\mathbf{w}\in\mathbb{R}^d}\Vert\mathbf{w}\Vert_1\quad
        \mathrm{s.t.}~\left\Vert \frac{1}{s}\hat{\mathbf{X}}_{\mathcal{I}_s}\mathbf{Y}_{\mathcal{I}_s}
        -\frac{1}{s}\mathbf{H}_{\mathcal{I}_s}\mathbf{w}\right\Vert_{\infty}\leq \gamma_s,
    \end{equation}
    where $\gamma_s$ is a time-variant threshold defined as follows
    \begin{equation}
    \label{eq:ICML25:r_s}
        \gamma_s=\left(\frac{8}{3}+2\sigma\right)\frac{g_{d,k}}{s}\ln\frac{d}{\delta}+
        \frac{6.9+1.2\sigma}{\sqrt{s}}\sqrt{\frac{d-1}{k-1}\ln\frac{d}{\delta}}+
        \frac{1}{s}\sqrt{3g_{d,k}\sum^s_{\tau=1}\Vert\Delta_{\tau-1}(S)\Vert^2_1\ln\frac{d}{\delta}},
    \end{equation}
    in which $\Delta_{\tau-1}(S)=\hat{\mathbf{w}}_{\tau-1}(S)-\mathbf{w}^\ast$.
    Denoted by $\hat{\mathbf{w}}_s$ the optimal solution of $\mathrm{DS}(\gamma_s)$.
    $\hat{\mathbf{X}}_{\mathcal{I}_s}\mathbf{Y}_{\mathcal{I}_s}$ can be computed incrementally
    in a space and per-round time complexity of $O(d)$.
    Besides,
    computing $\mathbf{H}_{\mathcal{I}_s}$ necessitates a space and per-round time complexity of $O(d^2)$.
    $\mathrm{DS}(\gamma_s)$ can be recast as a linear programming \citep{Candes2007The},
    and be solved in polynomial time.
    For instance,
    the interior point method \citep{Karmarkar1984A} requires time in $O(d^{3.5}L)$
    and the algorithm by \cite{Vaidya1989Speeding} requires time in $O(d^{2.5}L)$,
    in which $L$ is the number of bits in the input of the linear programming.

\subsection{Exploitation}

    Let $S_0\subseteq [d]$ be an arbitrary subset satisfying $\vert S_0\vert=k$.
    For $s\geq 1$,
    it can be proved that $\Vert \hat{\mathbf{w}}_s-\mathbf{w}^\ast\Vert_1\rightarrow 0$ in probability,
    suggesting a natural estimation of the true support set $S$, denoted by $S_s$,
    $$
        S_s\subseteq [d],\quad \mathrm{s.t.}~\vert S_s\vert=k,~
        \forall i\in S_s, j\in [d]\setminus S_s,\quad
        \left\vert \hat{w}_{s,i}\right\vert
        \geq \left\vert\hat{w}_{s,j}\right\vert.
    $$
    Before $S_s$ converges to $S$,
    we can trust $\hat{\mathbf{w}}_s(S_s)$ throughout the $s$-th epoch $\mathcal{T}_s$,
    that is,
    $\hat{\mathbf{w}}_s(S_s)$ can serve as a good predictor for all $\mathbf{x}_t$, $t\in \mathcal{T}_s$.
    The regret in $\mathcal{T}_s$ obviously depends on the convergence rate $\Vert\hat{\mathbf{w}}_s-\mathbf{w}^\ast\Vert_1$.
    There is a $s_2$, such that $S_s=S$ for any $s>s_2$ with a high probability.
    In this case,
    $\hat{\mathbf{w}}_s(S_s)$ is not the best predictor for all $\mathbf{x}_t$,
    since $\Vert\hat{\mathbf{w}}_s-\mathbf{w}^\ast\Vert_1$ decays with $s$,
    not $t$ (the number of observed examples).
    To address this issue,
    we can solve an ordinary online linear regression
    on the observations $\{({\mathbf{x}}_t(S_s),y_t)\}_{t\in \mathcal{T}_s}$,
    and use an online learning algorithm to learn $\mathbf{w}^\ast$.
    Since the square loss function is exp-concave,
    we will use ONS
    \footnote{If we use online gradient descent \citep{Zinkevich2003Online},
    then the algorithm only requires a per-round time complexity in $O(k)$,
    but incurs larger regret.}.
    The per-round time complexity is $O(k^2)$.

    As $s_2$ is unknown,
    we will update $\hat{\mathbf{w}}_s(S_s)$ during the epoch $\mathcal{T}_s$ for all $s\geq 1$.
    We first define a feasible set $\mathcal{W}_s$ as follows
    $$
        \forall s\geq 1,\quad
        \mathcal{W}^t_s=\left\{ \mathbf{w}\in\mathbb{R}^d: \mathrm{Supp}(\mathbf{w})\subset S_s,
        t\in\mathcal{T}_s,\langle \mathbf{w},{\mathbf{x}}_t(S_s)\rangle\leq 1\right\},\quad
        \mathcal{W}_s=\cap_{t\in\mathcal{T}_s}\mathcal{W}^t_s.
    $$
    It is obvious that $\hat{\mathbf{w}}_s(S_s)\in \mathcal{W}_s$.
    Initializing a parameter vector $\bar{\mathbf{w}}_{s^2+1}(S_s)=\hat{\mathbf{w}}_s(S_s)$,
    and a covariance matrix $\mathbf{A}_s=\varepsilon\cdot\mathbf{I}_{k\times k}$.
    At each round $t\in\mathcal{T}_s$,
    the prediction is given by $\hat{y}_t=\langle\bar{\mathbf{w}}_t(S_s),\mathbf{x}_t\rangle$.
    Let $\mathbf{g}_t=2(\hat{y}_t-y_t){\mathbf{x}}_t(S_s)$ be the gradient.
    The parameters are updated as follows
    \footnote{{\color{red}
    There is a minor mistake regarding the projection operation in our paper published at COLT 2025.
    It has been addressed in the current version.
    We have also updated the pseudo-code of Algorithm \ref{alg:ICML2025:DS-OSLRC}.
    Our theoretical analyses, as well as all theoretical bounds, remain unaffected by those changes}.},
    $$
        \mathbf{A}_t = \mathbf{A}_{t-1}+\rho\mathbf{g}_t\mathbf{g}^\top_t,~
        \bar{\mathbf{w}}_{t+1}(S_s)
        =\mathop{\arg\min}_{\mathbf{w}\in\mathcal{W}^{t+1}_s}
        \left\Vert \mathbf{w}-\bar{\mathbf{w}}_t(S_s)+\mathbf{A}^{-1}_t\mathbf{g}_t\right\Vert_{\mathbf{A}_t}
        :=\mathcal{P}^t_s
        (\bar{\mathbf{w}}_t(S_s)-\mathbf{A}^{-1}_t\mathbf{g}_t),
    $$
    in which $\mathcal{P}^t_s(\cdot)$ is a projector operator \citep{Luo2016Efficient} defined as follows,
    $$
        \mathcal{P}^t_s\left(\mathbf{w}\right)
        =\mathbf{w}
        -\frac{\tau(\langle \mathbf{w},\mathbf{x}_{t+1}(S_s)\rangle)}
        {\langle\mathbf{x}_{t+1}(S_s),\mathbf{A}^{-1}_t\mathbf{x}_{t+1}(S_s)\rangle}
        \mathbf{A}^{-1}_t\mathbf{x}_{t+1}(S_s),\quad
        \tau(y)=\mathrm{sign}(y)\cdot \max\{\vert u\vert-1,0\}.
    $$
    If $S_s=S_{s-1}$,
    then resetting $\bar{\mathbf{w}}_{s^2+1}=\hat{\mathbf{w}}_s(S_s)$ and
    $\mathbf{A}_{s^2}=\varepsilon\cdot\mathbf{I}_{k\times k}$ must increase the regret.
    To address this issue,
    we redefine the initial configurations as follows
    $$
    \left\{
    \begin{array}{l}
        \bar{\mathbf{w}}_{s^2+1}(S_s)=\hat{\mathbf{w}}_s(S_s),~~\quad
        \mathbf{A}_{s^2}=\varepsilon\cdot\mathbf{I}_{k\times k},\quad \mathrm{if}~S_s\neq S_{s-1},\\
    \bar{\mathbf{w}}_{s^2+1}(S_s)=\bar{\mathbf{w}}_{s^2}(S_{s-1}),~
        \mathbf{A}_{s^2}=\mathbf{A}_{s^2-1},\quad\mathrm{otherwise}.
    \end{array}
    \right.
    $$

\subsection{Adaptive Parameters-tuning}
\label{sec:COLT2025:adaptive_parameter_tuning}

    By \eqref{eq:ICML25:r_s},
    $\gamma_s$ requires the prior information of $\mathbf{w}^\ast$.
    It is necessary to construct an estimator of $\gamma_s$, denoted by $\hat{\gamma}_s$.
    By theoretical analyses,
    we must ensure that $\mathrm{DS}(\hat{\gamma}_s)$ has a solution $\hat{\mathbf{w}}_s$
    satisfying $\Vert \hat{\mathbf{w}}_s\Vert_1\leq \Vert \mathbf{w}^\ast\Vert_1$.
    To this end,
    it is sufficient to make $\hat{\gamma}_s$ be a tight upper bound of $\gamma_s$.
    By the definition of $\gamma_s$,
    it is further reduced to provide a tight upper bound on
    $\frac{1}{s}\sqrt{3g_{d,k}\sum^s_{\tau=1}\Vert\Delta_{\tau-1}(S)\Vert^2_1\ln\frac{d}{\delta}}$.
    Let $\nu_s$ be an estimator such that
    $\nu_s\geq \frac{1}{s}\sqrt{3g_{d,k}\sum^s_{\tau=1}\Vert\Delta_{\tau-1}(S)\Vert^2_1\ln\frac{d}{\delta}}$.
    Then $\hat{\gamma}_s$ can be defined by
    \begin{equation}
    \label{eq:ICML25:hat_gamma_s}
        \hat{\gamma}_s=\left(\frac{8}{3}+2\sigma\right)\frac{g_{d,k}}{s}\ln\frac{d}{\delta}+
        \frac{6.9+1.2\sigma}{\sqrt{s}}\sqrt{\frac{d-1}{k-1}\ln\frac{d}{\delta}}+
        \nu_s.
    \end{equation}
    The challenge is how to define $\nu_s$.
    Note that if $\gamma_s$ is known,
    we can obtain the ideal convergence rate of $\Vert\Delta_s(S)\Vert_1$.
    To be specific,
    we first solve \eqref{eq:COLT2025:high_level_upper_bound_ideal} by an induction method,
    and obtain the ideal convergence rate.
    Then we use it to define $\nu_s$.
    We will verify whether $\hat{\gamma}_s\geq \gamma_s$ for all $s\geq 1$,
    and concurrently establish the real convergence rate of $\Vert\Delta_s(S)\Vert_1$ by an induction method.
    The real convergence rate nearly matches the ideal one.
    Let $\delta\in(0,1)$,
    $\mu_1=\frac{9}{9-2\sqrt{3}}$, and
    \begin{equation}
    \label{eq:COLT2025:s0-s1_DS-OSLRC}
        \mu_2=\frac{1}{1-\frac{\sqrt{6}}{9\sqrt{\frac{d-2}{k-2}\ln\frac{d^2}{\delta}}}},\quad
        s_0=\frac{24^2 k^2g_{d,k}}{\delta^4_S}\ln\frac{d^2}{\delta},\quad
        s_1=\frac{24^2 k^2g_{d,k}}{\delta^4_S}\frac{d-2}{k-2}
        \ln\left(\frac{d}{\delta}\right)\ln\frac{d^2}{\delta}.
    \end{equation}
    We consider the problems where the number of examples is sufficient large,
    i.e., $T> (s_1+1)^2$.
    Let
    \begin{equation}
    \label{eq:COLT2025:a1-a5_DS-OSLRC}
    \left\{
    \begin{split}
        a_1=&\left(\frac{64}{3}+\frac{32}{3}\sigma\right)\ln\frac{d}{\delta},\quad
        a_2=\frac{16\left(6.9+1.2\sigma\right)}{3}\sqrt{\ln\frac{d}{\delta}},\quad
        a_3=\frac{8}{3}\sqrt{3\ln\frac{d}{\delta}},\\
        a_4=&\delta^2_S\frac{a_1}{k}
            +24a_2\sqrt{\frac{k-2}{d-2}\ln\frac{d^2}{\delta}}
            +4a_3\left(24\sqrt{\ln\frac{d^2}{\delta}}+\frac{\delta^2_S}{k\sqrt{g_{d,k}}}\right),\\
        a_5=&\frac{9}{9-2\sqrt{3}}\left(\delta^2_S\frac{8+4\sigma}{9k}
            +\frac{32}{\sqrt{3}}
            +\frac{4\sqrt{3}\delta^2_S}{9k\sqrt{g_{d,k}\ln{\frac{d^2}{\delta}}}}\right)
            +a_2+\frac{2\sqrt{3}a_2}{9-2\sqrt{3}}\sqrt{\frac{k-2}{(d-2)\ln\frac{d^2}{\delta}}}.
    \end{split}
    \right.
    \end{equation}
    Then $\nu_s$ is defined as follows
    $$
        \nu_s=
        \left\{
        \begin{array}{ll}
        \frac{2}{\sqrt{s}}\sqrt{3 g_{d,k}\ln\frac{d}{\delta}},&s\leq [1,s_0],\\
        \frac{1}{s}\sqrt{3g_{d,k}\ln\frac{d}{\delta}}
        \left(\frac{48k}{\delta^2_S}\sqrt{g_{d,k}\ln\frac{d^2}{\delta}}+2\right),&s=s_0+1,\\
        \frac{s_0+1}{s}\nu_{s_0+1}+\frac{1}{s}\sqrt{3g_{d,k}\ln\frac{d}{\delta}}\cdot
        \frac{\mu_1 a_4}{\delta^4_S}\sqrt{\sum^{s-1}_{\tau=s_0+1}\frac{k^4g^2_{d,k}}{\tau^2}},&s\in(s_0+1,s_1],\\
        \frac{1}{s}\sqrt{3g_{d,k}\ln\frac{d}{\delta}}\left(\frac{48k}{\delta^2_S}\sqrt{g_{d,k}\ln\frac{d^2}{\delta}}+
        2+\frac{\mu_1a_4k^2g_{d,k}}{\delta^4_S}\sqrt{\sum^{s_1}_{\tau=s_0+1}\frac{1}{\tau^2}}\right),&s=s_1+1,\\
        \frac{s_1+1}{s}\nu_{s_1+1}+
        \frac{1}{s}\sqrt{3g_{d,k}\ln\frac{d}{\delta}}
        \cdot\frac{\mu_2a_5}{\delta^2_S}\sqrt{\sum^{s-1}_{\tau=s_1+1}\frac{k^2(d-1)}{\tau(k-1)}},&s>s_1+1.
        \end{array}
        \right.
    $$
    Now we can solve $\mathrm{DS}(\hat{\gamma}_s)$ and obtain the solution $\hat{\mathbf{w}}_s$.
    We name this algorithm DS-OSLRC (Dantzig Selector for OSLR with Compatibility condition)
    and show the pseudo-code in Algorithm \ref{alg:ICML2025:DS-OSLRC}.

    \begin{algorithm2e}[!t]
        \caption{\small{DS-OSLRC}}
        \LinesNumbered
        \footnotesize
        \label{alg:ICML2025:DS-OSLRC}
        \KwIn{$k$, $d$, $\sigma$, $\delta_S$, $\rho$, $\delta$}
        Initialize $\mathbf{A}_1=\varepsilon\cdot \mathbf{I}_{k\times k}$,
        $\hat{\mathbf{w}}_0=\frac{1}{d}\mathbf{1}_d$,
        $\hat{\mathbf{x}}_{\mathcal{I}_0}=\mathbf{0}_d$,
        $\mathbf{H}_{\mathcal{I}_0}=\mathbf{0}_{d\times d}$, $S_0$\;
        \For{$t=1,2,\ldots,T$}
        {
            \If{$t\in\mathcal{T}$}
            {
                $s=\sqrt{t}$\;
                Obtain $B_s\subseteq [d]$ from SAMPLING($k,d,\hat{\mathbf{w}}_{s-1}$)\;
                Output the prediction $\hat{y}_t=\langle \hat{\mathbf{w}}_{s-1}(B_s),\mathbf{x}_t(B_s)\rangle$\;
                Compute $\hat{\mathbf{X}}_{\mathcal{I}_s}\mathbf{Y}_{\mathcal{I}_s}=
                \hat{\mathbf{X}}_{\mathcal{I}_{s-1}}\mathbf{Y}_{\mathcal{I}_{s-1}}+
                \hat{\mathbf{x}}_{s^2}y_{s^2}$
                where $\hat{\mathbf{x}}_{s^2}$
                is computed by combining \eqref{eq:COLT2025:estimators} with \eqref{eq:COLT2025:probabilities}\;
                Compute $\mathbf{H}_{\mathcal{I}_s}=\mathbf{H}_{\mathcal{I}_{s-1}}+\mathbf{h}_{s^2}$
                where $\mathbf{h}_{s^2}$
                is computed by combining \eqref{eq:COLT2025:estimators} with \eqref{eq:COLT2025:probabilities}\;
                Compute $\hat{\gamma}_s$ by \eqref{eq:ICML25:hat_gamma_s}\;
                Obtain $\hat{\mathbf{w}}_s$ by solving $\mathrm{DS}(\hat{\gamma}_s)$ and select $S_s$\;
                \If{$S_s\neq S_{s-1}$}
                {
                    Initialize
                    $\tilde{\mathbf{w}}_{s^2+1}(S_s)=\hat{\mathbf{w}}_s(S_s)$\;
                    Initialize $\mathbf{A}_{s^2}=\varepsilon\cdot\mathbf{I}_{k\times k}$\;
                }
                \Else
                {
                    Initialize $\tilde{\mathbf{w}}_{s^2+1}(S_s)=\tilde{\mathbf{w}}_{s^2}(S_{s-1})$\;
                    Initialize $\mathbf{A}_{s^2}=\mathbf{A}_{s^2-1}$\;
                }
            }
            \Else
            {
                Compute $\bar{\mathbf{w}}_t(S_s)=
                \mathcal{P}^{t-1}_s\left(\tilde{\mathbf{w}}_t(S_s)\right)$\;
                Output the prediction $\hat{y}_t=\langle \bar{\mathbf{w}}_t(S_s),\mathbf{x}_t(S_s)\rangle$\;
                Compute $\mathbf{g}_t=2(\hat{y}_t-y_t){\mathbf{x}}_t(S_s)$\;
                Update $\mathbf{A}_t=\mathbf{A}_{t-1}+\rho \cdot \mathbf{g}_t\mathbf{g}^\top_t$\;
                Compute $\tilde{\mathbf{w}}_{t+1}(S_s)=\bar{\mathbf{w}}_t(S_s)-\mathbf{A}^{-1}_t\mathbf{g}_t$\;
            }
        }
    \end{algorithm2e}

    \begin{algorithm2e}[!t]
        \caption{\small{SAMPLING}}
        \LinesNumbered
        \footnotesize
        \label{alg:ICML2025:SAMPLING}
        \KwIn{$k$, $d$, $\mathbf{w}$}
        Initialize $B=\emptyset$\;
        Compute $\mathbf{q}=\frac{\mathbf{w}}{\Vert\mathbf{w}\Vert_1}$\;
        Sample $I_t\in [d]$ following $\mathbf{q}$\;
        Update $B=B\cup\{I_t\}$\;
        \For{$r=1,\ldots,k-1$}
        {
            Sampling $I_r\in[d]\setminus B$ uniformly\;
            $B=B\cup \{I_{r+1}\}$\;
        }
        Return $B$\;
    \end{algorithm2e}

\section{Main Results}

    In this section,
    we give the $\ell_1$-norm error of $\hat{\mathbf{w}}_s(S)$
    and the regret bound of DS-OSLRC.

\subsection{$\ell_1$-norm Error Bound}

    Lemma \ref{lemma:ICML25:sampling_probability:two_varianbles}
    gives the sampling probabilities,
    making it possible to compute $\hat{\mathbf{x}}_{s^2}$ and $\mathbf{h}_{s^2}$.
    The analysis is non-trivial,
    given that DS-OSLRC samples attributes without replacement.
    We do not prove Lemma \ref{lemma:ICML25:sampling_probability:two_varianbles},
    but instead, prove a more general version, namely,
    Lemma \ref{lemma:ICML25:sampling_probability} in Appendix \ref{sec:COLT2025:technical_lemmas}.
    \begin{Mylemma}
    \label{lemma:ICML25:sampling_probability:two_varianbles}
        For any $s\geq 1$,
        let $B_s$ be the output of Algorithm \ref{alg:ICML2025:SAMPLING}.
        \begin{equation}
        \label{eq:COLT2025:probabilities}
        \begin{split}
            \forall i\in[d],\quad\mathbb{P}\left[i\in B_s\right]
            =&\frac{d-k}{d-1}q_{s,i}+\frac{k-1}{d-1},\\
            \forall i\neq j\in[d],\quad\mathbb{P}\left[i,j\in B_s, i\neq j\right]
            =&\frac{(k-1)(k-2)}{(d-1)(d-2)}+ \frac{(k-1)(d-k)}{(d-1)(d-2)}\cdot(q_{s,i}+q_{s,j}).
        \end{split}
        \end{equation}
    \end{Mylemma}
    Next we provide the $\ell_1$-norm error of $\hat{\mathbf{w}}_s(S)$,
    which serves as the foundation for regret analysis.

    \begin{Mylemma}[Estimation Error]
    \label{lemma:estimator_error:DS-OSLRC}
        Let $S=\mathrm{Supp}(\mathbf{w}^\ast)$,
        $0<\delta <1$ and $\hat{\mathbf{w}}_0=\frac{1}{d}\mathbf{1}_d$.
        If $3\leq k\leq d-3$, $T>(s_1+1)^2$,
        Assumptions \ref{ass:ICML2025:Realizable assumption}-\ref{ass:COLT2025:sparsity} hold,
        and $\mathbf{X}_{\mathcal{I}_s}$ satisfies the $(\delta_S,S,1)$-compatibility condition
        for all $s\geq 1$,
        then with probability at least $1-\sqrt{T}(6+\log_{1.5}\frac{d+k}{k-2})\delta$,
        DS-OSLRC guarantees
        $$
            \forall s\geq 1,~
            \Vert \Delta_s(S)\Vert_1 \leq
            \left\{
            \begin{array}{ll}
            \frac{16+8\sigma}{\delta^2_S}\frac{kg_{d,k}}{s}\ln\frac{d}{\delta}
            +\frac{\left(26+4.8\sigma\right)k}{\delta^2_S}\sqrt{\frac{(d-1)\ln\frac{d}{\delta}}{s(k-1)}}+
            \frac{22k}{\delta^2_S}\sqrt{\frac{g_{d,k}}{s}\ln\frac{d}{\delta}},&s\leq [1,s_0]\\
            \frac{9}{9-2\sqrt{3}}\frac{a_4}{\delta^4_S}\frac{k^2g_{d,k}}{s},&s\in(s_0,s_1]\\
            \mu_2\cdot\frac{a_5}{\delta^2_S}\sqrt{\frac{k^2(d-1)}{s(k-1)}},&s>s_1,\\
            \end{array}
            \right.
        $$
        in which $\mu_2$, $s_0$ and $s_1$ follow \eqref{eq:COLT2025:s0-s1_DS-OSLRC},
        $a_4$ and $a_5$ follow \eqref{eq:COLT2025:a1-a5_DS-OSLRC}.
    \end{Mylemma}

    By Lemma \ref{lemma:estimator_error:DS-OSLRC},
    it is easy to establish the $\ell_1$-norm error bound of $\hat{\mathbf{w}}_s$.
    To be specific,
    by Lemma \ref{lemma:ICML25:Dantzig2005},
    we have $\Vert \Delta_s\Vert_1\leq 2\Vert \Delta_s(S)\Vert_1$.
    We can also obtain the $\ell_2$-norm error bound
    by the inequality $\Vert \Delta_s\Vert_2\leq \Vert\Delta_s\Vert_1$.
    It is worth mentioning that by the restricted eigenvalues condition \citep{Bickel2009Simultaneous},
    we can obtain a tighter $\ell_2$-norm error bound.

    The Algorithm 3 in \citep{Ito2017Efficient} attains $\Vert\Delta_s\Vert_1=O(s^{-\frac{1}{4}})$
    (please refer to Lemma 14 in original paper),
    while our convergence rate is $\tilde{O}(s^{-\frac{1}{2}})$.
    The algorithm in \citep{Kale2017Adaptive} attains
    $
        \Vert\Delta_s\Vert_1 =O(\frac{d}{k_0}\sqrt{\frac{d}{k_0}\frac{k^2}{s}\ln\frac{d}{\delta}}).
    $
    For $s> s_1$,
    DS-OSLRC improves the convergence rate by a factor of $O(d)$.
    In the full information setting where algorithms can observe $\mathbf{x}_s$ for all $s=1,\ldots,T$,
    the OLin-LASSO algorithm \citep{Yang2023Online} attains $\Vert\Delta_s\Vert_1 =\tilde{O}(\sqrt{k/s})$.
    Our convergence rate only deteriorates by a factor of $O(\sqrt{d})$.

\subsection{Regret Bounds of DS-OSLRC}

    \begin{theorem}[Regret Bound w.r.t. $\mathbf{w}^\ast$]
    \label{thm:ICML2025:regret_bound_DS-OSLRC}
        Let $\varepsilon=k$, $\delta\in(0,1)$ and
        $$
            Y_{\delta}=1+\sigma\sqrt{2\ln\frac{1}{\delta}},\quad
            \rho=\frac{1}{2(1+Y_{\delta})^2},\quad
            s_2=4\frac{(\mu_2a_5)^2}{\delta^4_S}\frac{k^2(d-1)}{\min_{i\in S}\vert w^\ast_i\vert^2(k-1)}.
        $$
        Under the same assumptions in Lemma \ref{lemma:estimator_error:DS-OSLRC}
        and the condition
        $$
            T>(s_2+1)^2>(s_1+1)^2,
        $$
        with probability at least $1-(T+\sqrt{T}(6+\log_{1.5}\frac{d+k}{k-2})+1)\delta$,
        the regret of DS-OSLRC satisfies
        \begin{align*}
            \mathrm{Reg}(\mathbf{w}^\ast)
            \leq&4\sqrt{T}
            +\frac{2(\mu_2a_5)^2}{\delta^4_S}\cdot\frac{k^3(d-1)\ln\left(4(1+Y_{\delta})^2T+1\right)}
            {(k-1)\min_{i\in S}\vert w^\ast_i\vert^2}+\\
            &\frac{22a^2_4}{\delta^8_S}k^4g^2_{d,k}\ln\frac{(d-2)\ln\frac{d}{\delta}}{k-2}
            +\frac{2(2\mu_2a_5)^4}{\delta^8_S}\cdot
            \frac{k^4(d-1)^2}{\min_{i\in S}\vert w^\ast_i\vert^2(k-1)^2}
            +O(1),
        \end{align*}
        where $\mu_2$ follows \eqref{eq:COLT2025:s0-s1_DS-OSLRC},
        $a_4$ and $a_5$ follow \eqref{eq:COLT2025:a1-a5_DS-OSLRC}
        and $O(1)$ hides the lower order constant terms.
    \end{theorem}

    For the sake of simplicity,
    we only consider the scenario where $T$ is adequately large.
    It is easy to analyze the regret in the cases where $T\leq (s_2+1)^2$, or $s_2<s_1$.
    If $t\in\mathcal{T}$,
    DS-OSLRC solves a linear programming.
    The time complexity is denoted by $O(\mathrm{LP}_d)=O(\mathrm{poly}(d))$.
    Otherwise,
    the per-round time complexity is $O(k^2)$.
    Thus the average per-round time complexity is only $O(k^2+\frac{\mathrm{LP}_d}{\sqrt{T}})$.
    Let $\delta=\Theta(\frac{1}{T})$.
    Then DS-OSLRC achieves a
    $O\left(\sqrt{T}+\frac{k^2d^2}{\delta^8_Sh(\mathbf{w}^\ast)^2}\ln^2\frac{dT}{\delta}\right)$ regret bound
    within a time complexity of $O(\mathrm{poly}(d))$ per-iteration,
    making OSLR tractable.
    Note that if $k\ll d$ and $\sigma\leq 1$,
    then the numerical factor on the dominated term is
    $2(2\mu_2a_5)^4\ln^{-2}\frac{dT}{\delta}\approx 2(2a_2)^4\ln^{-2}\frac{dT}{\delta}
    \approx 102^4$.

    \cite{Ito2017Efficient} proposed two algorithms denoted by \textbf{alg2} and \textbf{alg3}, for OSLR,
    both of which enjoy an expected regret of $O(\sqrt{kT}+\mathrm{poly}(d,k))$ in time $O(d)$ per-iteration.
    Under the assumptions including (i) realizable,
    (ii) $\mathbf{x}_t\sim\mathcal{D}_{\mathbf{x}}$ for all $t\in[T]$,
    (iii) the features in $\mathbf{x}_t$ are linearly independent for all $t\in[T]$,
    and (iv) bounded noises,
    the expected regret of \textbf{alg2} satisfies
    $$
        \mathbb{E}[\mathrm{Reg}(\mathbf{w}^\ast)]
        \leq8\sqrt{kT}+\sum_{i\in S}\frac{8192^2d^4(d-1)^4}{\sigma^8_d\vert w^\ast_i\vert^7k^4(k-1)^4} +O(1)
        =O\left(\sqrt{kT}+\frac{d^8}{\sigma^8_dk^7h(\mathbf{w}^\ast)^7}\right),
    $$
    where $\sigma^2_d$ is the smallest eigenvalue of
    $\mathbb{E}_{\mathbf{x}\sim\mathcal{D}_{\mathbf{x}}}[\mathbf{x}\mathbf{x}^\top]$.
    It must be $\sigma_d \leq \delta_S$.
    The linear independence of features condition is stronger than the compatibility condition.
    What's more, the regret bound is much worse than ours
    w.r.t. the dependence on $d$, $\min_{i\in S}\vert w^\ast_i\vert$ and $T$.

    Under the same assumptions (except for the bounded noises and i.i.d. instances) with our algorithm,
    the expected regret of \textbf{alg3} satisfies
    \begin{align*}
        \mathbb{E}[\mathrm{Reg}(\mathbf{w}^\ast)]
        \leq&8\sqrt{kT}+
        \sum_{i\in S}\frac{128^2\cdot 36^8d^4(d-1)^4}{\delta^8_S\vert w^\ast_i\vert^7(k-1)^4}+O(1)
        =O\left(\sqrt{kT}+\frac{d^8}{\delta^8_Sk^3h(\mathbf{w}^\ast)^7}\right).
    \end{align*}
    The regret bound is also much worse than ours
    w.r.t. the dependence on $d$, $\min_{i\in S}\vert w^\ast_i\vert$ and $T$.
    Besides,
    the constant factor is also much larger than ours.

    \begin{MyRemark}
        We can also analyze the regret w.r.t. the best $k$-sparse linear regressor.
        To be specific,
        by the regret analysis in \citep{Kale2017Adaptive},
        we obtain,
        with probability at least $1-\delta$,
        \begin{equation}
        \label{eq:ICML2025:regret_w_ast_any_w}
            \max_{\mathbf{w}\in\{\mathbf{v}\in\mathbb{R}^d:\Vert\mathbf{v}\Vert_0\leq k\}}
            \sum^T_{t=1}[\ell(\langle \mathbf{w}^\ast,\mathbf{x}_t\rangle,y_t)
            -\ell(\langle \mathbf{w},\mathbf{x}_t\rangle,y_t)]
            \leq2k\sigma^2+ 4k\sigma^2\ln\frac{d}{\delta}.
        \end{equation}
        By incorporating the upper bound on $\mathrm{Reg}(\mathbf{w}^{\ast})$,
        i.e., Theorem \ref{thm:ICML2025:regret_bound_DS-OSLRC},
        we can obtain the regret bound w.r.t. the best $k$-sparse linear regressor.
    \end{MyRemark}

\subsection{Discussion of Lower Bounds and Upper Bounds}

    Next we compare our upper bounds with lower bounds.
    The results are summarized in Table \ref{tab:COLT25:Response}.

    The lower bound on estimator error is for sparse linear regression \citep{Candes2013How}.
    OSLR is an online, partial information variant of sparse linear regression.
    Therefore, this lower bound is applicable to OSLR.
    To be specific,
    any algorithm for OSLR can return an estimator of $\mathbf{w}^\ast$,
    denoted by $\hat{\mathbf{w}}_T$,
    which can serve as a solution of sparse linear regression.
    Note that the original lower bound established in \citep{Candes2013How}
    applies to any design matrix $\mathbf{X}$.
    We adapt the general bound to Gaussian designs whose entries are i.i.d. $\mathcal{N}(0,1)$.
    As Gaussian designs satisfy the compatibility condition,
    the lower bound remains applicable to OSLR under the assumptions in our paper.
    There is a gap of $O\left(\sqrt{d}\right)$ between the lower and upper bounds.
    We conjecture that the lower bound can be refined,
    given the partial information constraint inherent to OSLR.
    Under the condition that all of the attributes per instance can be observed,
    the OLin-LASSO algorithm \citep{Yang2023Online} indeed attains the lower bound.

    The lower bound on the regret remains unestablished.
    We can easily obtain a lower bound on regret by the lower bound on estimation error.
    Let $\mathbf{X}$ be a Gaussian design.
    Consider any algorithm that maintains an estimator $\hat{\mathbf{w}}_s$,
    and produces a prediction using $\hat{\mathbf{w}}_s(S_s)$
    where $S_s\subseteq[d]$ and $\vert S_s\vert\leq k$ at each round $s\geq 1$.
    The expected regret satisfies
    \begin{align*}
    \mathbb{E}[\mathrm{Reg}(\mathbf{w}^\ast)]
    =\sum^T_{s=1}\mathbb{E}[\langle \hat{\mathbf{w}}_s(S_s)-\mathbf{w}^{\ast},\mathbf{x}_s\rangle^2]
    =\Omega\left(\sum^T_{s=1}\frac{1}{s}k\ln\frac{d}{k}\right)
    =&\Omega\left(k\ln(T+1)\ln\frac{d}{k}\right),
    \end{align*}
    where the expectation is w.r.t. the noises and $\mathbf{X}$.

    \begin{table}[!t]
      \centering
      \begin{tabular}{l|r|r}
      \toprule
        lower bound on estimation error & upper bound for OSLR  & upper bound for $(k,k_0,d)$-OSLR  \\
      \midrule
        $\Omega\left(\sqrt{\frac{k}{T}\log{\frac{d}{k}}}\right)$
        & \multirow{2}{*}{$O\left(\sqrt{\frac{kd}{T}\log{\frac{dT}{\delta}}}\right)$}
        & \multirow{2}{*}{$O\left(\sqrt{\frac{k^2d}{Tk_0}\log{\frac{dT}{\delta}}}\right)$}\\
        \citep{Candes2013How}&&\\
      \midrule
        lower bound on regret& upper bound for OSLR & upper bound for $(k,k_0,d)$-OSLR  \\
      \midrule
        $\Omega\left(k\ln(T)\ln\frac{d}{k}\right)$
        & $O\left(\sqrt{T}+k^2d^2\ln^2\frac{dT}{\delta}\right)$
        & $O\left(\frac{k^2d}{k_0}(\frac{d}{k_0}+\ln{T})\ln\frac{dT}{\delta}\right)$\\
      \bottomrule
      \end{tabular}
      \caption{Lower bounds and upper bounds on estimation error and regret. The estimation error is
      $\min_{\hat{\mathbf{w}}}\max_{\mathbf{w}^\ast}\mathbb{E}[\Vert\hat{\mathbf{w}}-\mathbf{w}^\ast\Vert_2]$.}
      \label{tab:COLT25:Response}
    \end{table}

\section{Conclusion and Future Work}

    In this paper,
    we have proposed a new polynomial-time algorithm for OSLR
    that significantly improves previous regret bounds
    in terms of both problem-dependent parameters and constant factors under some mild assumptions.
    Notably,
    we assume the data matrix satisfies the compatibility condition that is less restrictive than
    the linear independence of features condition and RIP utilized in prior work.
    We further extend the algorithm to $(k,k_0,d)$-OSLR
    and improve previous regret bounds.

    Our work opens several directions for future research.
    The first one is to develop more efficient algorithms for OSLR
    that can avoid solving a linear programming
    while maintaining the regret bound under the same assumptions.
    The second one is to establish more efficient algorithms for $(k,k_0,d)$-OSLR.
    Our algorithm requires solving a linear programming at each round
    (please refer to Appendix \ref{sec:COLT2025:extension}).
    The third one is to establish tight lower bounds on both estimation error and regret.

\acks{This work is supported by the National Natural Science Foundation of China
under grants No. 62236003 and 62076181.
We also appreciate all the anonymous reviewers for their valuable comments and constructive suggestions.
}


\bibliography{OSLR}

\newpage

\appendix


\section{Extension to OSLR with Additional Observations}
\label{sec:COLT2025:extension}

    In this section,
    we extend DS-OSLRC to $(k,k_0,d)$-OSLR \citep{Foster2016Online},
    which is also called proper online sparse linear regression (POSLR) \citep{Kale2017Adaptive}.
    At each round $s\geq 1$,
    the learner chooses $k$ attributes from $\mathbf{x}_s$ and makes a prediction $\hat{y}_s$.
    Then the adversary gives the true output $y_s$.
    After that the learner can observe another $k_0>1$ attributes.
    In this work,
    we consider the case $k_0\geq 3$ and $k_0=O(k\ln{d})$.
    We will propose a new algorithm for $(k,k_0,d)$-OSLR,
    which can improve the previous regret bounds under weaker assumptions.

    Let $\hat{\mathbf{w}}_0=\frac{1}{d}\mathbf{1}_d$,
    and $S_0\subseteq [d]$ be an arbitrary subset satisfying $\vert S_0\vert=k$.
    For each $s\geq 1$,
    let $S_s\subseteq [d]$ follow the definition in DS-OSLRC.
    At the beginning of the $s$-th round,
    we choose $\mathbf{x}_s(S_{s-1})$
    and output $\hat{y}_s=\langle \hat{\mathbf{w}}_{s-1}(S_{s-1}),\mathbf{x}_s(S_{s-1})\rangle$.
    Given $y_s$,
    we choose $k_0$ attributes from $\{x_{s,i}:i\notin S_{s-1}\}$.
    Let $d'=d-k$.
    We construct $\hat{\mathbf{w}}'_{s-1}\in\mathbb{R}^{d'}$
    by removing the elements $\hat{w}_{s,i}$ from $\hat{\mathbf{w}}_{s-1}$ for all $i\in S_{s-1}$.
    Then we define
    $\bar{\mathbf{q}}_s=\frac{1}{\Vert \hat{\mathbf{w}}'_{s-1}\Vert_1}(\vert\hat{w}'_{s-1,1}\vert,\ldots,\vert\hat{w}'_{s-1,d'}\vert)$,
    and send $(k_0,d',\hat{\mathbf{w}}'_{s-1})$ into Algorithm \ref{alg:ICML2025:SAMPLING}
    that returns a set $B'_s\subseteq[d]$.
    Let $B_s=B'_s\cup S_{s-1}$.
    We can construct $\mathbf{H}_{\mathcal{I}_s}$ and
    $\hat{\mathbf{X}}_{\mathcal{I}_s}\mathbf{Y}_{\mathcal{I}_s}$ following DS-OSLRC,
    in which $\mathcal{I}_s=\{1,2,\ldots,s\}$.
    Let $g_{d',k_0}=\frac{(d'-1)(d'-2)}{(k_0-1)(k_0-2)}$,
    and $a_1,a_2,a_3$ follow the definition in \eqref{eq:COLT2025:a1-a5_DS-OSLRC}.
    Let
    \begin{equation}
    \label{eq:COLT2025:a1-a5_DS-POSLRC}
    \left\{
    \begin{split}
        \mu_1=&\frac{9}{9-2\sqrt{3}},\quad
        \mu_2=\frac{1}{1-\frac{\sqrt{6}}{9\sqrt{\frac{d'-2}{k_0-2}\ln\frac{d^2}{\delta}}}},\\
        s_0=&\frac{24^2 \cdot k^2g_{d',k_0}}{\delta^4_S}\ln\frac{d^2}{\delta},\quad
        s_1=\frac{24^2 \cdot k^2g_{d',k_0}}{\delta^4_S}\frac{d'-2}{k_0-2}
        \ln\left(\frac{d}{\delta}\right)\ln\frac{d^2}{\delta},\\
        a_4=&\delta^2_S\frac{a_1}{k}
            +24a_2\sqrt{\frac{k_0-2}{d'-2}\ln\frac{d^2}{\delta}}
            +4a_3\left(24\sqrt{\ln\frac{d^2}{\delta}}+\frac{\delta^2_S}{k\sqrt{g_{d',k_0}}}\right),\\
        a_5=&\frac{36}{9-2\sqrt{3}}\left(\delta^2_S\frac{2+\sigma}{9k}
            +\frac{8}{\sqrt{3}}
            +\frac{\frac{\sqrt{3}}{9}\delta^2_S}{k\sqrt{g_{d',k_0}\ln{\frac{d^2}{\delta}}}}\right)
            +a_2+\frac{2\sqrt{3}a_2}{9-2\sqrt{3}}\sqrt{\frac{k_0-2}{(d'-2)\ln\frac{d^2}{\delta}}}.
    \end{split}
    \right.
    \end{equation}
    Similar to DS-OSLRC,
    we define $\hat{\gamma}_s$ as follows,
    $$
        \hat{\gamma}_s=\left(\frac{8}{3}+2\sigma\right)\frac{g_{d',k_0}}{s}\ln\frac{d}{\delta}
        +\left(6.9+1.2\sigma\right)\sqrt{\frac{d'-1}{s(k_0-1)}\ln\frac{d}{\delta}}+\nu_s,
    $$
    in which
    $$
        \nu_s=
        \left\{
        \begin{array}{ll}
        \frac{2}{\sqrt{s}}\sqrt{3 g_{d',k_0}\ln\frac{d}{\delta}},&s\leq [1,s_0],\\
        \frac{1}{s}\sqrt{3g_{d',k_0}\ln\frac{d}{\delta}}
        \left(\frac{48k}{\delta^2_S}\sqrt{g_{d',k_0}\ln\frac{d^2}{\delta}}+2\right),&s=s_0+1,\\
        \frac{s_0+1}{s}\nu_{s_0+1}+\frac{1}{s}\sqrt{3g_{d',k_0}\ln\frac{d}{\delta}}\cdot
        \frac{\mu_1 a_4}{\delta^4_S}\sqrt{\sum^{s-1}_{\tau=s_0+1}\frac{k^4g^2_{d',k_0}}{\tau^2}},&s\in(s_0+1,s_1],\\
        \frac{1}{s}\sqrt{3g_{d',k_0}\ln\frac{d}{\delta}}\left(\frac{48k}{\delta^2_S}\sqrt{g_{d',k_0}\ln\frac{d^2}{\delta}}+
        2+\frac{\mu_1a_4k^2g_{d',k_0}}{\delta^4_S}\sqrt{\sum^{s_1}_{\tau=s_0+1}\frac{1}{\tau^2}}\right),&s=s_1+1,\\
        \frac{s_1+1}{s}\nu_{s_1+1}+
        \frac{1}{s}\sqrt{3g_{d',k_0}\ln\frac{d}{\delta}}
        \cdot\frac{\mu_2a_5}{\delta^2_S}\sqrt{\sum^{s-1}_{\tau=s_1+1}\frac{k^2(d'-1)}{\tau(k_0-1)}},&s>s_1+1.
        \end{array}
        \right.
    $$
    Now we can solve $\mathrm{DS}(\hat{\gamma}_s)$ and obtain the solution $\hat{\mathbf{w}}_s$.
    Different from DS-OSLRC,
    we do not use ONS to update parameters.
    There are two reasons.
    (i) We use $\hat{\mathbf{w}}_{s-1}(S_{s-1})$ to make a prediction at each round $s$,
    ensuring a tight regret bound.
    (ii) Although ONS can further improve the regret bound by a factor of $O(\ln{T})$,
    it also introduces additional terms and makes the algorithm more complicate.
    We name this algorithm DS-POSLRC (Dantzig Selector for POSLR with Compatibility condition)
    and show the pseudo-code in Algorithm \ref{alg:COLT2025:DS-POSLRC}.

    \begin{algorithm2e}[!t]
        \caption{\small{DS-POSLRC}}
        \LinesNumbered
        \footnotesize
        \label{alg:COLT2025:DS-POSLRC}
        \KwIn{$k$, $k_0$, $d$, $\delta_S$, $\delta$}
        Initialize $\hat{\mathbf{w}}_0=\frac{1}{d}\mathbf{1}_d$,
        $\hat{\mathbf{x}}_{\mathcal{I}_0}=\mathbf{0}_d$,
        $\mathbf{H}_{\mathcal{I}_0}=\mathbf{0}_{d\times d}$, $S_0$\;
        \For{$s=1,2,\ldots,T$}
        {
            Output the prediction
            $\hat{y}_s=\langle \hat{\mathbf{w}}_{s-1}(S_{s-1}),\mathbf{x}_s(S_{s-1})\rangle$\;
            Obtain $B'_s\subseteq [d]$ from SAMPLING($k_0,d',\hat{\mathbf{w}}'_{s-1}$)\;
            Let $B_s=B'_s\cup S_{s-1}$\;
            Compute $\hat{\mathbf{X}}_{\mathcal{I}_s}\mathbf{Y}_{\mathcal{I}_s}=
                \hat{\mathbf{X}}_{\mathcal{I}_{s-1}}\mathbf{Y}_{\mathcal{I}_{s-1}}+
                \hat{\mathbf{x}}_sy_s$\;
            Compute $\mathbf{H}_{\mathcal{I}_s}=\mathbf{H}_{\mathcal{I}_{s-1}}+\mathbf{h}_s$\;
            Compute $\hat{\gamma}_s$\;
            Output the solution of $\mathrm{DS}(\hat{\gamma}_s)$, denoted by $\hat{\mathbf{w}}_s$\;
            Select $S_s\subseteq [d]$\;
        }
    \end{algorithm2e}

    \begin{Mylemma}[Estimation Error]
    \label{lemma:estimator_error:DS-POSLRC}
        Let $S=\mathrm{Supp}(\mathbf{w}^\ast)$,
        $\hat{\mathbf{w}}_0=\frac{1}{d}\mathbf{1}_d$ and $3\leq k_0=O(k\ln{d})$.
        If $T>(s_1+1)^2$,
        Assumptions
        \ref{ass:ICML2025:Realizable assumption}-\ref{ass:COLT2025:sparsity} hold,
        and $\mathbf{X}_{\mathcal{I}_s}$ satisfies the $(\delta_S,S,1)$-compatibility condition
        for all $s\geq 1$,
        then with probability at least $1-T(6+\log_{1.5}\frac{d'+k}{k_0-2})\delta$,
        DS-POSLRC guarantees, $\forall s\geq 1$,
        $$
            \Vert \Delta_s(S)\Vert_1 \leq
            \left\{
            \begin{array}{ll}
            \frac{16+8\sigma}{\delta^2_S}\frac{kg_{d',k_0}}{s}\ln\frac{d}{\delta}
            +\frac{\left(26+4.8\sigma\right)k}{\delta^2_S\sqrt{s}}\sqrt{\frac{d'-1}{k_0-1}\ln\frac{d}{\delta}}+
            \frac{22k}{\delta^2_S}\sqrt{\frac{g_{d',k_0}}{s}\ln\frac{d}{\delta}},&s\leq [1,s_0]\\
            \frac{9}{9-2\sqrt{3}}\frac{a_4}{\delta^4_S}\frac{k^2g_{d',k_0}}{s},&s\in(s_0,s_1]\\
            \mu_2\frac{a_5}{\delta^2_S}\sqrt{\frac{k^2(d'-1)}{s(k_0-1)}},&s>s_1,\\
            \end{array}
            \right.
        $$
        in which $s_0$, $s_1$, $a_4$ and $a_5$ follow the definition in
        \eqref{eq:COLT2025:a1-a5_DS-POSLRC}.
    \end{Mylemma}

    \begin{proof}[of Lemma \ref{lemma:estimator_error:DS-POSLRC}]
        It can be easily confirmed that Corollary \ref{coro:data-dependent_second_moment}
        and Corollary \ref{coro:data-independent_second_moment}
        are valid with the parameter $g_{d',k_0}$.
        The proof of Lemma \ref{lemma:estimator_error:DS-POSLRC}
        is same with that of Lemma \ref{lemma:estimator_error:DS-OSLRC}.
    \end{proof}

    The algorithm in \citep{Kale2017Adaptive} attains
    $\Vert\Delta_s\Vert_1 =O\left(\frac{d}{k_0}\sqrt{\frac{d}{k_0}\frac{k^2}{s}\ln\frac{d}{\delta}}\right)$.
    For $s\geq s_1$,
    we improve the convergence rate by a factor of $O(\frac{d}{k_0})$.
    Next we give the regret bound of DS-POSLRC.

    \begin{theorem}[Regret Bound w.r.t. $\mathbf{w}^\ast$]
    \label{thm:ICML2025:regret_bound_DS-POSLRC}
        Let $\varepsilon=k$ and $\delta\in(0,1)$.
        Under the same assumptions in Lemma \ref{lemma:estimator_error:DS-POSLRC},
        with probability at least $1-(T(6+\log_{1.5}\frac{d+k}{k-2})+1)\delta$,
        the regret of DS-POSLRC satisfies
        $$
            \mathrm{Reg}(\mathbf{w}^\ast)\leq
            \frac{48^2k^2g'_{d',k_0}}{\delta^4_S}\ln\frac{d^2}{\delta}
            +\frac{2.7a^2_4k^2g_{d',k_0}}{64\delta^4_S\ln\frac{d^2}{\delta}}
            +9\mu^2_2a^2_5\frac{k^2(d'-1)}{\delta^4_S(k_0-1)}\ln\frac{T}{s_1+1}
            +O(1),
        $$
        in which $s_1$, $a_4$ and $a_5$ follow \eqref{eq:COLT2025:a1-a5_DS-POSLRC}
        and $O(1)$ hides the lower order constant terms.
    \end{theorem}

    Next we compare our regret bound with previous results.
    Let $\delta=\Theta(\frac{1}{T})$.
    Then with probability at least $1-\delta$,
    the regret of DS-POSLRC satsifies
    $$
        \mathrm{Reg}(\mathbf{w}^\ast)=O\left(\frac{(kd')^2}{\delta^4_Sk^2_0}\ln\frac{Td^2}{\delta}
        +\frac{k^2(d'-1)}{\delta^4_S(k_0-1)}\ln(T)\ln\frac{Td}{\delta}\right).
    $$

    Under RIP,
    the first algorithm by \citep{Kale2017Adaptive} enjoys a regret of
    $O(\frac{k^2d^3}{k^2_0}\ln(T)\ln\frac{Td}{\delta})$.
    Our algorithm improves the regret bound by a factor of $O(\min\{\frac{d^2}{k^2_0},\frac{d}{k_0}\ln{T}\})$.
    Besides,
    our result requires the compatibility condition which is weaker than RIP.

    Under the linear independence of features condition,
    the first algorithm by \citep{Ito2017Efficient}
    enjoys an expected regret of $O(\frac{d}{\sigma^2_dk_0}\sqrt{T})$.
    The regret bound is much worse than ours in terms of the dependence on $T$.
    What's more,
    the linear independence of features condition is stronger than RIP
    and the compatibility condition.

    \begin{MyRemark}
        It is interesting to explore
        whether recomputing $\hat{\mathbf{w}}_s$ for $s\in \{2^0,2^1,2^2,\ldots\}$
        can decrease the per-round time complexity,
        while maintaining a similar convergence rate and regret bound,
        as illustrated by the first algorithm in \citep{Kale2017Adaptive}.
        In this work,
        our goal is to demonstrate the power of our algorithm-dependent sampling scheme.
        It is left as a further work to give more efficient algorithms for $(k,k_0,d)$-OSLR.
    \end{MyRemark}

\section{Technical Lemmas}
\label{sec:COLT2025:technical_lemmas}

    \begin{Mylemma}
    \label{lemma:ICML2025:summation:s^-2}
        For any two positive integers $2<a<b$, it must be
        $$
            \sum^b_{s=a}\frac{1}{s^2}\leq \frac{1}{a-1}-\frac{1}{b},\quad
            \sum^b_{s=a}\frac{1}{bs}\leq \frac{1}{2(a-1)}.
        $$
    \end{Mylemma}
    \begin{proof}[of Lemma \ref{lemma:ICML2025:summation:s^-2}]
        As $a>2$, we have
        \begin{align*}
            \sum^b_{s=a}\frac{1}{s^2} \leq \sum^b_{s=a}\frac{1}{s(s-1)}
            =\sum^b_{s=a}\left(\frac{1}{s-1}-\frac{1}{s}\right)
            =\frac{1}{a-1}-\frac{1}{b}.
        \end{align*}
        For the second inequality,
        we consider two cases.

        \noindent\textbf{case 1} $b> 2(a-1)$.
        If $b$ is even,
        then by the first inequality,
        we have
        \begin{align*}
            \sum^b_{s=a}\frac{1}{bs}
            =\sum^{\frac{b}{2}+1}_{s=a}\frac{1}{bs}
            +\sum^{b}_{s=\frac{b}{2}+2}\frac{1}{bs}
            \leq\sum^{\frac{b}{2}+1}_{s=a}\frac{1}{2s(s-1)}
            +\frac{1}{\frac{b}{2}+1}-\frac{1}{b}
            \leq&\frac{1}{2}\left(\frac{1}{a-1}-\frac{1}{\frac{b}{2}+1}\right)
            +\frac{1}{2}\frac{1}{\frac{b}{2}+1}\\
            =&\frac{1}{2(a-1)}.
        \end{align*}
        If $b$ is odd, then we replace $b/2$ with $(b-1)/2$ in the second term.

        \noindent\textbf{case 2} $b\leq 2(a-1)$.
        By the first inequality in the lemma,
        $$
            \sum^{b}_{s=a}\frac{1}{bs}
            \leq \sum^{b}_{s=a}\frac{1}{s^2}
            \leq \frac{1}{a-1}-\frac{1}{b}\leq \frac{1}{2(a-1)},
        $$
        which concludes the proof.
    \end{proof}

    \begin{Mylemma}
    \label{lemma:COLT2025:approximation_error_w_s}
        For any $s\geq 1$,
        let $\hat{\mathbf{w}}_s$ be the solution of $\mathrm{DS}(\hat{\gamma}_s)$.
        Let $S_s\subseteq[d]$ satisfy
        $\vert S_s\vert=k$ and for any $i\in S_s$
        and $j\in [d]\setminus S_s$, $\vert \hat{w}_{s,i}\vert\geq \vert \hat{w}_{s,j}\vert$.
        Let $\mathbf{w}_s=\mathbf{w}^{\ast}(S\cap S_s)$.
        It must be
        $$
            \Vert \mathbf{w}_s - \mathbf{w}^\ast\Vert_1
            \leq \Vert \hat{\mathbf{w}}_s - \mathbf{w}^\ast\Vert_1.
        $$
    \end{Mylemma}

    \begin{proof}[of Lemma \ref{lemma:COLT2025:approximation_error_w_s}]
        Unfolding $\Vert \mathbf{w}_s - \mathbf{w}^\ast\Vert_1$ gives
        \begin{align*}
            \Vert \mathbf{w}_s - \mathbf{w}^\ast\Vert_1
            =\sum_{i\in S\setminus S_s} \vert w^{\ast}_i\vert
            =\sum_{i\in S\setminus S_s} \vert w^{\ast}_i-\hat{w}_{s,i}+\hat{w}_{s,i}\vert
            \leq& \sum_{i\in S\setminus S_s} \vert w^{\ast}_i-\hat{w}_{s,i}\vert+
            \sum_{i\in S\setminus S_s}\vert\hat{w}_{s,i}\vert\\
            \leq& \sum_{i\in S\setminus S_s} \vert w^{\ast}_i-\hat{w}_{s,i}\vert+
            \sum_{i\in S_s\setminus S}\vert\hat{w}_{s,i}\vert\\
            \leq&\Vert\hat{\mathbf{w}}_s - \mathbf{w}^\ast\Vert_1,
        \end{align*}
        in which $w^\ast_i=0$ for all $i\in S_s\setminus S$.
        We conclude the proof.
    \end{proof}

    By Lemma \ref{lemma:COLT2025:approximation_error_w_s},
    our algorithm can significantly reduce the constant factor on the regret bound.
    To be specific,
    the proof of Theorem \ref{thm:ICML2025:regret_bound_DS-OSLRC}
    will make use of the following inequality
    \begin{equation}
    \label{eq:ICML2025:bounding_Delta(S_s)_by_Delta_s:ours}
        \Vert \mathbf{w}_s - \mathbf{w}^\ast\Vert^2_1
        \leq \Vert \hat{\mathbf{w}}_s - \mathbf{w}^\ast\Vert^2_1
        \mathop{\leq}^{\mathrm{Lemma}~\ref{lemma:ICML25:Dantzig2005}} 4\Vert \Delta_s(S)\Vert^2_1.
    \end{equation}
    The second approach to prove Theorem \ref{thm:ICML2025:regret_bound_DS-OSLRC}
    is to use the following inequality
    (please refer to Lemma 4 in \cite{Kale2017Adaptive} or Lemma 3 in \cite{Ito2017Efficient}),
    \begin{equation}
    \label{eq:ICML2025:bounding_Delta(S_s)_by_Delta_s:Kale2017}
        \Vert \hat{\mathbf{w}}_s(S_s)-\mathbf{w}^\ast\Vert^2_2
        \leq 3 \Vert \hat{\mathbf{w}}_s-\mathbf{w}^\ast\Vert^2_2.
    \end{equation}
    If we define $\mathbf{w}_s=\hat{\mathbf{w}}_s(S_s)$,
    then by \eqref{eq:ICML2025:bounding_Delta(S_s)_by_Delta_s:Kale2017}
    and Lemma \ref{lemma:ICML25:Dantzig2005},
    we can obtain
    \begin{align*}
        \Vert \hat{\mathbf{w}}_s(S_s)-\mathbf{w}^\ast\Vert^2_1
        \leq k\Vert \hat{\mathbf{w}}_s(S_s)-\mathbf{w}^\ast\Vert^2_2
        \leq 3k\Vert \Delta_s\Vert^2_2
        \leq 3k\Vert \Delta_s\Vert^2_1
        &\leq 12k\Vert \Delta_s(S)\Vert^2_1.
    \end{align*}
    Thus our analysis can reduce the constant factor by a factor of $3k$.

    \begin{Mylemma}[Bernstein's inequality for martingale]
    \label{lemma:ICML25:Hoeffding_inequality}
        Let $X_1,\ldots,X_n$ be a bounded martingale difference sequence w.r.t. the filtration
        $\mathcal{H}=(\mathcal{H}_k)_{1\leq k\leq n}$ and with $\vert X_k\vert\leq a$.
        Let $Z_t=\sum^t_{k=1}X_{k}$ be the associated martingale.
        Denote the sum of the conditional variances by
        \begin{align*}
            \Sigma^2_n=\sum^n_{k=1}\mathbb{E}\left[X^2_k\vert\mathcal{H}_{k-1}\right]\leq v.
        \end{align*}
        Then for all constants $a,v>0$,
        with probability at least $1-\delta$,
        \begin{align*}
            \max_{t=1,\ldots,n}Z_t < \frac{2}{3}a\ln\frac{1}{\delta}+\sqrt{2v\ln\frac{1}{\delta}}.
        \end{align*}
    \end{Mylemma}

    Lemma \ref{lemma:ICML25:Hoeffding_inequality}
    is derived from Lemma A.8 in \citep{Cesa-Bianchi2006Prediction}.
    Note that $v$ must be a constant.
    Next we give a new Bernstein's inequality for martingale
    in which $v$ is a random variable depending on $X_1,\ldots,X_n$.

    \begin{Mylemma}
    \label{lemma:ICML25:New_Bernstein_inequality}
        Let $X_1,\ldots,X_n$ be a bounded martingale difference sequence w.r.t. the filtration
        $\mathcal{H}=(\mathcal{H}_k)_{1\leq k\leq n}$ and with $\vert X_k\vert\leq a$.
        Let $Z_t=\sum^t_{k=1}X_{k}$ be the associated martingale.
        Denote the sum of the conditional variances by
        $
            \Sigma^2_n=\sum^n_{k=1}\mathbb{E}\left[X^2_k\vert\mathcal{H}_{k-1}\right]\leq v,
        $
        where $v\in[a_1n,a_2n]$ is a random variable depending on $X_1,\ldots,X_n$ and $0< a_1<a_2$ are constants.
        Then for any constant $a>0$,
        with probability at least $1-\left(1+\log_{\beta}\frac{a_2}{a_1}\right)\delta$,
        \begin{align*}
            \max_{t=1,\ldots,n}Z_t < \frac{2}{3}a\ln\frac{1}{\delta}+\sqrt{2\beta v\ln\frac{1}{\delta}},
        \end{align*}
        in which $\beta>1$ is a constant.
    \end{Mylemma}

    Lemma \ref{lemma:ICML25:New_Bernstein_inequality} is a slight variant
    of Lemma 1 in \citep{Li2024On}.

    \begin{proof}[of Lemma \ref{lemma:ICML25:New_Bernstein_inequality}]
        We divide the interval $[a_1n,a_2n]$ as follows
        \begin{align*}
            [a_1n,a_2n]\subseteq
            \bigcup^{\lfloor\log_{\beta}\frac{a_2}{a_1}\rfloor}_{j=0}
            \left[a_1n\cdot \beta^j,a_1n\cdot \beta^{j+1}\right).
        \end{align*}
        We decompose the random event as follows,
        \begin{align*}
            &\mathbb{P}\left[\max_{t=1,\ldots,n}Z_t>\frac{2}{3}a\ln\frac{1}{\delta}+\sqrt{2\beta v\ln\frac{1}{\delta}},\Sigma^2_n\leq v\right]\\
            =&\mathbb{P}\left[\max_{t\leq n}Z_t>\frac{2}{3}a\ln\frac{1}{\delta}+\sqrt{2\beta v\ln\frac{1}{\delta}},
            \Sigma^2_n\leq v,\cup^{\lfloor\log_{\beta}\frac{a_2}{a_1}\rfloor}_{j=0}
            a_1n\cdot \beta^j\leq v < a_1n\cdot \beta^{j+1}\right]\\
            \leq&\sum^{\lfloor\log_{\beta}\frac{a_2}{a_1}\rfloor}_{j=0}
            \mathbb{P}\left[\max_{t\leq n}Z_t>\frac{2}{3}a\ln\frac{1}{\delta}+\sqrt{2\beta v\ln\frac{1}{\delta}},\Sigma^2_n\leq v,
            a_1n\cdot \beta^j\leq v < a_1n\cdot \beta^{j+1}\right]\\
            \leq&\sum^{\lfloor\log_{\beta}\frac{a_2}{a_1}\rfloor}_{j=0}
            \mathbb{P}\left[\max_{t\leq n}Z_t>\frac{2}{3}a\ln\frac{1}{\delta}
            +\sqrt{2\beta\cdot a_1n\cdot \beta^j\ln\frac{1}{\delta}},\Sigma^2_n\leq v,
            a_1n\cdot \beta^j\leq v < a_1n\cdot \beta^{j+1}\right]\\
            =&\sum^{\lfloor\log_{\beta}\frac{a_2}{a_1}\rfloor}_{j=0}
            \mathbb{P}\left[\max_{t\leq n}Z_t>\frac{2}{3}a\ln\frac{1}{\delta}
            +\sqrt{2\cdot a_1n\cdot \beta^{j+1}\ln\frac{1}{\delta}},\Sigma^2_n\leq v,
            a_1n\cdot \beta^j\leq v < a_1n\cdot \beta^{j+1}\right]\\
            \leq&\left(1+\log_{\beta}\frac{a_2}{a_1}\right)\delta,
        \end{align*}
        in which we use
        Lemma \ref{lemma:ICML25:Hoeffding_inequality} for each sub-event.
    \end{proof}

    \begin{Mylemma}[\cite{Hazan2007Logarithmic}]
    \label{lemma:ICML2025:technical_lemma:ONS}
        Let $\mathbf{u}_t\in \mathbb{R}^n$, for $t=1,2,\ldots,T$, be a sequence of vectors
        such that for some $r>0$, $\Vert \mathbf{u}_t\Vert_2\leq r$.
        Define $\mathbf{V}_t=\sum^t_{s=1}\mathbf{u}_s\mathbf{u}^\top_s+\varepsilon \cdot\mathbf{I}_{n\times n}$.
        Then
        $$
            \sum^T_{t=1}\mathbf{u}^\top_t\mathbf{V}^{-1}_t\mathbf{u}_t \leq n\ln\left(\frac{r^2T}{\varepsilon}+1\right).
        $$
    \end{Mylemma}

    \begin{Mylemma}
    \label{lemma:ICML25:sampling_probability}
        For each $s=1,2,\ldots$,
        let $\mathbf{q}_s=(\frac{\vert\hat{w}_{s-1,1}\vert}{\Vert \hat{\mathbf{w}}_{s-1}\Vert_1},
        \ldots,\frac{\vert\hat{w}_{s-1,d}\vert}{\Vert \hat{\mathbf{w}}_{s-1}\Vert_1})$
        and $B_s$ be the indexes of the features selected by DS-OSLRC.
        \begin{equation}
        \label{eq:ICML25:sampling_probability}
        \begin{split}
            \forall i\in[d],\quad\mathbb{P}[i\in B_s]=&\frac{d-k}{d-1}q_{s,i}+\frac{k-1}{d-1},\\
            \forall i\neq j\in[d],\quad\mathbb{P}[i,j\in B_s]=&\frac{1}{g_{d,k}}+
            \frac{1}{g_{d,k}}\cdot(q_{s,i}+q_{s,j}),\\
            \forall i\neq j\neq r\in[d],\quad
            \mathbb{P}[i,j,r\in B_s]=&\frac{1}{g_{d,k}}\cdot\frac{k-3}{d-3}+
            \frac{1}{g_{d,k}}\cdot\frac{d-k}{d-3}\cdot (q_{t,i}+q_{t,j}+q_{t,r}).
            \end{split}
        \end{equation}
        Besides,
        $$
            \mathbb{E}_s\left[\hat{\mathbf{x}}_{s^2}\right]={\mathbf{x}}_{s^2},\quad
            \mathbb{E}_s\left[\mathbf{h}_{s^2}\right]={\mathbf{x}}_{s^2}{\mathbf{x}}^\top_{s^2},
        $$
        where $\mathbb{E}_s[\cdot]=\mathbb{E}[\cdot\vert B_1,\ldots,B_{s-1}]$ is the conditional expectation
        and is taken with respect to $B_s$.
    \end{Mylemma}

\begin{proof}[of Lemma \ref{lemma:ICML25:sampling_probability}]
    \begin{table}[!t]
      \centering
      \begin{tabular}{|r|r|r|r|r|r|}
      \hline
      $(m,n)$  & $1$         &     $2$       & $3$       & $\ldots$ & $k$ \\
      \hline
      $1$      & -           &     $p_{1,2}[i,j]$ & $p_{1,3}[i,j]$ & $\ldots$ & $p_{1,k}[i,j]$\\
      $2$      & $p_{2,1}[i,j]$   &     -         & $p_{2,3}[i,j]$ & $\ldots$ & $p_{2,k}[i,j]$\\
      $3$      & $p_{3,1}[i,j]$   &     $p_{3,2}[i,j]$ & -         & $\ldots$ & $p_{3,k}[i,j]$\\
      $\ldots$ & $\ldots$    &     $\ldots$  & $\ldots$  & $\ldots$ & $\ldots$\\
      $k$      & $p_{k,1}[i,j]$   &     $p_{k,2}[i,j]$ & $p_{k,3}[i,j]$ & $\ldots$ & -\\
      \hline
      \end{tabular}
      \caption{The probabilities of the event $(i,j)\in B_s$.
        The notation ``-'' means the event is invalid.}
      \label{tab:ICML2025:probability_matrix}
    \end{table}
    It is easy to prove the first equality in \eqref{eq:ICML25:sampling_probability}.
    For any $i\in B_s$,
    summing all the probabilities selected at each sampling step $n=1,2,3,\ldots,k$,
    yields the following result.
    \begin{align*}
        \mathbb{P}[i\in B_s]
        =&
        p_{s,i}+(1-p_{s,i})\cdot\frac{1}{d-1}
        +\sum^k_{n=3}\underbrace{(1-p_{s,i})\cdot \prod^{n-2}_{r=1}\left(1-\frac{1}{d-r}\right)\frac{1}{d-(n-1)}}_
        {\mathbb{P}[i~\mathrm{is~selected~at~the}~n\text{-}\mathrm{th~round}]}\\
        =&p_{s,i}+(1-p_{s,i})\cdot\frac{k-1}{d-1}.
    \end{align*}
    Rearranging terms recoveries the first equality.
    It is more complicated to prove the second equality
    and the third equality in \eqref{eq:ICML25:sampling_probability}.
    We consider any pair of $(i,j)$,
    and analyze the probability $\mathbb{P}[i\neq j\in B_s]$.
    For clarity,
    we enumerate all of the combinations $(m,n)$
    where the $i$-th feature is selected during the $m$-th sampling
    and the $j$-th feature is selected during the $n$-th sampling
    in Table \ref{tab:ICML2025:probability_matrix},
    in which $p_{m,n}[i,j]$ is the corresponding probability.
    It is obvious that
    $\mathbb{P}[i\neq j\in B_s]=\sum_{m\neq n}p_{m,n}[i,j]$.
    Next we analyze $p_{n,m}[i,j]$.

    It is worth mentioning that the probability that the $i$-th feature is selected at the first sampling
    does not equal to the probability that the $j$-th feature is selected at the first sampling.
    Thus we must separately analyze the cases $m=1$ and $n=1$.
    For $m\geq2$ and $n\geq2$,
    the probability that the $i$-th feature selected at the $r$-th sampling
    is same with that of the $j$-th feature, in which $r\geq 2$.
    Thus we just analyze $p_{m,n}$ for all $m=2,3,\ldots$,
    which equals to $p_{m,n}$ for all $n=2,3,\ldots$
    by the symmetry.

    We first consider the case that the $i$-th feature is selected at the first sampling, i.e., $m=1$.
    \begin{align*}
        p_{1,2}[i,j]=&q_{s,i}\frac{1}{d-1},\\
        \forall n\geq 3, \quad p_{1,n}[i,j]=&q_{s,i}\cdot \prod^{n-2}_{r=1} \left(1-\frac{1}{d-r}\right)\frac{1}{d-(n-1)}
        =q_{s,i}\frac{1}{d-1}.
    \end{align*}
    If the $j$-th feature is sampled at the first sampling, i.e., $n=1$,
    then we have
    \begin{align*}
            p_{2,1}[i,j]=&q_{s,j}\frac{1}{d-1},\\
        \forall m\geq 3, \quad p_{m,1}[i,j]=&q_{s,j}\cdot \prod^{m-2}_{r=1} \left(1-\frac{1}{d-r}\right)\frac{1}{d-(m-1)}
        =q_{s,j}\frac{1}{d-1}.
    \end{align*}
    Then we consider the case $m=2$.
    \begin{align*}
        p_{2,3}[i,j]
        =&(1-q_{s,i}-q_{s,j})\cdot \frac{1}{d-1}\cdot \frac{1}{d-2},\\
        \forall n\geq 4,\quad p_{2,n}[i,j]
        =&(1-q_{s,i}-q_{s,j})\cdot \frac{1}{d-1}\cdot
        \prod^{n-2}_{r=2}\left(1-\frac{1}{d-r}\right)\cdot \frac{1}{d-(n-1)}\\
        =&\frac{1-q_{s,i}-q_{s,j}}{(d-1)(d-2)}.
    \end{align*}
    By the symmetry, the probability of sampling $i$ and $j$ for $m\geq 2$ and $n= 2$ is
    $$
        \forall m\geq 3\quad p_{m,2}[i,j]
        =(1-q_{s,i}-q_{s,j})\cdot \frac{1}{d-1}\cdot \frac{1}{d-2}.
    $$
    Finally, we consider the case $m\geq 3$.
    \begin{align*}
        p_{m,m+1}[i,j]
        =&(1-q_{s,i}-q_{s,j})\cdot \prod^{m-2}_{r=1}
         \left(1-\frac{2}{d-r}\right)\frac{1}{d-(m-1)}\cdot \frac{1}{d-m}\\
        =&\frac{1-q_{s,i}-q_{s,j}}{(d-1)(d-2)},\\
         \forall n\geq m+2,\quad p_{m,n}[i,j]
        =&(1-q_{s,i}-q_{s,j})\prod^{m-2}_{r=1}\left(1-\frac{2}{d-r}\right)\frac{1}{d-(m-1)}
         \prod^{n-m-2}_{r=0}\frac{1-\frac{1}{d-m-r}}{d-(n-1)}\\
        =&\frac{1-q_{s,i}-q_{s,j}}{(d-1)(d-2)}.
    \end{align*}
    By the symmetry,
    for the case of $n\geq 3$,
    we have
    $$
        \forall m\geq n+1\quad p_{m,n}[i,j]
        =(1-q_{s,i}-q_{s,j})\cdot \frac{1}{d-1}\cdot \frac{1}{d-2}.
    $$
    Combining all of the above results yields
    \begin{align*}
        \mathbb{P}[i\neq j\in B_s]
        =&\sum_{n\neq 1}p_{1,n}[i,j]+\sum^{k}_{m=2}\sum_{n\neq m}p_{m,n}[i,j]\\
        =&q_{s,i}\cdot\frac{k-1}{d-1}+\sum^{k}_{m=2}\left(\frac{q_{s,j}}{d-1}+(1-q_{s,i}-q_{s,j})\frac{k-2}{(d-1)(d-2)}\right)\\
        =&q_{s,i}\cdot\frac{k-1}{d-1}+(k-1)\cdot\left(\frac{q_{s,j}}{d-1}+\left(1-q_{s,i}-q_{s,j}\right)
        \frac{k-2}{(d-1)(d-2)}\right)\\
        =&\frac{(k-1)(k-2)}{(d-1)(d-2)}+
        \frac{(k-1)(d-k)}{(d-1)(d-2)}\cdot(q_{s,i}+q_{s,j}),
    \end{align*}
    which recoveries the second equality in \eqref{eq:ICML25:sampling_probability}.

    Finally,
    we will prove the third equality.
    We consider any pair of $(i,j,r)$,
    and analyze $\mathbb{P}[i\neq j\neq r\in B_t]$.
    We can enumerate all combinations $(m,n,o)$
    where the $i$-th feature is selected during the $m$-th sampling,
    the $j$-th feature is selected during the $n$-th sampling,
    and the $r$-th feature is selected during the $o$-th sampling.
    We consider four cases.
    \begin{itemize}
      \item $m=1$\\
            Assuming that the $i$-th feature was selected at the first sampling step.
            We only need to compute $\mathbb{P}[j\neq r\in B_s]$ following the second equality
            in \eqref{eq:ICML25:sampling_probability}.
            To be specific,
            we will sample $k-1$ indexes from $[d]\setminus \{m\}$ without replacement.
            Thus we have
            $$
                \mathbb{P}[i\neq j\neq r\in B_s]
                =q_{s,i}\cdot \left(\frac{(k-2)(k-3)}{(d-2)(d-3)}
                +\frac{(k-2)(d-k)}{(d-2)(d-3)}\left(\frac{1}{d-1}+\frac{1}{d-1}\right)\right),
            $$
            in which we use $(d-1,k-1,\frac{1}{d-1},\frac{1}{d-1})$
            to replace the value of $(d,k,q_{t,i},q_{t,j})$.
      \item $n=1$\\
            The analysis is same with that of $m=1$.
            $$
                \mathbb{P}[i\neq j\neq r\in B_s]
                =q_{s,j}\cdot \left(\frac{(k-2)(k-3)}{(d-2)(d-3)}
                +\frac{(k-2)(d-k)}{(d-2)(d-3)}\cdot\frac{2}{d-1}\right).
            $$
      \item $o=1$\\
            The analysis is also same with that of $m=1$.
            $$
                \mathbb{P}[i\neq j\neq r\in B_s]
                =q_{s,r}\cdot \left(\frac{(k-2)(k-3)}{(d-2)(d-3)}
                +\frac{(k-2)(d-k)}{(d-2)(d-3)}\cdot\frac{2}{d-1}\right).
            $$
      \item $m\neq 1, n\neq 1, o\neq 1$:\\
            Assuming that the $u$-th feature satisfying $u\neq i, u\neq j, u\neq r$ has been selected.
            Then we will sample $k-1$ indexes from $[d]\setminus \{m\}$ without replacement.
            By the analyzing of the second equality in \eqref{eq:ICML25:sampling_probability},
            it is easy to be verified that
            the probabilities that any combination of $(m,n,o)$ for $m\geq 2, n\geq 2, o\geq 2$
            are the same.
            In this case,
            the number of combinations of $(m,n,o)$ is $(k-1)(k-2)(k-3)$.
            Without loss of generality,
            assuming that $2=m<n<o$.
            Thus
            \begin{align*}
                &\mathbb{P}[i\neq j\neq r\in B_s]\\
                =&\left(1-q_{s,i}-q_{s,j}-q_{s,r}\right) (k-1)(k-2)(k-3)
                \left(\frac{(k-2)(d-k)}{(d-2)(d-3)}
                +\frac{(k-2)(d-k)}{(d-2)(d-3)}\frac{2}{d-1}\right)\\
                =&\left(1-q_{s,i}-q_{s,j}-q_{s,r}\right)\cdot (k-1)(k-2)(k-3)\cdot
                \left(\frac{1}{d-1}\cdot \frac{1}{d-2}\cdot \frac{1}{d-3}\right).
            \end{align*}
    \end{itemize}
    Summing all results gives
    \begin{align*}
        &\mathbb{P}[i\neq j\neq r\in B_s]\\
        =&\left(\frac{(k-2)(k-3)}{(d-2)(d-3)}
            +\frac{(k-2)(d-k)}{(d-2)(d-3)}\cdot\frac{2}{d-1}\right)(q_{s,i}+q_{s,j}+q_{s,r})+\\
            &\left(1-q_{s,i}-q_{s,j}-q_{s,r}\right)
            \cdot \frac{(k-1)(k-2)(k-3)}{(d-1)(d-2)(d-3)}\\
        =&\frac{(k-1)(k-2)}{(d-1)(d-2)}\cdot (q_{s,i}+q_{s,j}+q_{s,r})
        +\left(1-q_{s,i}-q_{s,j}-q_{s,r}\right)
            \cdot \frac{(k-1)(k-2)(k-3)}{(d-1)(d-2)(d-3)}\\
        =&\frac{(k-1)(k-2)(k-3)}{(d-1)(d-2)(d-3)}+
        \frac{(k-1)(k-2)(d-k)}{(d-1)(d-2)(d-3)}\cdot (q_{s,i}+q_{s,j}+q_{s,r}),
    \end{align*}
    which concludes the proof.

    As $B_s$ depends on $B_1,B_2,\ldots,B_{s-1}$,
    thus
    $\hat{\mathbf{w}}_0$, $\hat{\mathbf{w}}_1,\ldots,\hat{\mathbf{w}}_s$ are not independent.
    Let $\mathbb{E}_s[\cdot]=\mathbb{E}[\cdot\vert B_1,\ldots,B_{s-1}]$ be the conditional expectation.
    It is easy to show that
    $\mathbb{E}_s[\hat{x}_{s^2,i}]=x_{s^2,i}$ for all $i\in[d]$,
    implying $\mathbb{E}_s[\hat{\mathbf{x}}_{s^2}]={\mathbf{x}}_{s^2}$.
    Similarly,
    it must be $\mathbb{E}_s[h_{s^2}[i,i]]=x^2_{s^2,i}$
    and $\mathbb{E}_s[h_{s^2}[i,j]]=x_{s^2,i}x_{s^2,j}$ for all $i\neq j$.
    Thus $\mathbb{E}_s[\mathbf{h}_{s^2}]=\mathbf{x}_{s^2}{\mathbf{x}}^\top_{s^2}$.
\end{proof}

\begin{Mylemma}
\label{lemma:ICML25:aux_inequality}
    For any $a>0,b>0,c>0,d>0$ and $0\leq x\leq y\leq 1$,
    if $a+b=c+d=1$ and $a\leq c$, then it must be
    $$
        \frac{a+bx}{c+dy}\leq \frac{a+b}{c+d}.
    $$
\end{Mylemma}
\begin{proof}[of Lemma \ref{lemma:ICML25:aux_inequality}]
    The above inequality is equivalent to
    \begin{align*}
        &(a+bx)(c+d) \leq (a+b)(c+dy)\\
        \Leftrightarrow& ad+bcx+bdx \leq ady+bc+bdy\\
        \Leftrightarrow& a(1-c)+(1-a)cx+(1-a)(1-c)x \leq a(1-c)y+(1-a)c+(1-a)(1-c)y\\
        \Leftrightarrow& a(1-c)(1-y)\leq c(1-a)(1-x)+(1-a)(1-c)(y-x).
    \end{align*}
    By applying the constraints on these variables,
    the following two inequalities can be derived.
    $$
        1-y\leq 1-x, \quad 1-c \leq 1-a.
    $$
    It is obvious that the inequality in Lemma \ref{lemma:ICML25:aux_inequality} holds.
\end{proof}

\begin{Mylemma}
\label{lemma:ICML25:ratio_probabilities}
    Let $k\geq 3$.
    At any $\tau\geq 1$,
    for any $i\neq j\neq r\in[d]$,
    it must be
    $$
        \frac{\mathbb{P}[i,j,r\in B_{\tau}]}
        {\mathbb{P}[i,j\in B_{\tau}]\cdot\mathbb{P}[i,r\in B_{\tau}]}\leq \frac{d-1}{k-1}.
    $$
\end{Mylemma}

\begin{proof}[of Lemma \ref{lemma:ICML25:ratio_probabilities}]
    By Lemma \ref{lemma:ICML25:sampling_probability} and Lemma \ref{lemma:ICML25:aux_inequality},
    we can obtain
    \begin{align*}
        &\frac{\mathbb{P}[i,j,r\in B_{\tau}]}{\mathbb{P}[i,j\in B_{\tau}]\cdot \mathbb{P}[i,r\in B_{\tau}]}\\
        =&
        \frac{\frac{(k-1)(k-2)(k-3)}{(d-1)(d-2)(d-3)}+
        \frac{(k-1)(k-2)(d-k)}{(d-1)(d-2)(d-3)}\cdot (q_{\tau-1,i}+q_{\tau-1,j}+q_{\tau-1,r})}
        {\left(\frac{(k-1)(k-2)}{(d-1)(d-2)}+
        \frac{(k-1)(d-k)}{(d-1)(d-2)}\cdot(q_{\tau-1,i}+q_{\tau-1,j})\right)
        \cdot \left(\frac{(k-1)(k-2)}{(d-1)(d-2)}+
        \frac{(k-1)(d-k)}{(d-1)(d-2)}\cdot(q_{\tau-1,i}+q_{\tau-1,r})\right)}\\
        =&\frac{d-1}{k-1}\cdot\frac{\frac{k-3}{d-3}+
        \frac{d-k}{d-3}\cdot (q_{\tau-1,i}+q_{\tau-1,j}+q_{\tau-1,r})}
        {\frac{k-2}{d-2}+ \frac{d-k}{d-2}(2q_{\tau-1,i}+q_{\tau-1,j}+q_{\tau-1,r})
        +\frac{(d-k)^2}{(d-2)(k-2)}(q_{\tau-1,i}+q_{\tau-1,j})(q_{\tau-1,i}+q_{\tau-1,r})}\\
        \leq&\frac{d-1}{k-1}\cdot\frac{\frac{k-3}{d-3}+
        \frac{d-k}{d-3}}{\frac{k-2}{d-2}+ \frac{d-k}{d-2}}=\frac{d-1}{k-1},
    \end{align*}
    which concludes the proof.
\end{proof}

\begin{Mylemma}
\label{lemma:ICML25:general_second_order_moment}
    For any $\tau\geq 1$ and $\mathbf{v}\in\mathbb{R}^d$,
    let $\mathbf{z}_{\tau}=\mathbf{h}_{\tau^2}\mathbf{v}
        -{\mathbf{x}}_{\tau^2}{\mathbf{x}}^\top_{\tau^2}\mathbf{v}$.
    Then
    \begin{align*}
        \forall i\in[d],\quad\mathbb{E}_{\tau}[z_{\tau,i}]=&0,\\
        \mathbb{E}_{\tau}\left[ z^2_{\tau,i}\right]=&
        \left(\frac{1}{\mathbb{P}[i\in B_{\tau}]}-1\right)x^4_{\tau^2,i}\cdot v^2_i
        +2\sum_{j\ne i}\left(
        \frac{1}{\mathbb{P}[i\in B_{\tau}]}-1\right)x^3_{\tau^2,i}x_{\tau^2,j}\cdot v_iv_j+\\
        &\sum_{j\neq i}\sum_{r\neq i}
        \left(\frac{\mathbb{P}[i,j,r\in B_{\tau}]}{\mathbb{P}[i,j\in B_{\tau}]\cdot \mathbb{P}[i,r\in B_{\tau}]}-1\right)
        x^2_{\tau^2,i}x_{\tau^2,j}x_{\tau^2,r}\cdot v_rv_j.
    \end{align*}
\end{Mylemma}

\begin{proof}[of Lemma \ref{lemma:ICML25:general_second_order_moment}]
    Substituting into the definition of $\mathbf{h}_{\tau^2}$,
    and using Lemma \ref{lemma:ICML25:sampling_probability},
    it is easy to verify that $\mathbb{E}_{\tau}[z_{\tau,i}]=0$ for all $i\in[d]$.
    It is more challenge to analyze the variance,
    as the even $i\in B_{\tau}$, $i\neq j \in B_{\tau}$ and
    $i \neq j\neq r\in B_{\tau}$ are not independent.
    Note that we have
    $\mathbb{I}_{i\in B_{\tau}}\cdot\mathbb{I}_{i,j\in B_{\tau}}
    =\mathbb{I}_{i,j\in B_{\tau}}$
    and $\mathbb{I}_{i,j\in B_{\tau}}\cdot \mathbb{I}_{i,r\in B_{\tau}}=
        \mathbb{I}_{i,j,r\in B_{\tau}}$.
    \begin{align*}
        &\mathbb{E}_{\tau}[z^2_{\tau,i}]\\
        =&\mathbb{E}_{\tau}\left[\left(\frac{x^2_{\tau^2,i}}{\mathbb{P}[i\in B_{\tau}]}
        \mathbb{I}_{i\in B_{\tau}}-x^2_{\tau^2,i}\right)^2 v^2_i+
        \left(\sum_{j\ne i}\left(
        \frac{x_{\tau^2,i}x_{\tau^2,j}}{\mathbb{P}[i,j\in B_{\tau}]}
        \mathbb{I}_{i,j\in B_{\tau}}-x_{\tau^2,i}x_{\tau^2,j}\right)v_j\right)^2\right]+\\
        &2\mathbb{E}_{\tau}\left[\left(\frac{x^2_{\tau^2,i}}{\mathbb{P}[i\in B_{\tau}]}
        \mathbb{I}_{i\in B_{\tau}}-x^2_{\tau^2,i}\right)\cdot v_i
        \left(\sum_{j\ne i}\left(
        \frac{x_{\tau^2,i}x_{\tau^2,j}}{\mathbb{P}[i,j\in B_{\tau}]}
        \mathbb{I}_{i,j\in B_{\tau}}-x_{\tau^2,i}x_{\tau^2,j}\right)v_j\right)\right]\\
        =&\left(\frac{1}{\mathbb{P}[i\in B_{\tau}]}-1\right)x^4_{\tau^2,i}\cdot v^2_i+
        2\sum_{j\ne i}\left(
        \frac{1}{\mathbb{P}[i\in B_{\tau}]}-1\right)x^3_{\tau^2,i}x_{\tau^2,j}\cdot v_iv_j+\\
        &\mathbb{E}_{\tau}\left[\sum_{j\neq i}\left(
        \frac{x_{\tau^2,i}x_{\tau^2,j}}{\mathbb{P}[i,j\in B_{\tau}]}
        \mathbb{I}_{i,j\in B_{\tau}}-x_{\tau^2,i}x_{\tau^2,j}\right)^2v^2_j\right]+\\
        &\mathbb{E}_{\tau}\left[\sum_{j\neq i}\sum_{r\neq i, r\neq j}\left(
        \frac{x_{\tau^2,i}x_{\tau^2,j}}{\mathbb{P}[i,j\in B_{\tau}]}
        \mathbb{I}_{i,j\in B_{\tau}}-x_{\tau^2,i}x_{\tau^2,j}\right)
        \left(
        \frac{x_{\tau^2,i}x_{\tau^2,r}}{\mathbb{P}[i,r\in B_{\tau}]}
        \mathbb{I}_{i,r\in B_{\tau}}-x_{\tau^2,i}x_{\tau^2,r}\right)v_rv_j\right]\\
        =&\left(\frac{1}{\mathbb{P}[i\in B_{\tau}]}-1\right)x^4_{\tau^2,i}\cdot v^2_i+
        2\sum_{j\ne i}\left(
        \frac{1}{\mathbb{P}[i\in B_{\tau}]}-1\right)x^3_{\tau^2,i}x_{\tau^2,j}\cdot v_iv_j+\\
        &\sum_{j\neq i}\sum_{r\neq i}
        \left(\frac{\mathbb{P}[i,j,r\in B_{\tau}]}{\mathbb{P}[i,j\in B_{\tau}]\cdot \mathbb{P}[i,r\in B_{\tau}]}-1\right)
        x^2_{\tau^2,i}x_{\tau^2,j}x_{\tau^2,r}\cdot v_rv_j.
    \end{align*}
    We conclude the proof.
\end{proof}

    By Lemma \ref{lemma:ICML25:general_second_order_moment},
    we give two types of upper bounds on the variance.

\begin{MyCoro}
\label{coro:data-dependent_second_moment}
    Assuming that $\Vert\mathbf{x}_t\Vert_{\infty}\leq 1,t\in[T]$.
    For any $ \tau\geq 1$,
    let $\mathbf{v}=\hat{\mathbf{w}}_{\tau-1}$ in Lemma \ref{lemma:ICML25:general_second_order_moment}.
    Then
    \begin{align*}
        \forall i\in[d],\quad
        \vert z_{\tau,i}\vert\leq&\frac{(d-1)(d-2)}{(k-1)(k-2)}\Vert\hat{\mathbf{w}}_{\tau-1}\Vert_1,\\
        \mathbb{E}_{\tau}\left[ z^2_{\tau,i}\right]\leq&
        \frac{2(d-1)}{k-1}\Vert\hat{\mathbf{w}}_{\tau-1}\Vert^2_1.
    \end{align*}
\end{MyCoro}

\begin{proof}[of Corollary \ref{coro:data-dependent_second_moment}]
    By Lemma \ref{lemma:ICML25:sampling_probability},
    we can obtain
    \begin{align*}
        \vert z_{\tau,i}\vert
        =&\left\vert\sum^d_{j=1}\left(
        \frac{x_{\tau^2,i}x_{\tau^2,j}}{\mathbb{P}[i,j\in B_{\tau}]}
        \mathbb{I}_{i,j\in B_{\tau}}-x_{\tau^2,i}x_{\tau^2,j}\right)\hat{ w}_{\tau-1,j}\right\vert\\
        \leq& \left\vert \left(\frac{d-1}{k-1}-1\right)\vert \hat{w}_{\tau-1,i}\vert
        +\sum_{j\neq i}\left(\frac{(d-1)(d-2)}{(k-1)(k-2)}-1\right)\cdot \vert \hat{w}_{\tau-1,j}\vert \right\vert\\
        \leq&\frac{(d-1)(d-2)}{(k-1)(k-2)}\Vert\hat{\mathbf{w}}_{\tau-1}\Vert_1.
    \end{align*}
    Next we analyze the second-order moment.
    It is easy to show that the second-order moment is $O(\frac{d-1}{k-1})$.
    The technical challenge lies in maintaining a small constant.
    To this end,
    we will use Lemma \ref{lemma:ICML25:ratio_probabilities}.
    By Lemma \ref{lemma:ICML25:general_second_order_moment},
    we just need to analyze the following three terms.
    \begin{align*}
        \left(\frac{1}{\mathbb{P}[i\in B_{\tau}]}-1\right)x^4_{\tau^2,i}\cdot v^2_i
        \leq& \frac{1}{\frac{k-1}{(d-1)q_{\tau,i}}+\frac{d-k}{d-1}}\vert\hat{w}_{\tau-1,i}\vert\cdot
        \Vert \hat{\mathbf{w}}_{\tau-1}\Vert_1-\vert\hat{w}_{\tau-1,i}\vert^2,\\
        \leq&\vert\hat{w}_{\tau-1,i}\vert\cdot
        \Vert \hat{\mathbf{w}}_{\tau-1}\Vert_1-\vert\hat{w}_{\tau-1,i}\vert^2.\\
        \left(\frac{1}{\mathbb{P}[i\in B_{\tau}]}-1\right)x^3_{\tau^2,i}x_{\tau^2,j}\cdot v_iv_j
        \leq& \vert\hat{w}_{\tau-1,j}\vert\cdot\Vert \hat{\mathbf{w}}_{\tau-1}\Vert_1
        -\vert\hat{w}_{\tau-1,i}\vert\cdot\vert\hat{w}_{\tau-1,j}\vert.
    \end{align*}
    If $r=j$, then we have
    \begin{align*}
        \left(\frac{\mathbb{P}[i,j,r\in B_{\tau}]}{\mathbb{P}[i,j\in B_{\tau}]\cdot \mathbb{P}[i,r\in B_{\tau}]}-1\right)
        x^2_{\tau^2,i}x_{\tau^2,j}&x_{\tau^2,r}v_rv_j
        =\left(\frac{1}{\mathbb{P}[i,j\in B_{\tau}]}-1\right)
        x^2_{\tau^2,i}x_{\tau^2,j}x_{\tau^2,r} v^2_j\\
        \leq&\frac{1}{\frac{(k-1)(k-2)}{(d-1)(d-2)q_{\tau,j}}+\frac{(k-1)(d-k)}{(d-1)(d-2)}}
        \cdot \vert\hat{w}_{\tau-1,j}\vert\cdot\Vert\hat{\mathbf{w}}_{\tau-1}\Vert_1\\
        \leq&\frac{d-1}{k-1}\cdot
        \vert\hat{w}_{\tau-1,j}\vert\cdot\Vert\hat{\mathbf{w}}_{\tau-1}\Vert_1.
    \end{align*}
    If $r\neq j$,
    then by Lemma \ref{lemma:ICML25:ratio_probabilities},
    we have
    \begin{align*}
        &\left(\frac{\mathbb{P}[i,j,r\in B_{\tau}]}{\mathbb{P}[i,j\in B_{\tau}]\cdot \mathbb{P}[i,r\in B_{\tau}]}-1\right)
        x^2_{\tau^2,i}x_{\tau^2,j}x_{\tau^2,r}\cdot v_rv_j
        \leq\left(\frac{d-1}{k-1}-1\right) \cdot\vert \hat{w}_{\tau-1,r}\vert\cdot \vert\hat{w}_{\tau-1,j}\vert.
    \end{align*}
    Summing the above results yields
    \begin{align*}
        \mathbb{E}_{\tau}\left[ z^2_{\tau,i}\right]\leq&
        \vert\hat{w}_{\tau-1,i}\vert\cdot
        \Vert \hat{\mathbf{w}}_{\tau-1}\Vert_1
        +2\sum_{j\ne i}\vert\hat{w}_{\tau-1,j}\vert\cdot\Vert\hat{\mathbf{w}}_{\tau-1}\Vert_1
        -\Vert\hat{\mathbf{w}}_{\tau-1}\Vert^2_1+\\
        &\frac{d-1}{k-1}
        \sum_{j\neq i}\vert \hat{w}_{\tau-1,j}\vert\cdot\Vert\hat{\mathbf{w}}_{\tau-1}\Vert_1
        +\left(\frac{d-1}{k-1}-1\right)
        \sum_{j\neq i}\sum_{r\neq i,r\neq j}\vert\hat{w}_{\tau-1,r}\vert\cdot\vert\hat{w}_{\tau-1,j}\vert\\
        =&\frac{2(d-1)}{k-1}\Vert\hat{\mathbf{w}}_{\tau-1}\Vert^2_1,
    \end{align*}
    which concludes the proof.
\end{proof}

\begin{MyCoro}
\label{coro:data-independent_second_moment}
    Assuming that $\Vert\mathbf{x}_t\Vert_{\infty}\leq 1,t\in[T]$.
    Let $\mathbf{v}$ be any vector in $\mathbb{R}^d$ in Lemma \ref{lemma:ICML25:general_second_order_moment}.
    Then
    \begin{align*}
        \forall i\in[d],\quad
        \vert z_{\tau,i}\vert\leq&\frac{(d-1)(d-2)}{(k-1)(k-2)}\Vert \mathbf{v}\Vert_1,\\
        \mathbb{E}_{\tau}\left[ z^2_{\tau,i}\right]\leq&
        \left(\frac{(d-1)(d-2)}{(k-1)(k-2)}-1\right)\Vert \mathbf{v}\Vert^2_2
        +\left(\frac{d-1}{k-1}-1\right)\left(\Vert \mathbf{v}\Vert^2_1-\Vert \mathbf{v}\Vert^2_2\right).
    \end{align*}
\end{MyCoro}

\begin{proof}[of Corollary \ref{coro:data-independent_second_moment}]
    The analysis on $\vert z_{\tau,i}\vert$
    is same with that of Corollary \ref{coro:data-dependent_second_moment}.
    By Lemma \ref{lemma:ICML25:ratio_probabilities},
        \begin{align*}
        \mathbb{E}_{\tau}\left[ z^2_{\tau,i}\right]
        =&
        \left(\frac{1}{\mathbb{P}[i\in B_{\tau}]}-1\right)x^4_{\tau^2,i}\cdot v^2_i
        +2\sum_{j\ne i}\left(
        \frac{1}{\mathbb{P}[i\in B_{\tau}]}-1\right)x^3_{\tau^2,i}x_{\tau^2,j}\cdot v_iv_j+\\
        &\sum_{j\neq i,r=j}
        \left(\frac{1}{\mathbb{P}[i,j\in B_{\tau}]}-1\right)
        x^2_{\tau^2,i}x^2_{\tau^2,j}\cdot v^2_j+\\
        &\sum_{j\neq i}\sum_{r\neq i,r\neq j}
        \left(\frac{\mathbb{P}[i,j,r\in B_{\tau}]}{\mathbb{P}[i,j\in B_{\tau}]\cdot \mathbb{P}[i,r\in B_{\tau}]}-1\right)
        x^2_{\tau^2,i}x_{\tau^2,j}x_{\tau^2,r}\cdot v_rv_j\\
        \leq&\left(\frac{d-1}{k-1}-1\right)\left(\vert v_i\vert^2+2\sum_{j\neq i}\vert v_iv_j\vert
        +\sum_{j\neq i}\sum_{r\neq i,r\neq j}\vert v_rv_j\vert\right)
        +\sum_{j\neq i,r=j}\left(g_{d,k}-1\right)\vert v_j\vert^2\\
        \leq&\left(g_{d,k}-1\right)\Vert \mathbf{v}\Vert^2_2
        +\left(\frac{d-1}{k-1}-1\right)\left(\Vert \mathbf{v}\Vert^2_1-\Vert \mathbf{v}\Vert^2_2\right),
    \end{align*}
    which concludes the proof.
\end{proof}

Next we prove that if $\gamma_s$ and $\hat{\gamma}_s$ are well selected,
then with a high probability,
$\mathrm{DS}(\hat{\gamma}_s)$ has a solution at least, denoted by $\hat{\mathbf{w}}_s$
satisfying $\Vert\hat{\mathbf{w}}_s\Vert_1\leq \Vert\mathbf{w}^\ast\Vert_1$.

\begin{Mylemma}
\label{lemma:ICML25:optimal_solution_constraint}
    Let $\delta\in(0,1)$.
    For any $s\geq 1$,
    if $\hat{\gamma}_\tau\geq \gamma_\tau$ for all $\tau\leq s$
    where $\gamma_{\tau}$ follows \eqref{eq:ICML25:r_s}
    and $\hat{\gamma}_{\tau}$ follows \eqref{eq:ICML25:hat_gamma_s},
    then
    \begin{enumerate}[(i)]
      \item with probability at least $1-s\left(5+\log_{1.5}\frac{2(d-2)}{3(k-2)}\right)\delta$,
      $\frac{1}{\tau}\left\Vert \hat{\mathbf{X}}_{\mathcal{I}_{\tau}}\mathbf{Y}_{\mathcal{I}_{\tau}}
        -\mathbf{H}_{\mathcal{I}_{\tau}}\mathbf{w}^\ast\right\Vert_{\infty}\leq \gamma_{\tau}$ for all $\tau\leq s$,
      \item $\mathrm{DS}(\hat{\gamma}_{\tau})$
      is feasible and its optimal solution $\hat{\mathbf{w}}_{\tau}$ satisfying
    $\Vert\hat{\mathbf{w}}_{\tau}\Vert_1\leq \Vert\mathbf{w}^\ast\Vert_1$ for all $\tau\leq s$.
    \end{enumerate}
\end{Mylemma}

\begin{proof}[of Lemma \ref{lemma:ICML25:optimal_solution_constraint}]
    The main challenge to prove this lemma is the coupling between (i) and (ii).
    To address this issue,
    we will use an induction method to prove this lemma.

    First, we consider $s=1$.
    In this case,
    $$
        \gamma_1=\left(\frac{8}{3}+2\sigma\right)g_{d,k}\ln\frac{d}{\delta}
        +\left(6.9+1.2\sigma\right)\sqrt{\frac{d-1}{k-1}\ln\frac{d}{\delta}},\quad
        \hat{\mathbf{w}}_0=\frac{1}{d}\mathbf{1}_d.
    $$
    Recalling that $y_1=\langle \mathbf{w}^\ast,\mathbf{x}_1\rangle+\eta_1$.
    With probability at least $1-\delta$,
    \begin{align*}
        \left\Vert\mathbf{h}_1\mathbf{w}^\ast-\hat{\mathbf{x}}_1y_1\right\Vert_{\infty}
        =&\left\Vert\mathbf{h}_1\mathbf{w}^\ast
        -\hat{\mathbf{x}}_1\left(\mathbf{x}^\top_1\mathbf{w}^\ast+\eta_1\right)\right\Vert_{\infty}\\
        =&\left\Vert\left(\mathbf{h}_1-\mathbf{x}_1\mathbf{x}^\top_1\right)\mathbf{w}^\ast
        +\left(\mathbf{x}_1\mathbf{x}^\top_1-\hat{\mathbf{x}}_1
        \mathbf{x}^\top_1\right)\mathbf{w}^\ast
        -\hat{\mathbf{x}}_1\eta_1\right\Vert_{\infty}\\
        \leq&\left\Vert\left(\mathbf{h}_1-\mathbf{x}_1\mathbf{x}^\top_1\right)\mathbf{w}^\ast\right\Vert_{\infty}
        +\left\Vert\left(\mathbf{x}_1-\hat{\mathbf{x}}_1\right)
        \mathbf{x}^\top_1\mathbf{w}^\ast\right\Vert_{\infty}
        +\left\Vert\hat{\mathbf{x}}_1\eta_1\right\Vert_{\infty}\\
        \leq&\max_{i\in[d]}\left\Vert\sum^d_{j=1}(h_1[i,j]-x_{1,i}x_{1,j})w^\ast_j\right\Vert
        +\frac{d-1}{k-1}
        +\left\Vert\hat{\mathbf{x}}_1\right\Vert_{\infty}\cdot \eta_1\\
        \leq&\frac{(d-1)(d-2)}{(k-1)(k-2)}+\frac{d-1}{k-1}
        +\frac{d-1}{k-1}\cdot\sigma\sqrt{2\ln\frac{1}{\delta}}
        <\gamma_1.
    \end{align*}
    For the Gaussian random variable $\eta_1\sim\mathcal{N}(0,\sigma^2)$,
    with probability (w.p.) at least $1-\delta$,
    it must be $\vert\eta_1\vert\leq \sigma\sqrt{2\ln\frac{1}{\delta}}$.
    We will give a detail proof later.
    Since $\gamma_1\leq \hat{\gamma}_1$,
    $\mathrm{DS}(\hat{\gamma}_1)$ has a solution at least,
    i.e., $\hat{\mathbf{w}}_1=\mathbf{w}^\ast$.
    This lemma holds for $s=1$.

    Assuming (i) holds for any $s= r-1\geq 1$,
    that is, w.p. at least $1-(r-1)\left(5+\log_{1.5}\frac{2(d-2)}{3(k-2)}\right)\delta$,
    $\frac{1}{\tau}\left\Vert \hat{\mathbf{X}}_{\mathcal{I}_{\tau}}\mathbf{Y}_{\mathcal{I}_{\tau}}
        -\mathbf{H}_{\mathcal{I}_{\tau}}\mathbf{w}^\ast\right\Vert_{\infty}\leq \gamma_{\tau}$
    for all $\tau\leq r-1$.
    Since $\hat{\gamma}_{\tau}\geq \gamma_{\tau}$ for all $\tau\leq r-1$,
    it is obvious that $\mathrm{DS}(\hat{\gamma}_\tau)$ has a solution $\hat{\mathbf{w}}_{\tau}$ satisfying
    $\Vert\hat{\mathbf{w}}_{\tau}\Vert_1\leq \Vert\mathbf{w}^\ast\Vert_1$ for all $\tau\leq r-1$.
    Thus (ii) also holds for $s=r-1$.
    Next we verify the case of $s=r$.
    Replacing $\mathbf{Y}_{\mathcal{I}_r}$ with
    $\mathbf{X}^\top_{\mathcal{I}_r}\mathbf{w}^\ast+\eta_{\mathcal{I}_r}$
    yields the following inequality
    \begin{align*}
        &\frac{1}{r}\left\Vert\mathbf{H}_{\mathcal{I}_r}\mathbf{w}^\ast
        -\hat{\mathbf{X}}_{\mathcal{I}_r}\mathbf{Y}_{\mathcal{I}_r}\right\Vert_{\infty}\\
        \leq&\underbrace{\frac{1}{r}\left\Vert\left(\mathbf{H}_{\mathcal{I}_r}-\mathbf{X}_{\mathcal{I}_r}\mathbf{X}^\top_{\mathcal{I}_r}\right)\mathbf{w}^\ast\right\Vert_{\infty}}_{\Xi_1}
        +\underbrace{\frac{1}{r}\left\Vert\left(\mathbf{X}_{\mathcal{I}_r}-\hat{\mathbf{X}}_{\mathcal{I}_r}\right)
        \mathbf{X}^\top_{\mathcal{I}_r}\mathbf{w}^\ast\right\Vert_{\infty}}_{\Xi_2}
        +\underbrace{\frac{1}{r}\left\Vert\hat{\mathbf{X}}_{\mathcal{I}_r}\eta_{\mathcal{I}_r}\right\Vert_{\infty}}_{\Xi_3}\\
        =&\frac{1}{r}\left\Vert\sum^r_{\tau=1}\left[\left(\mathbf{H}_{\tau^2}-\mathbf{X}_{\tau^2}
        \mathbf{X}^\top_{\tau^2}\right)\hat{\mathbf{w}}_{\tau-1}+\left(\mathbf{H}_{\tau^2}-\mathbf{X}_{\tau^2}
        \mathbf{X}^\top_{\tau^2}\right)
        \left(\mathbf{w}^\ast-\hat{\mathbf{w}}_{\tau-1}\right)\right]\right\Vert_{\infty}+\Xi_2+\Xi_3\\
        \leq&\underbrace{\frac{1}{r}\left\Vert\sum^r_{\tau=1}
        \left(\mathbf{h}_{\tau^2}-\mathbf{x}_{\tau^2}{\mathbf{x}}^\top_{\tau^2}\right)
        \hat{\mathbf{w}}_{\tau-1}\right\Vert_{\infty}}_{\Xi_{1,1}}
        +\underbrace{\frac{1}{r}
        \left\Vert \sum^r_{\tau=1}\left(\mathbf{h}_{\tau^2}-{\mathbf{x}}_{\tau^2}{\mathbf{x}}^\top_{\tau^2}\right)
        \left(\mathbf{w}^\ast-\hat{\mathbf{w}}_{\tau-1}\right)\right\Vert_{\infty}}_{\Xi_{1,2}}+\Xi_2+\Xi_3.
    \end{align*}
    We separately give an upper bound on $\Xi_1$, $\Xi_2$ and $\Xi_3$.
    The key our analysis is to prove a tighter upper bound on $\Xi_1$ and $\Xi_3$,
    compared to the analysis in \cite{Kale2017Adaptive}.
    $\Xi_1$ can be decomposed into $\Xi_{1,1}$ and $\Xi_{1,2}$.
    By our algorithm-dependent sampling scheme,
    it is possible to give a tight upper bound on $\Xi_{1,1}$.
    The tighter upper bound of $\Xi_3$ comes from a more subtle analysis.

    \noindent\textbf{Analyzing} $\Xi_{1,1}$.
    We define a random vector $\mathbf{z}_{\tau^2}$ as follows
    $$
        \mathbf{z}_{\tau^2}:=\mathbf{h}_{\tau^2}\hat{\mathbf{w}}_{\tau-1}
        -\mathbf{x}_{\tau^2}\mathbf{x}^\top_{\tau^2}\hat{\mathbf{w}}_{\tau-1},\quad
        \tau=1,2,\ldots,r.
    $$
    Given the selections $\{B_1,\ldots,B_{\tau-1}\}$,
    $\hat{\mathbf{w}}_{\tau-1}$ is deterministic.
    By Lemma \ref{lemma:ICML25:sampling_probability},
    it is easy to verify that $\mathbb{E}_{\tau}[z_{\tau^2,i}]=0$ for all $i\in[d]$.
    Therefore, $z_{1,i},z_{2^2,i},\ldots,z_{r^2,i}$ is a sequence of martingale differences.
    By Corollary \ref{coro:data-dependent_second_moment},
    the sum of the conditional variances satisfies
    \begin{align*}
        \sum^r_{\tau=1}\mathbb{E}_{\tau}[z^2_{\tau,i}]
        =\sum^r_{\tau=2}\mathbb{E}_{\tau}[z^2_{\tau,i}]+\mathbb{E}[z^2_{1,i}]
        \leq\frac{2(d-1)}{k-1}
        \left(\sum^r_{\tau=2}\Vert\hat{\mathbf{w}}_{\tau-1}\Vert^2_1
        +\Vert\hat{\mathbf{w}}_0\Vert^2_1\right)
        \leq2r\frac{d-1}{k-1},
    \end{align*}
    in which
    $\Vert \hat{\mathbf{w}}_{\tau}\Vert_1\leq \Vert\mathbf{w}^\ast\Vert_1\leq 1$ for all $\tau=1,\ldots,r-1$.
    By Corollary \ref{coro:data-dependent_second_moment}, we have $\vert z_{\tau^2,i}\vert\leq g_{d,k}$.
    By Lemma \ref{lemma:ICML25:Hoeffding_inequality} and
    the union-of-events bound over $i\in[d]$,
    w.p. at least $1-\delta$,
    $$
        \Xi_{1,1} \leq \frac{2 g_{d,k}}{3r}\ln\frac{d}{\delta}
        +\frac{2}{\sqrt{r}}\sqrt{\frac{d-1}{k-1}\ln\frac{d}{\delta}}.
    $$
    \noindent\textbf{Analyzing} $\Xi_{1,2}$.
    We redefine $\mathbf{z}_{\tau^2}$ as follows
    \begin{align*}
        \mathbf{z}_{\tau^2}:=&\mathbf{h}_{\tau^2}\Delta_{\tau-1}
        -\mathbf{x}_{\tau^2}\mathbf{x}^\top_{\tau^2}\Delta_{\tau-1},\quad \tau=1,2,\ldots,r.
    \end{align*}
    Given the selections $\{B_1,\ldots,B_{\tau-1}\}$,
    $\Delta_{\tau-1}$ is deterministic.
    It is easy to be verified that $\mathbb{E}_{\tau}[z_{\tau^2,i}]=0$ for all $i\in[d]$.
    Next we analyze the sum of conditional variances.
    Recalling that $\Delta_{\tau-1}(S^c)=\hat{\mathbf{w}}_{\tau}(S^c)$.
    We further decompose $z_{\tau^2,i}$ into two components
    $$
        \underbrace{\sum_{j\in S}\left(
        \frac{x_{\tau^2,i}x_{\tau^2,j}}{\mathbb{P}[i,j\in B_{\tau}]}
        \mathbb{I}_{i,j\in B_{\tau}}-x_{\tau^2,i}x_{\tau^2,j}\right)\Delta_{\tau-1,j}}_{:=z_{\tau^2,S}}
        +\underbrace{\sum_{j\in S^c}\left(
        \frac{x_{\tau^2,i}x_{\tau^2,j}}{\mathbb{P}[i,j\in B_{\tau}]}
        \mathbb{I}_{i,j\in B_{\tau}}-x_{\tau^2,i}x_{\tau^2,j}\right)\hat{w}_{\tau-1,j}}_{:=z_{\tau^2,S^c}}.
    $$
    Without loss of generality, assuming that $i\in S$.
    By
    Lemma \ref{lemma:ICML25:ratio_probabilities},
    Lemma \ref{lemma:ICML25:general_second_order_moment},
    Corollary \ref{coro:data-dependent_second_moment}
    and Corollary \ref{coro:data-independent_second_moment},
    we have
    \begin{align*}
        &\mathbb{E}_{\tau}[z^2_{\tau,i}]\\
        =&\mathbb{E}_{\tau}\left[
        \left(z_{\tau^2,S}\right)^2\right]+\mathbb{E}_{\tau}\left[\left(z_{\tau^2,S^c}\right)^2\right]+
        2\mathbb{E}_{\tau}\left[z_{\tau^2,S}\cdot z_{\tau^2,S^c}\right]\\
        \leq&g_{d,k}\Vert \Delta_{\tau-1}(S)\Vert^2_1
        +\frac{2(d-1)}{k-1}\Vert\hat{\mathbf{w}}_{\tau-1}\Vert_1\cdot \Vert\hat{\mathbf{w}}_{\tau-1}(S^c)\Vert_1+
        \frac{2(d-1)}{k-1}\Vert\Delta_{\tau-1}(S)\Vert_1\cdot
        \Vert\hat{\mathbf{w}}_{\tau-1}(S^c)\Vert_1\\
        \leq&g_{d,k}\Vert \Delta_{\tau-1}(S)\Vert^2_1
        +\frac{2(d-1)}{k-1}\Vert\hat{\mathbf{w}}_{\tau-1}\Vert_1\cdot \Vert\hat{\mathbf{w}}_{\tau-1}(S^c)\Vert_1+\\
        &\frac{2(d-1)}{k-1}\left(\Vert\mathbf{w}^\ast\Vert_1+\Vert\hat{\mathbf{w}}_{\tau-1}(S)\Vert_1\right)\cdot
        \left(\Vert\hat{\mathbf{w}}_{\tau-1}\Vert_1-\Vert\hat{\mathbf{w}}_{\tau-1}(S)\Vert_1\right)\\
        =&g_{d,k}\Vert \Delta_{\tau-1}(S)\Vert^2_1
        +\frac{4(d-1)}{k-1}\Vert\mathbf{w}^\ast\Vert^2_1
        -\frac{2(d-1)}{k-1}\Vert\hat{\mathbf{w}}_{\tau-1}(S)\Vert_1\cdot \left(\Vert\mathbf{w}^\ast\Vert_1+\Vert\hat{\mathbf{w}}_{\tau-1}(S)\Vert_1\right)\\
        \leq&g_{d,k}\Vert \Delta_{\tau-1}(S)\Vert^2_1+\frac{4(d-1)}{k-1}.
    \end{align*}
    Note that $\Vert \Delta_{\tau-1}(S)\Vert^2_1\in(0,4]$ is a random variable.
    By Lemma \ref{lemma:ICML25:sampling_probability}, we have
    $$
        \vert z_{\tau^2,i}\vert
        =\left\vert\sum^d_{j=1}\left(
        \frac{x_{\tau^2,i}x_{\tau^2,j}}{\mathbb{P}[i,j\in B_{\tau}]}
        \mathbb{I}_{i,j\in B_{\tau}}-x_{\tau^2,i}x_{\tau^2,j}\right)\Delta_{\tau-1,j}\right\vert
        \leq\frac{(d-1)(d-2)}{(k-1)(k-2)}\Vert \Delta_{\tau-1}\Vert_1
        \leq2g_{d,k}.
    $$
    By Lemma \ref{lemma:ICML25:New_Bernstein_inequality} and the union-of-events bound over $i\in[d]$,
    w.p. at least $1-\left(1+\log_{\beta}\frac{d+k}{k-2}\right)\delta$,
    $$
        \Xi_{1,2} \leq\frac{4 g_{d,k}}{3 r}\ln\frac{d}{\delta}
        +\frac{1}{r}\sqrt{2\beta g_{d,k}\sum^r_{\tau=1}\Vert\Delta_{\tau-1}(S)\Vert^2_1\ln\frac{d}{\delta}}
        +\sqrt{\frac{8\beta(d-1)}{r(k-1)}\ln\frac{d}{\delta}}.
    $$
    Let $\beta=1.5$.
    Combining the upper bounds on $\Xi_{1,1}$ and $\Xi_{1,2}$,
    w.p. at least $1-\left(2+\log_{1.5}\frac{d+4}{k-2}\right)\delta$,
    $$
        \Xi_1 \leq
        \frac{2 g_{d,k}}{r}\ln\frac{d}{\delta}
        +\left(2+\sqrt{12}\right)\sqrt{\frac{d-1}{r(k-1)} \ln\frac{d}{\delta}}
        +\frac{1}{r}\sqrt{3 g_{d,k}\sum^r_{\tau=1}\Vert\Delta_{\tau-1}(S)\Vert^2_1\ln\frac{d}{\delta}}.
    $$
    \textbf{Analyzing} $\Xi_2$.
    Similarly, we redefine $\mathbf{z}_{\tau^2}$ as follows
    $$
        \mathbf{z}_{\tau^2}:={\mathbf{x}}_{\tau^2}\hat{\mathbf{x}}^\top_{\tau^2}\mathbf{w}^\ast
        -{\mathbf{x}}_{\tau^2}{\mathbf{x}}^\top_{\tau^2}\mathbf{w}^\ast,\quad \tau=1,2,\ldots,r.
    $$
    It is obvious
    $\vert z_{\tau^2,i}\vert\leq \frac{d-1}{k-1}$ for all $i\in[d]$.
    By Lemma \ref{lemma:ICML25:sampling_probability},
    the sum of conditional variances is upper bounded by $\frac{d-1}{k-1}r$.
    By Lemma \ref{lemma:ICML25:Hoeffding_inequality}
    and the union-of-events bound over $i\in[d]$,
    w.p. at least $1-\delta$,
    $$
        \Xi_2 \leq \frac{2(d-1)}{3(k-1)r}\ln\frac{d}{\delta}
        +\frac{1}{\sqrt{r}}\sqrt{\frac{2(d-1)}{k-1}\ln\frac{d}{\delta}}.
    $$
    \textbf{Analyzing} $\Xi_3$.
    For each $i\in[d]$,
    we define a random variable $z_{r,i}$ as follows
    $$
        z_{r,i}=\sum^r_{\tau=1}\eta_{\tau^2}\hat{x}_{\tau^2,i}, \quad i=1,2,\ldots,d.
    $$
    Given that $\eta_1, \eta_4,\ldots,\eta_{r^2}$ are independent Gaussian variables,
    we have $\mathbb{E}_{\eta_{\tau^2},\eta_{t^2}}[\eta_{\tau^2}\hat{x}_{\tau^2,i}\cdot \eta_{t^2}\hat{x}_{t^2,i}]=
    \mathbb{E}_{\eta_{t^2}}[\eta_{t^2}]\mathbb{E}_{\eta_{\tau^2}}[\eta_{\tau^2}\hat{x}_{\tau^2,i}\hat{x}_{t^2,i}]=0$
    for any $\tau\neq t$ in which we assume $\tau < t$.
    Thus $\eta_{\tau^2}\hat{x}_{\tau^2,i}$ is also independent of $\eta_{t^2}\hat{x}_{t^2,i}$,
    and
    $z_{r,i}$ is a Gaussian random variable with $\mathbb{E}[z_{r,i}]=0$
    and
    $$
        \mathbb{E}[(z_{r,i})^2]
        =\sum^r_{\tau=1}\mathbb{E}[\eta^2_{\tau^2}]\hat{x}^2_{\tau^2,i}
        =\sigma^2\cdot\sum^r_{\tau=1}\hat{x}^2_{\tau^2,i}
        \leq\frac{d-1}{k-1}\sigma^2 r+
        \sigma^2\cdot\sum^r_{\tau=1}\underbrace{\left(\hat{x}^2_{\tau^2,i}-\frac{x^2_{\tau^2,i}}{\mathbb{P}[i\in B_{\tau}]}\right)}_{:=a_{\tau,i}}.
    $$
    It can be proved that $\vert a_{\tau,i}\vert\leq \frac{(d-1)^2}{(k-1)^2}$
    and $\sum^s_{\tau=1}\mathbb{E}[a^2_{\tau,i}]
    \leq \frac{(d-1)^3}{(k-1)^3}s$.
    By Lemma \ref{lemma:ICML25:Hoeffding_inequality},
    w.p. at least $1-\delta$,
    \begin{align*}
        \mathbb{E}[(z_{r,i})^2]
        \leq \frac{d-1}{k-1}\sigma^2 r+\frac{2(d-1)^2}{3(k-1)^2}\sigma^2\ln\frac{1}{\delta}
        +\sigma^2\sqrt{\frac{2r(d-1)^3}{(k-1)^3}\ln\frac{1}{\delta}}.
    \end{align*}
    Denote by $\Sigma_{r,i}$ the standard variance of $z_{r,i}$.
    For Gaussian random variables,
    we have
    \begin{align*}
        \forall z_0>0,~\mathbb{P}[\vert z_{r,i}\vert >z_0]
        =&2\int^{\infty}_{z_0}\frac{1}{\sqrt{2\pi}\Sigma_{r,i}}\exp\left(-\frac{z^2_{r,i}}{2\Sigma^2_{r,i}}\right)
        \mathrm{d}\,z_{r,i}\\
        =&-\frac{2\Sigma_{r,i}}{\sqrt{2\pi}z_{r,i}}\exp\left(-\frac{z^2_{r,i}}{2\Sigma^2_{r,i}}\right)\left\vert^{\infty}_{z_0} \right.
        -2\int^{\infty}_{z_0}\frac{\Sigma_{r,i}}{\sqrt{2\pi}z^2_{r,i}}\exp\left(-\frac{z^2_{r,i}}{2\Sigma^2_{r,i}}\right)
        \mathrm{d}\,z_{r,i}\\
        \leq&2\frac{\Sigma_{r,i}}{\sqrt{2\pi}z_0}\exp\left(-\frac{z^2_0}{2\Sigma^2_{r,i}}\right).
    \end{align*}
    For any $\delta\in(0,1)$,
    let $z_0=\Sigma_{r,i}\sqrt{2\ln\frac{1}{\delta}}$.
    Then we have
    $$
        \mathbb{P}[\vert z_{r,i}\vert >z_0]
        \leq \frac{1}{\sqrt{\pi}\sqrt{\ln\frac{1}{\delta}}}\delta<\delta.
    $$
    By the union-of-events bound over $i\in[d]$,
    with probability at least $1-2\delta$,
    \begin{align*}
        \Xi_3
        \leq&\frac{1}{r}
        \sqrt{\frac{d-1}{k-1} r+\frac{2(d-1)^2}{3(k-1)^2}\ln\frac{1}{\delta}
        +\sqrt{\frac{2r(d-1)^3}{(k-1)^3}\ln\frac{1}{\delta}}}
        \cdot \sigma\sqrt{2\ln\frac{d}{\delta}}\\
        \leq&\frac{1}{r}
        \sqrt{\frac{d-1}{k-1} r+\frac{2(d-1)^2}{3(k-1)^2}\ln\frac{1}{\delta}
        +\frac{1}{2}\left(\frac{d-1}{k-1}\cdot\frac{3r}{4}+\frac{8(d-1)^2}{3(k-1)^2}\ln\frac{d}{\delta}\right)}
        \cdot \sigma\sqrt{2\ln\frac{d}{\delta}}\\
        \leq&\frac{1.2\sigma}{\sqrt{r}}\sqrt{\frac{d-1}{k-1}\ln\frac{d}{\delta}}
        +\frac{2\sigma}{r}\frac{d-1}{k-1}\ln\frac{d}{\delta},
    \end{align*}
    where we use the inequality $2\sqrt{ab}\leq a+b$ for $a>0, b>0$ to simplify the term in the square root.

\noindent \textbf{Combining the upper bounds on $\Xi_{1,1}, \Xi_{1,2}, \Xi_2$ and $\Xi_3$},
    w.p. at least $1-\left(5+\log_{1.5}\frac{2(d-2)}{3(k-2)}\right)\delta$,
    \begin{align*}
        &\frac{1}{r}\left\Vert\mathbf{H}_{\mathcal{I}_r}\mathbf{w}^\ast
        -\hat{\mathbf{X}}_{\mathcal{I}_r}\mathbf{Y}_{\mathcal{I}_r}\right\Vert_{\infty}\\
        \leq&\frac{2 g_{d,k}}{r}\ln\frac{d}{\delta}
        +\left(2+\sqrt{12}\right)\sqrt{\frac{d-1}{r(k-1)} \ln\frac{d}{\delta}}
        +\frac{1}{r}\sqrt{3 g_{d,k}\sum^r_{\tau=1}\Vert\Delta_{\tau-1}(S)\Vert^2_1\ln\frac{d}{\delta}}+\\
        &\frac{2(d-1)}{3(k-1)r}\ln\frac{d}{\delta}
        +\frac{1}{\sqrt{r}}\sqrt{\frac{2(d-1)}{k-1}\ln\frac{d}{\delta}}+
        \frac{1.2\sigma}{\sqrt{r}}\sqrt{\frac{d-1}{k-1}\ln\frac{d}{\delta}}
        +\frac{2\sigma}{r}\frac{d-1}{k-1}\ln\frac{d}{\delta}\\
        \leq&\left(\frac{8}{3}+2\sigma\right)\frac{g_{d,k}}{r}\ln\frac{d}{\delta}
        +(6.9+1.2\sigma)\sqrt{\frac{d-1}{r(k-1)} \ln\frac{d}{\delta}}+
        \frac{1}{r}\sqrt{3 g_{d,k}\sum^r_{\tau=1}\Vert\Delta_{\tau-1}(S)\Vert^2_1\ln\frac{d}{\delta}}
        =\gamma_r.
    \end{align*}
    Taking the union-of-events bound over $\tau\leq r-1$ and $\tau =r$,
    w.p. at least $1-r\left(5+\log_{1.5}\frac{2(d-2)}{3(k-2)}\right)\delta$,
    (i) holds for $s=r$.
    Since $\gamma_r\leq \hat{\gamma}_r$,
    $\mathrm{DS}(\hat{\gamma}_r)$ has a solution
    $\hat{\mathbf{w}}_r$ satisfying $\Vert \hat{\mathbf{w}}_r\Vert_1\leq \Vert\mathbf{w}^\ast\Vert_1$.
    Thus (ii) also holds for $s=r$.

    Therefore, the lemma holds for all $s\geq 1$,
    concluding the proof.
\end{proof}

\begin{Mylemma}
\label{lemma:ICML2025:upper_bounded_(XX-H)DELTA}
    For any $s\geq 1$,
    with probability at least $1-\delta$,
    $$
        \frac{1}{s}\left\Vert\left({\mathbf{X}}_{\mathcal{I}_s}{\mathbf{X}}^\top_{\mathcal{I}_s}
        -\mathbf{H}_{\mathcal{I}_s}\right)\Delta_s\right\Vert_{\infty}
        \leq\frac{2g_{d,k}}{3s}\Vert \Delta_s\Vert_1\ln\frac{d^2}{\delta}
        +\Vert \Delta_s\Vert_1\sqrt{\frac{2g_{d,k}}{s}\cdot \ln\frac{d^2}{\delta}}.
    $$
\end{Mylemma}

\begin{proof}[of Lemma \ref{lemma:ICML2025:upper_bounded_(XX-H)DELTA}]
    As $\Delta_s$ is not independent with $\mathbf{h}_{\tau^2}$, $\tau=1,2,\ldots,s$,
    it is necessary to give the element-wise bound for the matrix
    $\mathbf{X}_{\mathcal{I}_s}\mathbf{X}^\top_{\mathcal{I}_s}-\mathbf{H}_{\mathcal{I}_s}$.
    We first analyze the elements on the diagonal.
    Let $z_{\tau,i}=\mathbf{h}_{\tau^2}[i,i]-{\mathbf{x}}^2_{\tau^2,i}$.
    It is obvious that
    \begin{align*}
        \forall i\in[d],\quad\mathbb{E}_{\tau}[z_{\tau,i}]=0,\quad
        \vert z_{\tau,i}\vert\leq \frac{d-1}{k-1},\quad
        \sum^s_{\tau=1}\mathbb{E}_{\tau}[\vert z_{\tau,i}\vert^2]
        \leq \sum^s_{\tau=1}\frac{x^4_{\tau^2,i}}{\mathbb{P}[i\in  B_{\tau}]}
        \leq \frac{d-1}{k-1}s.
    \end{align*}
    By Lemma \ref{lemma:ICML25:Hoeffding_inequality},
    with probability at least $1-\delta$,
    $$
        \left\vert\sum^s_{\tau=1}z_{\tau,i} \right\vert\leq \frac{2(d-1)}{3(k-1)}\ln\frac{2}{\delta}
        +\sqrt{2s\frac{d-1}{k-1}\ln\frac{2}{\delta}}.
    $$
    Next we analyze the non-diagonal elements.
    Let
    $z_{\tau,i,j}=\mathbf{h}_{\tau^2}[i,j]-{\mathbf{x}}_{\tau^2,i}{\mathbf{x}}_{\tau^2,j}$.
    We have
    \begin{align*}
        \forall i\neq j\in [d],\quad\mathbb{E}_{\tau}[z_{\tau,i,j}]=0,\quad\vert z_{\tau,i,j}\vert\leq g_{d,k},\quad
        \sum^s_{\tau=1}\mathbb{E}_{\tau}[\vert z_{\tau,i,j}\vert^2]
        = \sum^s_{\tau=1}\frac{x^2_{\tau^2,i}x^2_{\tau^2,j}}{\mathbb{P}[i,j\in  B_{\tau}]}
        \leq g_{d,k}\cdot s.
    \end{align*}
    By Lemma \ref{lemma:ICML25:Hoeffding_inequality},
    with probability at least $1-\delta$,
    $$
        \left\vert\sum^s_{\tau=1}z_{\tau,i,j}\right\vert \leq \frac{2(d-1)(d-2)}{3(k-1)(k-2)}\ln\frac{2}{\delta}
        +\sqrt{2s\frac{(d-1)(d-2)}{(k-1)(k-2)}\ln\frac{2}{\delta}}.
    $$
    Since the matrix is symmetry,
    it is enough to take the union-of-events bound over $d+\frac{d(d-1)}{2}$ events.
    With probability at least $1-\delta$,
    \begin{align*}
        &\frac{1}{s}\left\Vert\left(\mathbf{X}_{\mathcal{I}_s}\mathbf{X}^\top_{\mathcal{I}_s}
        -\mathbf{H}_{\mathcal{I}_s}\right)\Delta_s\right\Vert_{\infty}\\
        \leq& \frac{1}{s}\max_{i\in [d]}\left\vert\sum^s_{\tau=1}z_{\tau,i}\Delta_i+
        \sum^s_{\tau=1}\sum_{j\neq i}z_{\tau,i,j}\Delta_j\right\vert\\
        \leq& \frac{1}{s}\max_{i\in [d]}\left(\left\vert \sum^s_{\tau=1}z_{\tau,i}\right\vert
        \cdot\vert\Delta_i\vert+
        \sum_{j\neq i}\left\vert\sum^s_{\tau=1}z_{\tau,i,j}\right\vert\cdot\vert\Delta_j\vert\right)\\
        \leq& \frac{1}{s}\max_{i\in [d]}
        \left(\frac{2(d-1)(d-2)}{3(k-1)(k-2)}\ln\frac{d(d+1)}{\delta}
        +\sqrt{2s\frac{(d-1)(d-2)}{(k-1)(k-2)}\ln\frac{d(d+1)}{\delta}}\right)
        \sum^j_{i=1}\vert\Delta_i\vert\\
        \leq&\frac{2g_{d,k}}{3s}\Vert \Delta_s\Vert_1\ln\frac{d(d+1)}{\delta}
        +\Vert \Delta_s\Vert_1\sqrt{\frac{2 g_{d,k}}{s}\cdot \ln\frac{d(d+1)}{\delta}},
    \end{align*}
    which concludes the proof.
\end{proof}

\begin{Mylemma}
\label{lemma:ICML2025:upper_bound_XXDelta_s}
    For any $s\geq 1$,
    if $\hat{\gamma}_{\tau}\geq \gamma_{\tau}$ for all $\tau\leq s$,
    then w.p. at least $1-s\left(6+\log_{1.5}\frac{2(d-2)}{3(k-2)}\right)\delta$,
    \begin{align*}
        \forall \tau\leq s,\quad
        &\left\Vert \frac{1}{\tau}\mathbf{X}_{\tau}\mathbf{X}^\top_{\tau}\Delta_{\tau}\right\Vert_{\infty}
        \leq
        2\left(\frac{8}{3}+2\sigma\right)\frac{g_{d,k}}{\tau}\ln\frac{d}{\delta}
        +2\left(6.9+1.2\sigma\right)\sqrt{\frac{d-1}{\tau(k-1)}\ln\frac{d}{\delta}}+\nu_{\tau}+\\
        &\qquad\qquad\frac{1}{\tau}\sqrt{3 g_{d,k}\sum^{\tau}_{r=1}\Vert\Delta_{r-1}(S)\Vert^2_1\ln\frac{d}{\delta}}
        +\frac{2g_{d,k}}{3\tau}\Vert \Delta_{\tau}\Vert_1\ln\frac{d^2}{\delta}+
        \sqrt{\frac{2g_{d,k}}{\tau}\ln\frac{d^2}{\delta}}\Vert\Delta_{\tau}\Vert_1.
    \end{align*}
\end{Mylemma}

\begin{proof}[of Lemma \ref{lemma:ICML2025:upper_bound_XXDelta_s}]
    We decompose $\left\Vert \frac{1}{\tau}\mathbf{X}_{\mathcal{I}_{\tau}}\mathbf{X}^\top_{\mathcal{I}_{\tau}}\Delta_{\tau}\right\Vert_{\infty}$
    into three components.
    \begin{align*}
        &\left\Vert \frac{1}{\tau}\mathbf{X}_{\mathcal{I}_{\tau}}\mathbf{X}^\top_{\mathcal{I}_{\tau}}\Delta_{\tau}\right\Vert_{\infty}\\
        =&\left\Vert \frac{1}{\tau}\mathbf{H}_{\mathcal{I}_{\tau}}\Delta_{\tau}\right\Vert_{\infty}
        +\left\Vert\frac{1}{\tau}\left(\mathbf{X}_{\mathcal{I}_{\tau}}\mathbf{X}^\top_{\mathcal{I}_{\tau}}
        -\mathbf{H}_{\mathcal{I}_{\tau}}\right)\Delta_{\tau}\right\Vert_{\infty}\\
        =&\frac{1}{\tau}\left\Vert \mathbf{H}_{\mathcal{I}_{\tau}}\hat{\mathbf{w}}_{\tau}
        -\hat{\mathbf{X}}^\top_{\mathcal{I}_{\tau}}\mathbf{Y}_{\mathcal{I}_{\tau}}
        -\left(\mathbf{H}_{\mathcal{I}_{\tau}}\mathbf{w}^\ast-
        \hat{\mathbf{X}}^\top_{\mathcal{I}_{\tau}}\mathbf{Y}_{\mathcal{I}_{\tau}}\right)\right\Vert_{\infty}
        +\left\Vert\frac{1}{\tau}\left(\mathbf{X}_{\mathcal{I}_{\tau}}\mathbf{X}^\top_{\mathcal{I}_{\tau}}
        -\mathbf{H}_{\mathcal{I}_{\tau}}\right)\Delta_{\tau}\right\Vert_{\infty}\\
        \leq&\frac{1}{\tau}\left\Vert \mathbf{H}_{\mathcal{I}_{\tau}}\hat{\mathbf{w}}_{\tau}
        -\hat{\mathbf{X}}^\top_{\mathcal{I}_{\tau}}\mathbf{Y}_{\mathcal{I}_{\tau}}\right\Vert_{\infty}
        +\frac{1}{\tau}\left\Vert\mathbf{H}_{\mathcal{I}_{\tau}}\mathbf{w}^\ast
        -\hat{\mathbf{X}}^\top_{\mathcal{I}_{\tau}}\mathbf{Y}_{\mathcal{I}_{\tau}}\right\Vert_{\infty}
        +\left\Vert\frac{1}{\tau}\left(\mathbf{X}_{\mathcal{I}_{\tau}}\mathbf{X}^\top_{\mathcal{I}_{\tau}}
        -\mathbf{H}_{\mathcal{I}_{\tau}}\right)\Delta_{\tau}\right\Vert_{\infty}.
    \end{align*}
    If $\hat{\gamma}_{\tau}\geq \gamma_{\tau}$ for all $\tau\leq s$,
    then Lemma \ref{lemma:ICML25:optimal_solution_constraint}
    ensures that the second component is upper bounded by $\gamma_{\tau}$,
    and the first component is upper bounded by $\hat{\gamma}_{\tau}$.
    By Lemma \ref{lemma:ICML2025:upper_bounded_(XX-H)DELTA},
    we can also obtain an upper bound on the third component.
    Combining all results concludes the proof.
\end{proof}

    \begin{Mylemma}
    \label{lemma:ICML25:Dantzig2005}
        Let $S=\mathrm{Supp}(\mathbf{w}^\ast)$.
        For any $s\geq 1$,
        if $\hat{\gamma}_{\tau}\geq \gamma_{\tau}$ for all $\tau\leq s$,
        then w.p. at least $1-s\left(5+\log_{1.5}\frac{2(d-2)}{3(k-2)}\right)\delta$,
        $$
            \forall \tau\leq s,\quad\Vert \Delta_{\tau}(S^c)\Vert_1 \leq \Vert \Delta_{\tau}(S)\Vert_1,
        $$
        in which $\Delta_{\tau}=\hat{\mathbf{w}}_{\tau}-\mathbf{w}^\ast$
        and $\hat{\mathbf{w}}_{\tau}$ be the solution of $\mathrm{Ds}(\hat{\gamma}_{\tau})$.
    \end{Mylemma}

    Recalling the definition of the restricted set $\Omega_S$
    in Assumption \ref{ass:ICML25:restricted_eigenvalue},
    $\Delta_{\tau}$ belongs to $\Omega_S$ with $\alpha=1$.
    In this way, it is possible to use the $(\delta_S,S,\alpha)$-compatibility condition.
    \begin{proof}[of Lemma \ref{lemma:ICML25:Dantzig2005}]
        By Lemma \ref{lemma:ICML25:optimal_solution_constraint},
        we obtain, w.p. at least $1-s\left(5+\log_{1.5}\frac{2(d-2)}{3(k-2)}\right)\delta$,
        $\Vert\hat{\mathbf{w}}_{\tau}\Vert_1\leq \Vert\mathbf{w}^\ast\Vert_1$
        for all $\tau\leq s$.
        Since $w^\ast_i=0$ for $i\in S^c$.
        We have
        $$
            \Vert \Delta_{\tau}(S^c)\Vert_1
            =\Vert \hat{\mathbf{w}}_{\tau}(S^c)\Vert_1
            =\Vert \hat{\mathbf{w}}_{\tau}\Vert_1-\Vert \hat{\mathbf{w}}_{\tau}(S)\Vert_1
            \leq\Vert \mathbf{w}^\ast(S)\Vert_1-\Vert \hat{\mathbf{w}}_{\tau}(S)\Vert_1
            \leq\Vert \mathbf{w}^\ast(S)- \hat{\mathbf{w}}_{\tau}(S)\Vert_1,
        $$
        which concludes the proof.
    \end{proof}

\section{Proof of Lemma \ref{lemma:estimator_error:DS-OSLRC}}

\begin{proof}[of Lemma \ref{lemma:estimator_error:DS-OSLRC}]
    Let $s'\in[1,\lfloor\sqrt{T}\rfloor]$.
    For any $s\in[1,s']$,
    we will separately give a lower bound
    and an upper bound on $\frac{1}{s}\Vert \mathbf{X}^\top_{\mathcal{I}_s}\Delta_s\Vert^2_2$.
    If $\hat{\gamma}_s\geq \gamma_s$ for all $1\leq s\leq s'$,
    then by Lemma \ref{lemma:ICML25:Dantzig2005} and Assumption \ref{ass:ICML25:restricted_eigenvalue},
    we have,
    w.p. at least $1-s'\left(5+\log_{1.5}\frac{d+k}{k-2}\right)\delta$,
    a lower bound is given as follows,
    $$
        \forall s\in[1,s'],\quad\frac{1}{s}\Vert \mathbf{X}^\top_{\mathcal{I}_s}\Delta_s\Vert^2_2
        \geq \frac{\delta^2_S}{\vert S\vert}\Vert \Delta_s(S)\Vert^2_1
        \geq \frac{\delta^2_S}{k}\Vert \Delta_s(S)\Vert^2_1,
    $$
    in which $\vert S\vert=\Vert\mathbf{w}^\ast\Vert_0\leq k$.
    Besides,
    an upper bound is as follows,
    \begin{align*}
        \forall s\in[1,s'],\quad\frac{1}{s}\Vert \mathbf{X}^\top_{\mathcal{I}_s}\Delta_s\Vert^2_2
        \leq\left\Vert \frac{1}{s}\Delta^\top_s\mathbf{X}_{\mathcal{I}_s}
        \mathbf{X}^\top_{\mathcal{I}_s}\right\Vert_{\infty}\Vert\Delta_s\Vert_1
        \leq2\left\Vert \frac{1}{s}\mathbf{X}_{\mathcal{I}_s}
        \mathbf{X}^\top_{\mathcal{I}_s}\Delta_s\right\Vert_{\infty}\Vert\Delta_s(S)\Vert_1.
    \end{align*}
    Lemma \ref{lemma:ICML2025:upper_bound_XXDelta_s} further provides
    an upper bound on $\left\Vert \frac{1}{s}\mathbf{X}_{\mathcal{I}_s}
        \mathbf{X}^\top_{\mathcal{I}_s}\Delta_s\right\Vert_{\infty}$.
    Combining the lower bound and upper bound,
    we obtain,
    w.p. at least $1-s'\left(6+\log_{1.5}\frac{d+k}{k-2}\right)\delta$,
    \begin{equation}
    \label{eq:ICML2025:generaly_bound_on_estimator_error}
    \begin{split}
        \forall s\leq s',\quad&\frac{\delta^2_S\Vert \Delta_s(S)\Vert_1}{2k}
        \leq 2\left(\frac{8}{3}+2\sigma\right)\frac{g_{d,k}}{s}\ln\frac{d}{\delta}
        +2\left(6.9+1.2\sigma\right)\sqrt{\frac{d-1}{s(k-1)}\ln\frac{d}{\delta}} +\nu_s+\\
        &\quad\frac{1}{s}\sqrt{3 g_{d,k}\sum^s_{\tau=1}\Vert\Delta_{\tau-1}(S)\Vert^2_1\ln\frac{d}{\delta}}
        +\frac{2g_{d,k}}{3s}\Vert \Delta_s\Vert_1\ln\frac{d^2}{\delta}
        +\sqrt{\frac{2g_{d,k}}{s}\ln\frac{d^2}{\delta}}\Vert\Delta_s\Vert_1.
    \end{split}
    \end{equation}
    Solving the inequality \eqref{eq:ICML2025:generaly_bound_on_estimator_error}
    will provide an explicit upper bound on $\Vert\Delta_s(S)\Vert_1$.
    However, it is highly non-trivial,
    as the right-hand side of the inequality depends on
    $\sqrt{\sum^s_{\tau=1}\Vert\Delta_{\tau-1}(S)\Vert^2_1}$
    and $\nu_s$ that will be estimated dynamically.
    An obvious upper bound can be derived by using the inequality
    $\Vert\Delta_{\tau}(S)\Vert_1\leq 2$ for all $\tau \leq s-1$.
    However, such a simple analysis overlooked the fact
    that $\Vert\Delta_{\tau}(S)\Vert_1$ becomes smaller as $\tau$ increases.
    We will carefully control
    $\sqrt{\sum^s_{\tau=1}\Vert\Delta_{\tau-1}(S)\Vert^2_1}$ by an induction method,
    naturally yielding a tighter upper bound of $\Vert \Delta_s(S)\Vert_1$.
    Next we consider three cases.
    \textbf{It is crucial to verify $\hat{\gamma}_s \geq \gamma_s$ for all $s \in[1,s']$}.

\subsection{Case 1: $s\leq s_0$}

    Let $s'=s_0$.
    We use the trivial upper bound $\Vert\Delta_{\tau}(S)\Vert_1 \leq 2$.
    Recalling that $\nu_s$ is defined as follows,
    $$
        \forall s\leq s_0,\quad \nu_s
        =\frac{2}{\sqrt{s}}\sqrt{3 g_{d,k}\ln\frac{d}{\delta}}
        \geq
        \frac{1}{s}\sqrt{3 g_{d,k}\sum^s_{\tau=1}\Vert\Delta_{\tau-1}(S)\Vert^2_1\ln\frac{d}{\delta}}
        \Rightarrow \hat{\gamma}_s\geq \gamma_s.
    $$
    Therefore, for any $s\leq s_0$,
    \textbf{by Lemma \ref{lemma:ICML25:optimal_solution_constraint},
    w.p. at least $1-s_0\left(6+\log_{1.5}\frac{d+k}{k-2}\right)\delta$,
    $\mathrm{DS}(\hat{\gamma}_s)$
    has a feasible solution at least}, i.e., $\mathbf{w}^\ast$.
    It is proper to derive the inequality \eqref{eq:ICML2025:generaly_bound_on_estimator_error}.
    We further obtain
    $$
        \frac{\delta^2_S\Vert \Delta_s(S)\Vert_1}{2k}
        \leq\left(8+4\sigma\right)\frac{g_{d,k}}{s}\ln\frac{d}{\delta}
        +2\left(6.9+1.2\sigma\right)\sqrt{\frac{(d-1)\ln\frac{d}{\delta}}{s(k-1)}}+
        \frac{4+4\sqrt{3}}{\sqrt{s}}\sqrt{g_{d,k}\ln\frac{d}{\delta}},
    $$
    where we use the inequality $\ln\frac{d^2}{\delta}\leq 2\ln\frac{d}{\delta}$.
    Rearranging terms concludes the desired result.

\subsection{Case 2: $s_0<s\leq s_1$}

    Let $s'=s_1$.
    We start from \eqref{eq:ICML2025:generaly_bound_on_estimator_error}.
    Substituting into the value of $s_0$,
    the inequality \eqref{eq:ICML2025:inequality_s_0} holds.
    \begin{equation}
    \label{eq:ICML2025:inequality_s_0}
        \forall s> s_0,\quad
        4k\sqrt{\frac{2g_{d,k}}{s}\cdot \ln\frac{d^2}{\delta}}
        \leq \frac{1}{4}\delta^2_S.
    \end{equation}
    If $\hat{\gamma}_s\geq \gamma_s$ for all $s\in(s_0,s_1]$,
    then,
    the inequality \eqref{eq:ICML2025:generaly_bound_on_estimator_error} can be rewritten as follows.
    \begin{align*}
        &\frac{\delta^2_S\Vert \Delta_s(S)\Vert_1}{2k}
        \leq\left(8+4\sigma\right)\frac{g_{d,k}}{s}\ln\frac{d}{\delta}+\\
        &\qquad\qquad\quad\frac{2(6.9+1.2\sigma)}{\sqrt{s}}\sqrt{\frac{d-1}{k-1}\ln\frac{d}{\delta}}+\nu_s+
        \frac{1}{s}\sqrt{3g_{d,k}\sum^s_{\tau=1}\Vert\Delta_{\tau-1}(S)\Vert^2_1\ln\frac{d}{\delta}}+
        \frac{\delta^2_S\Vert \Delta_s(S)\Vert_1}{8k},
    \end{align*}
    in which we use Lemma \ref{lemma:ICML25:Dantzig2005} and \eqref{eq:ICML2025:inequality_s_0}.
    Rearranging terms and substituting into the value of $s_0$ yields,
    \begin{align}
        &\Vert \Delta_s(S)\Vert_1\nonumber\\
        \leq&\frac{8k\nu_s}{3\delta^2_S}+\frac{a_1kg_{d,k}}{\delta^2_Ss}
        +\frac{a_2k}{\delta^2_S\sqrt{s}}\sqrt{\frac{d-1}{k-1}}
        +\frac{a_3k\sqrt{g_{d,k}}}{\delta^2_Ss}\sqrt{\sum^{s_0}_{\tau=1}\Vert\Delta_{\tau-1}(S)\Vert^2_1
        +\sum^{s}_{\tau=s_0+1}\Vert\Delta_{\tau-1}(S)\Vert^2_1}\nonumber\\
        \leq&\frac{8k\nu_s}{3\delta^2_S}+
        \frac{\delta^2_S\frac{a_1}{k}+48 a_3\sqrt{\ln\frac{d^2}{\delta}}}{\delta^4_S}
        \frac{k^2g_{d,k}}{s}
        +\frac{a_2k}{\delta^2_S\sqrt{s}}\sqrt{\frac{d-1}{k-1}}
        +\frac{a_3k\sqrt{g_{d,k}}}{s\delta^2_S}\sqrt{\sum^{s-1}_{\tau=s_0}\Vert\Delta_{\tau}(S)\Vert^2_1},
    \label{eq:ICML2025:refined_upper_bound:estimator_error}
    \end{align}
    in which $a_1$, $a_2$ and $a_3$ are defined in Lemma \ref{lemma:estimator_error:DS-OSLRC}.
    For simplicity, let
    \begin{align*}
        a'_4=\delta^2_S\frac{a_1}{k}+48a_3\sqrt{\ln\frac{d^2}{\delta}}.
    \end{align*}
    Next we will use the induction method to prove the convergence rate of $\Vert \Delta_s(S)\Vert_1$.
    To be specific,
    we will prove that,
    w.p. at least $1-s_1\left(6+\log_{1.5}\frac{d+k}{k-2}\right)\delta$,
    \begin{equation}
    \label{eq:ICML2025:second_type_of_estimator_error}
    \begin{split}
        \forall s_0 < s\leq s_1,\quad \Vert \Delta_s(S)\Vert_1
        \leq& \frac{a_4}{1-\frac{2\sqrt{3}}{9}} \frac{k^2g_{d,k}}{s\delta^4_S},\\
        a_4=&\delta^2_S\frac{a_1}{k}
        +24a_2\sqrt{\frac{k-2}{d-2}\ln\frac{d^2}{\delta}}
        +4a_3\left(24\sqrt{\ln\frac{d^2}{\delta}}+\frac{\delta^2_S}{k\sqrt{g_{d,k}}}\right).
    \end{split}
    \end{equation}
    We first analyze the case $s=s_0+1$.
    Recalling that $\mu_1=\frac{9}{9-2\sqrt{3}}$ and
    \begin{align*}
        \nu_{s_0+1}
        =&\frac{24kg_{d,k}}{(s_0+1)\delta^2_S}\sqrt{12\ln\left(\frac{d}{\delta}\right)\ln\frac{d^2}{\delta}}
        +\frac{2\sqrt{3g_{d,k}\ln\frac{d}{\delta}}}{s_0+1}
        \geq \frac{1}{s_0+1}\sqrt{3g_{d,k}\sum^{s_0+1}_{\tau=1}\Vert\Delta_{\tau-1}(S)\Vert^2_1\ln\frac{d}{\delta}},\\
        \nu_s
        =&\frac{24kg_{d,k}}{s\delta^2_S}\sqrt{12\ln\left(\frac{d}{\delta}\right)\ln\frac{d^2}{\delta}}
        +\frac{2\sqrt{3g_{d,k}\ln\frac{d}{\delta}}}{s}
        +\frac{\mu_1 a_4\sqrt{3g_{d,k}}}{\delta^4_S\cdot s}
        \sqrt{\sum^{s-1}_{\tau=s_0+1}\frac{k^4g^2_{d,k}}{\tau^2}\ln\frac{d}{\delta}},\\
        &\forall s_0+1 < s\leq s_1.
    \end{align*}
    \textbf{We still have $\hat{\gamma}_{s_0+1}\geq \gamma_{s_0+1}$.
    By Lemma \ref{lemma:ICML25:optimal_solution_constraint},
    with a high probability,
    $\mathrm{DS}(\hat{\gamma}_{s_0+1})$ has a feasible solution at least}.
    In this way,
    it is proper to derive \eqref{eq:ICML2025:refined_upper_bound:estimator_error} for any $s\leq s_0+1$.
    \begin{align*}
        &\Vert \Delta_{s_0+1}(S)\Vert_1\\
        \leq& \frac{8k\nu_{s_0+1}}{3\delta^2_S}+a'_4\frac{k^2g_{d,k}}{(s_0+1)\delta^4_S}
        +\frac{a_2k}{\delta^2_S}\sqrt{\frac{d-1}{(s_0+1)(k-1)}}
        +\frac{a_3k}{\delta^2_S}\frac{\sqrt{g_{d,k}}}{(s_0+1)}\sqrt{\Vert\Delta_{s_0}(S)\Vert^2_1}\\
        \leq& \left(64\sqrt{12\ln\left(\frac{d}{\delta}\right)\ln\frac{d^2}{\delta}}
        +a'_4+a_2\sqrt{24^2\frac{k-2}{d-2}\ln\frac{d^2}{\delta}}
        +\frac{16\sqrt{3\ln\frac{d}{\delta}}\delta^2_S}{3k\sqrt{g_{d,k}}}
        +\frac{2 a_3\delta^2_S}{k\sqrt{g_{d,k}}}\right)\frac{k^2g_{d,k}}{(s_0+1)\delta^4_S}\\
        \leq&a_4\frac{k^2g_{d,k}}{(s_0+1)\delta^4_S},
    \end{align*}
    in which
    \begin{align*}
        a_4
        =&\delta^2_S\frac{a_1}{k}
        +24a_2\sqrt{\frac{k-2}{d-2}\ln\frac{d^2}{\delta}}
        +4a_3\left(24\sqrt{\ln\frac{d^2}{\delta}}+\frac{\delta^2_S}{k\sqrt{g_{d,k}}}\right).
    \end{align*}
    Thus \eqref{eq:ICML2025:second_type_of_estimator_error} holds for $s=s_0+1$.

    Now assuming that
    \eqref{eq:ICML2025:second_type_of_estimator_error} holds for any $s=r\in [s_0+1,s_1)$,
    \textbf{implying that $\hat{\gamma}_{r+1}\geq \gamma_{r+1}$ and $\mathrm{DS}(\hat{\gamma}_{r+1})$ has a feasible solution at least}.
    In this way,
    it is proper to derive \eqref{eq:ICML2025:refined_upper_bound:estimator_error} for any $s\leq r+1$,
    and more importantly,
    it is possible to analyze the upper bound on $\Vert \Delta_{r+1}(S)\Vert_1$.
    \begin{align*}
        &\Vert \Delta_{r+1}(S)\Vert_1\\
        \leq& \frac{8k\nu_{r+1}}{3\delta^2_S}+\frac{a'_4k^2g_{d,k}}{(r+1)\delta^4_S}
        +\frac{a_2k}{\delta^2_S}\sqrt{\frac{d-1}{(r+1)(k-1)}}
        +\frac{a_3k}{\delta^2_S}\frac{\sqrt{g_{d,k}}}{r+1}
        \sqrt{\Vert\Delta_{s_0}(S)\Vert^2_1+\sum^r_{\tau=s_0+1}\Vert\Delta_{\tau}(S)\Vert^2_1}\\
        \leq&a_4\frac{k^2g_{d,k}}{(r+1)\delta^4_S}
        +\frac{8k\mu_1 a_4\sqrt{3g_{d,k}\ln\frac{d}{\delta}}}{3\delta^6_S(r+1)}
        \sqrt{\sum^{r}_{\tau=s_0+1}\frac{k^4g^2_{d,k}}{\tau^2}}
        +\frac{a_3k}{\delta^6_S}\frac{\sqrt{g_{d,k}}}{r+1}\cdot
        \mu_1 a_4\sqrt{\sum^r_{\tau=s_0+1}\frac{k^4g^2_{d,k}}{\tau^2}}\\
        \leq&a_4\frac{k^2g_{d,k}}{(r+1)\delta^4_S}
        +2a_3\frac{k}{\delta^6_S}\frac{\sqrt{g_{d,k}}}{r+1}\cdot
        \mu_1 a_4\sqrt{k^4g^2_{d,k}\left(\frac{1}{s_0}-\frac{1}{r}\right)}
        \qquad\qquad\qquad\qquad\qquad (\mathrm{By}~\mathrm{Lemma}~\ref{lemma:ICML2025:summation:s^-2})\\
        \leq&a_4\frac{k^2g_{d,k}}{(r+1)\delta^4_S}
        +2a_3\frac{k}{\delta^6_S}\frac{\sqrt{g_{d,k}}}{r+1}\cdot
        \mu_1 a_4\sqrt{\frac{k^4g^2_{d,k}}{\frac{24^2 g_{d,k}k^2}{\delta^4_S}\ln\frac{d^2}{\delta}}}\\
        =&\mu_1 a_4\frac{k^2g_{d,k}}{\delta^4_S(r+1)}.
    \end{align*}
    Thus \eqref{eq:ICML2025:second_type_of_estimator_error} also holds for $s=r+1$.
    We conclude that \eqref{eq:ICML2025:second_type_of_estimator_error} holds for all $s\in(s_0,s_1]$.

\subsection{Case 3: $s_1<s<\lfloor\sqrt{T}\rfloor$}

    Let $s'=\lfloor\sqrt{T}\rfloor$.
    Recalling that $\mu_2=\left(1-\frac{\sqrt{6}}{9\sqrt{\frac{d-2}{k-2}\ln\frac{d^2}{\delta}}}\right)^{-1}$
    and
    \begin{align*}
        \nu_{s_1+1}
        =&\frac{\sqrt{3g_{d,k}}}{s_1+1}\left(\frac{48k}{\delta^2_S}
        \sqrt{g_{d,k}\ln\left(\frac{d}{\delta}\right)\ln\frac{d^2}{\delta}}+
        \sqrt{4\ln\frac{d}{\delta}}
        +\frac{\mu_1a_4k^2g_{d,k}}{\delta^4_S}\sqrt{\sum^{s_1}_{\tau=s_0+1}\frac{1}{\tau^2}\ln\frac{d}{\delta}}\right),\\
        \nu_s
        =&\frac{\sqrt{3g_{d,k}}}{s}\left(\frac{48k}{\delta^2_S}
        \sqrt{g_{d,k}\ln\left(\frac{d}{\delta}\right)\ln\frac{d^2}{\delta}}+
        \sqrt{4\ln\frac{d}{\delta}}
        +\frac{\mu_1a_4k^2g_{d,k}}{\delta^4_S}\sqrt{\sum^{s_1}_{\tau=s_0+1}\frac{1}{\tau^2}\ln\frac{d}{\delta}}
        +\right.\\
        &\left.\frac{\mu_2a_5}{\delta^2_S}\sqrt{\sum^{s-1}_{\tau=s_1+1}\frac{k^2(d-1)}{\tau(k-1)}\ln\frac{d}{\delta}}\right),
        \quad\forall s>s_1+1.
    \end{align*}
    As \eqref{eq:ICML2025:second_type_of_estimator_error} holds,
    it can be verified that $\hat{\gamma}_{s_1+1}\geq \gamma_{s_1+1}$.
    \textbf{By Lemma \ref{lemma:ICML25:optimal_solution_constraint},
    with a high probability,
    $\mathrm{DS}(\hat{\gamma}_{s_1+1})$ has a feasible solution at least}.
    In this way,
    it is proper to derive \eqref{eq:ICML2025:refined_upper_bound:estimator_error} for any $s\leq s_1+1$.
    Similarly,
    we also use the induction method to prove the convergence rate of $\Vert \Delta_s(S)\Vert_1$
    and $\mathrm{DS}(\hat{\gamma}_s)$ has a feasible solution for all $s>s_1$.
    Specifically,
    by \eqref{eq:ICML2025:refined_upper_bound:estimator_error},
    we will prove that,
    with probability at least $1-\sqrt{T}\left(6+\log_{1.5}\frac{d+4}{k-2}\right)\delta$,
    \begin{equation}
    \label{eq:ICML2025:third_type_of_estimator_error}
    \begin{split}
        &\forall s>s_1,\quad \Vert \Delta_s(S)\Vert_1
        \leq \frac{\mu_2a_5k}{\delta^2_S}\sqrt{\frac{d-1}{s(k-1)}},\\
        &a_5=
        \mu_1\left(\delta^2_S\frac{\frac{8}{9}+\frac{4\sigma}{9}}{k}
            +\frac{32\sqrt{3}}{3}
            +\frac{4\sqrt{3}\delta^2_S}{9k\sqrt{g_{d,k}\ln{\frac{d^2}{\delta}}}}\right)
            +a_2+(\mu_1-1)a_2\sqrt{\frac{k-2}{(d-2)\ln\frac{d^2}{\delta}}}.
    \end{split}
    \end{equation}
    We first verify the case $s=s_1+1$.
    \begin{align*}
        &\Vert \Delta_{s_1+1}(S)\Vert_1\\
        \leq&\frac{8k\nu_{s_1+1}}{3\delta^2_S}+\frac{a'_4k^2g_{d,k}}{(s_1+1)\delta^4_S}
        +\frac{a_2k}{\delta^2_S\sqrt{s_1+1}}\sqrt{\frac{d-1}{k-1}}
        +\frac{a_3k}{\delta^2_S}\frac{\sqrt{g_{d,k}}}{s_1+1}
        \sqrt{\Vert\Delta_{s_0}(S)\Vert^2_1+\sum^{s_1}_{s=s_0+1}\Vert\Delta_{s}(S)\Vert^2_1}\\
        \leq&\underbrace{\frac{8k\nu_{s_1+1}}{3\delta^2_S}}_{\Xi_1}
        +\underbrace{\frac{a'_4k^2g_{d,k}}{(s_1+1)\delta^4_S}
        +\frac{a_2k}{\delta^2_S\sqrt{s_1+1}}\sqrt{\frac{d-1}{k-1}}}_{\Xi_2}
        +\underbrace{\frac{2a_3k}{\delta^2_S}\frac{\sqrt{g_{d,k}}}{s_1+1}}_{\Xi_3}
        +\underbrace{\frac{a_3k}{\delta^2_S}\frac{\sqrt{g_{d,k}}}{s_1+1}
        \sqrt{\sum^{s_1}_{s=s_0+1}\Vert\Delta_{s}(S)\Vert^2_1}}_{\Xi_4}.
    \end{align*}
    Next we separately analyze the four terms on the right-hand side of the inequality.

\noindent First, we analyze $\Xi_1$.
    Substituting into the value of $\nu_{s_1+1}$ and $s_0$,
    we obtain
    \begin{align*}
        &\Xi_1\\
        =&\frac{8k}{3\delta^2_S}\frac{\sqrt{3g_{d,k}}}{s_1+1}
        \left(\frac{48k}{\delta^2_S}\sqrt{g_{d,k}\ln\left(\frac{d}{\delta}\right)\ln\frac{d^2}{\delta}}+
        \sqrt{4\ln\frac{d}{\delta}}
        +\frac{\mu_1a_4k^2g_{d,k}}{\delta^4_S}\sqrt{\sum^{s_1-1}_{\tau=s_0+1}\frac{1}{\tau^2}\ln\frac{d}{\delta}}\right)\\
        \leq&\frac{1}{\delta^2_S\sqrt{s_1+1}}
        \left(a_3k\frac{g_{d,k}}{\sqrt{s_1+1}}\frac{48k}{\delta^2_S}\sqrt{\ln\frac{d^2}{\delta}}
        +2a_3k\frac{\sqrt{g_{d,k}}}{\sqrt{s_1+1}}
        +a_3k\frac{\sqrt{g_{d,k}}}{\sqrt{s_1+1}}\frac{\mu_1a_4k^2g_{d,k}}{\delta^4_S}\sqrt{\frac{1}{s_0}}\right)\\
        =&\frac{1}{\delta^2_S\sqrt{s_1+1}}
        \left(\frac{2a_3\sqrt{d-1}k}{\sqrt{(k-1)\ln\frac{d}{\delta}}}
        +\frac{a_3\delta^2_S}{12k\sqrt{g_{d,k}\ln(\frac{d}{\delta})\ln\frac{d^2}{\delta}}}k\sqrt{\frac{d-1}{k-1}}
        +\frac{a_3\mu_1a_4}{24^2\sqrt{\ln\frac{d}{\delta}}\ln\frac{d^2}{\delta}}
        \frac{k\sqrt{d-1}}{\sqrt{k-1}}\right)\\
        =&\frac{k}{\delta^2_S}\frac{\sqrt{d-1}}{\sqrt{(k-1)(s_1+1)}}
        \left(\frac{2a_3}{\sqrt{\ln\frac{d}{\delta}}}
        +\frac{a_3\delta^2_S}{12k\sqrt{g_{d,k}\ln(\frac{d}{\delta})\ln\frac{d^2}{\delta}}}
        +\frac{a_3\mu_1a_4}{24^2\sqrt{\ln\frac{d}{\delta}}\ln\frac{d^2}{\delta}}\right),
    \end{align*}
    in which $a_3=\frac{8\sqrt{3}}{3}\sqrt{\ln\frac{d}{\delta}}$
    and the first inequality comes from Lemma \ref{lemma:ICML2025:summation:s^-2}.

\noindent Next we analyze $\Xi_2$.
    Substituting into the value of $s_1$, $a_1$, $a_3$ and $a'_4$ yields
    \begin{align*}
        \Xi_2\leq&\left(a'_4\frac{kg_{d,k}}{\sqrt{\frac{24^2g_{d,k}k^2}{\delta^4_S}
        \frac{d-2}{k-2}\ln\left(\frac{d}{\delta}\right)\ln\frac{d^2}{\delta}}\delta^2_S}\frac{\sqrt{k-1}}{\sqrt{d-1}}
        +a_2\right)\frac{k}{\delta^2_S}\sqrt{\frac{d-1}{(s_1+1)(k-1)}}\\
        \leq&\left(\frac{a'_4}{24\sqrt{\ln\left(\frac{d}{\delta}\right)\ln\frac{d^2}{\delta}}}+a_2\right)
        \frac{k}{\delta^2_S}\sqrt{\frac{d-1}{(s_1+1)(k-1)}}\\
        \leq&\left(\frac{\delta^2_S\frac{1}{k}(\frac{64}{3}+\frac{32}{3}\sigma)\ln\frac{d}{\delta}
        +48a_3\sqrt{\ln\frac{d^2}{\delta}}}
        {24\sqrt{\ln\left(\frac{d}{\delta}\right)\ln\frac{d^2}{\delta}}}+a_2\right)
        \frac{k}{\delta^2_S}\sqrt{\frac{d-1}{(s_1+1)(k-1)}}\\
        \leq&\left(\delta^2_S\frac{\frac{8}{9}+\frac{4\sigma}{9}}{k}+\frac{2a_3}{\sqrt{\ln\frac{d}{\delta}}}+a_2\right)
        \frac{k}{\delta^2_S}\sqrt{\frac{d-1}{(s_1+1)(k-1)}}.
    \end{align*}
    For $\Xi_3$,
    we can obtain
    $$
        \Xi_3
        =\frac{2a_3k\sqrt{g_{d,k}}}{\sqrt{\frac{24^2 g_{d,k}k^2}{\delta^4_S}\frac{d-2}{k-2}
        \ln\left(\frac{d}{\delta}\right)\ln\frac{d^2}{\delta}}}
        \frac{1}{\delta^2_S\sqrt{s_1+1}}
        =\frac{a_3\delta^2_S}{12k\sqrt{g_{d,k}\ln\left(\frac{d}{\delta}\right)\ln{\frac{d^2}{\delta}}}}
        \frac{k\sqrt{d-1}}{\delta^2_S\sqrt{(s_1+1)(k-1)}}.
    $$
    Finally,
    we analyze $\Xi_4$.
    By Lemma \ref{lemma:ICML2025:summation:s^-2}
    and \eqref{eq:ICML2025:second_type_of_estimator_error},
    we can obtain
    \begin{align*}
        \frac{a_3k}{\delta^2_S}\frac{\sqrt{g_{d,k}}}{s_1+1}
        \sqrt{\sum^{s_1}_{s=s_0+1}\Vert\Delta_s(S)\Vert^2_1}
        \leq&
        \frac{a_3}{24k\sqrt{g_{d,k}\ln\left(\frac{d}{\delta}\right)
        \ln{\frac{d^2}{\delta}}}}\frac{k\sqrt{d-1}}{\sqrt{(s_1+1)(k-1)}}
        \frac{\mu_1a_4}{\delta^4_S}\cdot \sqrt{\frac{k^4g^2_{d,k}}{s_0}}\\
        =&\frac{a_3\mu_1a_4}{24^2\sqrt{\ln\frac{d}{\delta}}\ln\frac{d^2}{\delta}}\cdot
        \frac{k\sqrt{d-1}}{\delta^2_S\sqrt{(s_1+1)(k-1)}}.
    \end{align*}
    Combining the upper bounds on $\Xi_1,\Xi_2,\Xi_3$ and $\Xi_4$,
    we obtain
    \begin{align*}
        &\Vert \Delta_{s_1+1}(S)\Vert_1\\
        \leq& \left(\frac{4a_3}{\sqrt{\ln\frac{d}{\delta}}}
            +\frac{2a_3\delta^2_S}{12k\sqrt{g_{d,k}\ln(\frac{d}{\delta})\ln\frac{d^2}{\delta}}}
            +\frac{2a_3\mu_1a_4}{24^2\sqrt{\ln\frac{d}{\delta}}\ln\frac{d^2}{\delta}}
            +\delta^2_S\frac{\frac{8}{9}+\frac{4\sigma}{9}}{k}+a_2\right)
            \frac{k\sqrt{d-1}}{\delta^2_S\sqrt{(s_1+1)(k-1)}}\\
        \leq&\left(\delta^2_S\frac{\frac{8}{9}+\frac{4\sigma}{9}}{k}
            +\frac{32\sqrt{3}}{3}
            +\frac{4\sqrt{3}\delta^2_S}{9k\sqrt{g_{d,k}\ln{\frac{d^2}{\delta}}}}
            +a_2\right)
            \cdot \frac{k\sqrt{d-1}}{\delta^2_S\sqrt{(s_1+1)(k-1)}}+\\
            &(\mu_1-1)\left(\delta^2_S\frac{\frac{8}{9}+\frac{4\sigma}{9}}{k}
        +a_2\sqrt{\frac{k-2}{(d-2)\ln\frac{d^2}{\delta}}}
        +\frac{4a_3}{\sqrt{\ln\frac{d^2}{\delta}}}+\frac{4a_3\delta^2_S}{24k\sqrt{g_{d,k}}\ln\frac{d^2}{\delta}}\right)
             \frac{k\sqrt{d-1}}{\delta^2_S\sqrt{(s_1+1)(k-1)}}\\
        \leq&a_5\frac{k\sqrt{d-1}}{\delta^2_S\sqrt{(s_1+1)(k-1)}}.
    \end{align*}
    Thus \eqref{eq:ICML2025:third_type_of_estimator_error} holds for $s=s_1+1$.

    Now assuming that \eqref{eq:ICML2025:third_type_of_estimator_error} holds for any $s=r\geq s_1+1$,
    \textbf{implying that $\hat{\gamma}_{r+1}\geq \gamma_{r+1}$ and $\mathrm{DS}(\hat{\gamma}_{r+1})$ has a feasible solution}.
    In this way,
    it is proper to derive \eqref{eq:ICML2025:refined_upper_bound:estimator_error} for any $s\leq r+1$,
    By the second inequality in Lemma \ref{lemma:ICML2025:summation:s^-2},
    we obtain
    \begin{align*}
        &\Vert \Delta_{r+1}(S)\Vert_1\\
        \leq& \frac{8k\nu_{r+1}}{3\delta^2_S}+\frac{a'_4k^2g_{d,k}}{(r+1)\delta^4_S}
        +\frac{a_2k}{\delta^2_S\sqrt{r+1}}\sqrt{\frac{d-1}{k-1}}
        +\frac{a_3}{\delta^2_S}\frac{k\sqrt{g_{d,k}}}{r+1}
        \sqrt{\sum^{s_1}_{s=s_0}\Vert\Delta_s(S)\Vert^2_1
        +\sum^{r}_{\tau=s_1+1}\Vert\Delta_{\tau}(S)\Vert^2_1}\\
        \leq&\frac{a_5k}{\delta^2_S}\sqrt{\frac{d-1}{(k-1)(r+1)}}
        +2a_3\frac{k}{\delta^2_S}\frac{\sqrt{g_{d,k}}}{\sqrt{r+1}}\cdot\frac{\mu_2 a_5}{\delta^2_S}
        \sqrt{\sum^{r}_{s=s_1+1}\frac{k^2(d-1)}{(r+1)s(k-1)}}\\
        \leq&\frac{a_5k}{\delta^2_S}\sqrt{\frac{d-1}{(k-1)(r+1)}}+
        2a_3\frac{k\sqrt{g_{d,k}}}{\delta^2_S\sqrt{2s_1}}\cdot
        \frac{\mu_2 a_5}{\delta^2_S}\sqrt{\frac{k^2(d-1)}{(r+1)(k-1)}}\\
        =&\frac{a_5k}{\delta^2_S}\sqrt{\frac{d-1}{(k-1)(r+1)}}+
        \frac{\sqrt{6}}{9\sqrt{\frac{d-2}{k-2}\ln\frac{d^2}{\delta}}}\cdot
        \frac{\mu_2 a_5}{\delta^2_S}\sqrt{\frac{k^2(d-1)}{2(r+1)(k-1)}}\\
        =&\frac{\mu_2 a_5}{\delta^2_S}\sqrt{\frac{k(d-1)}{(r+1)(k-1)}}.
    \end{align*}
    Thus \eqref{eq:ICML2025:third_type_of_estimator_error} also holds for $s=r+1$.
    Therefore, \eqref{eq:ICML2025:third_type_of_estimator_error} holds for all $s>s_1$,
    concluding the proof.
\end{proof}

\section{Proof of Theorem \ref{thm:ICML2025:regret_bound_DS-OSLRC}}

\begin{Mylemma}
\label{lemma:ICML2025:convergence_S_s_to_S}
    Let
    $$
        s_2=4\cdot\frac{(\mu_2a_5)^2}{\delta^4_S}\cdot\frac{k^2(d-1)}{\min_{i\in S}\vert w^\ast_i\vert^2\cdot(k-1)}.
    $$
    For any $s> s_2$,
    with probability at least $1-\sqrt{T}\left(6+\log_{1.5}\frac{d+k}{k-2}\right)\delta$, it must be $S_s=S$.
\end{Mylemma}
\begin{proof}[of Lemma \ref{lemma:ICML2025:convergence_S_s_to_S}]
    Assuming that $S_s\neq S$, then there is a $j\in S_s\setminus S$ and $i\in S\setminus S_s$ such that
    $$
        \Vert \hat{\mathbf{w}}_s-\mathbf{w}^\ast\Vert_1
        \geq
        \vert w^\ast_i - \hat{w}_{s,i}\vert + \vert\hat{w}_{s,j}\vert
        \geq \vert w^\ast_i\vert
        \geq \min_{r\in S}\vert w^\ast_r\vert,
    $$
    in which $\vert \hat{w}_{s,i}\vert \leq \vert \hat{w}_{s,j}\vert$.
    For any $s> s_2$, we have
    $
        2\frac{\mu_2a_5k}{\delta^2_S}\sqrt{\frac{d-1}{s(k-1)}}< \min_{i\in S}\vert w^\ast_i\vert,
    $
    which implies that, with probability at least $1-\sqrt{T}\left(6+\log_{1.5}\frac{d+k}{k-2}\right)\delta$,
    $$
        \Vert \hat{\mathbf{w}}_s -\mathbf{w}^\ast\Vert_1
        \leq 2\Vert \hat{\mathbf{w}}_s(S) -\mathbf{w}^\ast\Vert_1
        < \min_{i\in S}\vert w^\ast_i\vert
        \leq\Vert \hat{\mathbf{w}}_s-\mathbf{w}^\ast\Vert_1.
    $$
    There is a contradiction.
    We conclude the proof.
\end{proof}

\begin{proof}[of Theorem \ref{thm:ICML2025:regret_bound_DS-OSLRC}]
    Let $s_T=\lfloor\sqrt{T}\rfloor$ and $T>(s_2+1)^2 >(s_1+1)^2$.
    For simplicity,
    we will alternately use the notation $\hat{y}_t=\langle\bar{\mathbf{w}}_t(S_s),{\mathbf{x}}_t\rangle$.
    The regret of DS-OSLRC can be decomposed as follows,
    \begin{align*}
        \mathrm{Reg}(\mathbf{w}^\ast)
        &=\underbrace{\sum^{s_T}_{s=1}\left[\ell(\hat{y}_{s^2},y_{s^2})
        -\ell(\langle\mathbf{w}^\ast,{\mathbf{x}}_{s^2}\rangle,y_{s^2})\right]}_{:=\mathrm{Reg}_0}
        +\sum^{s_T}_{s=1}\sum^{(s+1)^2-1}_{t=s^2+1}
        \left[\ell(\hat{y}_t,y_t)-\ell(\langle\mathbf{w}^\ast,{\mathbf{x}}_t\rangle,y_t)\right]\\
        =&\mathrm{Reg}_0
        +\left(\sum^{s_0}_{s=1}+\sum^{s_1}_{s=s_0+1}+\sum^{s_2}_{s=s_1+1}+\sum^{s_T}_{s=s_2+1}\right)
        \sum_{t\in\mathcal{T}_s}
        \left[\ell(\langle\bar{\mathbf{w}}_t(S_s),{\mathbf{x}}_t\rangle,y_t)
        -\ell(\langle\mathbf{w}^\ast,{\mathbf{x}}_t\rangle,y_t)\right],
    \end{align*}
    in which $(s_T+1)^2-1:=T$.
    For $\mathrm{Reg}_0$, unfolding the square loss function
    and substituting into the definition of $y_{s^2}$ gives
    \begin{align*}
        \sum^{s_T}_{s=1}\ell(\hat{y}_{s^2},y_{s^2})
        -&\ell(\langle\mathbf{w}^\ast,{\mathbf{x}}_{s^2}\rangle,y_{s^2})
        =\sum^{s_T}_{s=1}(\langle \hat{\mathbf{w}}_{s-1}(B_s)-\mathbf{w}^\ast,{\mathbf{x}}_{s^2}\rangle)^2
       -2\eta_{s^2}\langle \hat{\mathbf{w}}_{s-1}(B_s)-\mathbf{w}^\ast,{\mathbf{x}}_{s^2}\rangle.
    \end{align*}
    Next we separately analyze the regret in the intervals
    $[2,(s_0+1)^2-1]$, $[(s_0+1)^2+1,(s_1+1)^2-1]$, $[(s_1+1)^2+1,(s_2+1)^2-1]$,
    and $[(s_2+1)^2+1,T]$.

    We briefly explain the novelty of our analysis.
    As our algorithm updates $\bar{\mathbf{w}}_t(S_s)$ by ONS,
    the regret in each $\mathcal{T}_s$ can be decomposed into two components.
    The first one is the regret of our algorithm w.r.t. any $\mathbf{w}_s\in\mathcal{W}_s$.
    Note that we can arbitrary choose $\mathbf{w}_s\in\mathcal{W}_s$.
    Specifically,
    we will choose $\mathbf{w}_s$
    by Lemma \ref{lemma:COLT2025:approximation_error_w_s},
    i.e., $\mathbf{w}_s=\mathbf{w}^\ast(S_s\cap S)$.
    ONS ensures that our algorithm converges rapidly to $\mathbf{w}_s$.
    The second part of the regret is
    the difference of the cumulative losses of $\mathbf{w}_s$ and $\mathbf{w}^\ast$.
    As $\Vert \mathbf{w}_s -\mathbf{w}^\ast\Vert_1
    \leq \Vert \hat{\mathbf{w}}_s-\mathbf{w}^\ast\Vert_1$,
    it is easy to analyze the second part
    by the convergence rate in Lemma \ref{lemma:estimator_error:DS-OSLRC}.
    Compared with the vanilla selection of $\mathbf{w}_s$, i.e.,
    $\mathbf{w}_s=\hat{\mathbf{w}}_s(S_s)$,
    our choice can significantly reduce the constant factor
    (please see the discussion below Lemma \ref{lemma:COLT2025:approximation_error_w_s}).
    With probability at least $1-\delta$,
    $\langle\mathbf{w}^\ast,\mathbf{x}_t\rangle\leq Y_{\delta}$ for a fixed $t$.
    Then with probability at least $1-((s_0+1)^2-1)\delta$,
    \begin{align*}
        &{\mathrm{Reg}_{[1:s_0]}(\mathbf{w}^\ast)}\\
        =&\sum^{s_0}_{s=1}\sum_{t\in\mathcal{T}_s}
        \left[\ell(\hat{y}_t,y_t)-\ell(\langle \mathbf{w}_s,{\mathbf{x}}_t\rangle,y_t)
        +\ell(\langle \mathbf{w}_s,{\mathbf{x}}_t\rangle,y_t)-\ell(\langle\mathbf{w}^\ast,{\mathbf{x}}_t\rangle,y_t)\right]\\
        =&\sum^{s_0}_{s=1}\sum_{t\in\mathcal{T}_s}
        2(\hat{y}_t-y_t)(\hat{y}_t-\langle \mathbf{w}_s,{\mathbf{x}}_t\rangle)-
        (\hat{y}_t-\langle \mathbf{w}_s,{\mathbf{x}}_t\rangle)^2
        +(\langle\mathbf{w}_s-\mathbf{w}^\ast,{\mathbf{x}}_t\rangle)^2
        -2\eta_t\langle \mathbf{w}_s-\mathbf{w}^\ast,{\mathbf{x}}_t\rangle\\
        =&\sum^{s_0}_{s=1}\sum_{t\in\mathcal{T}_s}
        \langle \mathbf{g}_t,\bar{\mathbf{w}}_t(S_s)-\mathbf{w}_s\rangle-
        \frac{(\langle \mathbf{g}_t,\bar{\mathbf{w}}_t(S_s)-\mathbf{w}_s\rangle)^2}
        {4(\hat{y}_t-y_t)^2}+(\langle\mathbf{w}_s-\mathbf{w}^\ast,{\mathbf{x}}_t\rangle)^2
        -2\eta_t\langle \mathbf{w}_s-\mathbf{w}^\ast,{\mathbf{x}}_t\rangle\\
        \leq&\sum^{s_0}_{s=1}\sum_{t\in\mathcal{T}_s}
        \langle \mathbf{g}_t,\bar{\mathbf{w}}_t(S_s)-\mathbf{w}_s\rangle-
        \frac{(\langle \mathbf{g}_t,\bar{\mathbf{w}}_t(S_s)-\mathbf{w}_s\rangle)^2}{4(1+Y_{\delta})^2}
        +(\langle\mathbf{w}_s-\mathbf{w}^\ast,{\mathbf{x}}_t\rangle)^2
        -2\eta_t\langle \mathbf{w}_s-\mathbf{w}^\ast,{\mathbf{x}}_t\rangle.
    \end{align*}
    By the regret analysis of ONS \citep{Hazan2007Logarithmic},
    we have
    \begin{align*}
        \left\Vert \bar{\mathbf{w}}_{t+1}(S_s)-\mathbf{w}_s\right\Vert^2_{\mathbf{A}_t}
        \leq& \left\Vert \bar{\mathbf{w}}_t(S_s)-\mathbf{A}^{-1}_t\mathbf{g}_t-\mathbf{w}_s\right\Vert^2_{\mathbf{A}_t}\\
        =&\left\Vert \bar{\mathbf{w}}_t(S_s)-\mathbf{w}_s\right\Vert^2_{\mathbf{A}_t}
        -2\left\langle\bar{\mathbf{w}}_t(S_s)-\mathbf{w}_s,\mathbf{A}^{-1}_t\mathbf{g}_t\right\rangle_{\mathbf{A}_t}
        +\left\Vert \mathbf{A}^{-1}_t\mathbf{g}_t\right\Vert^2_{\mathbf{A}_t}.
    \end{align*}
    Recalling that $\rho=\frac{1}{2(1+Y_{\delta})^2}$.
    Rearranging terms yields
    \begin{align*}
        &\sum_{t\in\mathcal{T}_s}
        \langle \mathbf{g}_t,\bar{\mathbf{w}}_t(S_s)-\mathbf{w}_s\rangle-
        \frac{(\langle \mathbf{g}_t,\bar{\mathbf{w}}_t(S_s)-\mathbf{w}_s\rangle)^2}{4(1+Y_{\delta})^2}\\
        \leq&\sum^{(s+1)^2-1}_{t=s^2+1}
        \frac{\Vert \bar{\mathbf{w}}_t(S_s)-\mathbf{w}_s\Vert^2_{\mathbf{A}_t}
        -\Vert \bar{\mathbf{w}}_{t+1}(S_s)-\mathbf{w}_s\Vert^2_{\mathbf{A}_t}}{2}
        +\frac{\Vert \mathbf{A}^{-1}_t\mathbf{g}_t\Vert^2_{\mathbf{A}_t}}{2}
        -\frac{\Vert \bar{\mathbf{w}}_t(S_s)-\mathbf{w}_s\Vert^2_{\rho\cdot\mathbf{g}_t\mathbf{g}^\top_t}}{2}\\
        =&\sum^{(s+1)^2-1}_{t=s^2+1}
        \frac{\Vert \bar{\mathbf{w}}_t(S_s)-\mathbf{w}_s\Vert^2_{\mathbf{A}_{t-1}}
        -\Vert \bar{\mathbf{w}}_{t+1}(S_s)-\mathbf{w}_s\Vert^2_{\mathbf{A}_t}}{2}
        +\frac{\Vert \mathbf{A}^{-1}_t\mathbf{g}_t\Vert^2_{\mathbf{A}_t}}{2}\\
        =&
        \underbrace{
        \frac{\Vert \bar{\mathbf{w}}_{s^2+1}(S_s)-\mathbf{w}_s\Vert^2_{\mathbf{A}_{s^2}}}{2}
        -\frac{\Vert \bar{\mathbf{w}}_{(s+1)^2}(S_s)-\mathbf{w}_s\Vert^2_{\mathbf{A}_{(s+1)^2-1}}}{2}
        +\sum^{(s+1)^2-1}_{t=s^2+1}\frac{\mathbf{g}^\top_t\mathbf{A}^{-1}_t\mathbf{g}_t}{2}}_{:=\Xi_s},
    \end{align*}
    where
    $\mathbf{A}_t=\mathbf{A}_{t-1}+\rho\cdot \mathbf{g}_{t}\mathbf{g}^\top_t$.
    We have
    $$
        {\mathrm{Reg}_{[1:s_0]}(\mathbf{w}^\ast)}\leq \sum^{s_0}_{s=1}\Xi_s
        +\sum^{s_0}_{s=1}\left[\sum^{(s+1)^2-1}_{t=s^2+1}
        (\langle\mathbf{w}_s-\mathbf{w}^\ast,{\mathbf{x}}_t\rangle)^2
        -2\eta_t\langle \mathbf{w}_s-\mathbf{w}^\ast,{\mathbf{x}}_t\rangle\right].
    $$
    Similarly, the regret in the intervals $[(s_0+1)^2+1,(s_1+1)^2-1]$, $[(s_1+1)^2+1,(s_2+1)^2-1]$,
    and $[(s_2+1)^2+1,T]$ can be decomposed as follows,
    \begin{align*}
        {\mathrm{Reg}_{[s_0:s_1]}(\mathbf{w}^\ast)}
        :=&\sum^{s_1}_{s=s_0+1}\sum^{(s+1)^2-1}_{t=s^2+1}
        \left[\ell(\hat{y}_t,y_t)-\ell(\langle \mathbf{w}_s,{\mathbf{x}}_t\rangle,y_t)
        +\ell(\langle\mathbf{w}_s,{\mathbf{x}}_t\rangle,y_t)-\ell(\langle\mathbf{w}^\ast,{\mathbf{x}}_t\rangle,y_t)\right]\\
        =&\sum^{s_1}_{s=s_0+1}\Xi_s+\sum^{s_1}_{s=s_0+1}\left[\sum^{(s+1)^2-1}_{t=s^2+1}
        (\langle\mathbf{w}_s-\mathbf{w}^\ast,{\mathbf{x}}_t\rangle)^2
        -2\eta_t\langle \mathbf{w}_s-\mathbf{w}^\ast,{\mathbf{x}}_t\rangle\right],\\
        {\mathrm{Reg}_{[s_1:s_2]}(\mathbf{w}^\ast)}
        :=&\sum^{s_2}_{s=s_1+1}\sum^{(s+1)^2-1}_{t=s^2+1}
        \left[\ell(\hat{y}_t,y_t)-\ell(\langle\mathbf{w}^\ast,{\mathbf{x}}_t\rangle,y_t)\right]\\
        \leq&\sum^{s_2}_{s=s_1+1}\Xi_s+\sum^{s_2}_{s=s_1+1}\left[\sum^{(s+1)^2-1}_{t=s^2+1}
        (\langle\mathbf{w}_s-\mathbf{w}^\ast,{\mathbf{x}}_t\rangle)^2
        -2\eta_t\langle \mathbf{w}_s-\mathbf{w}^\ast,{\mathbf{x}}_t\rangle\right],\\
        {\mathrm{Reg}_{[s_2:s_T]}(\mathbf{w}^\ast)}
        :=&\sum^{s_T}_{s=s_2+1}\sum^{(s+1)^2-1}_{t=s^2+1}
        \left[\ell(\hat{y}_t,y_t)-\ell(\langle\mathbf{w}^\ast,{\mathbf{x}}_t\rangle,y_t)\right]\\
        \leq&\sum^{s_T}_{s=s_2+1}\Xi_s+\sum^{s_T}_{s=s_2+1}\left[\sum^{(s+1)^2-1}_{t=s^2+1}
        (\langle\mathbf{w}_s-\mathbf{w}^\ast,{\mathbf{x}}_t\rangle)^2
        -2\eta_t\langle \mathbf{w}_s-\mathbf{w}^\ast,{\mathbf{x}}_t\rangle\right].
    \end{align*}
    Summing over all results gives,
    with probability at least $1-T\delta$,
    \begin{align*}
        &\mathrm{Reg}(\mathbf{w}^\ast)=\\
        &\underbrace{\left[\sum^{s_0}_{s=1}+\sum^{s_1}_{s=s_0+1}+\sum^{s_2}_{s=s_1+1}+\sum^{s_T}_{s=s_2+1}\right]
        \frac{\Vert \bar{\mathbf{w}}_{s^2+1}(S_s)-\mathbf{w}_s\Vert^2_{\mathbf{A}_{s^2}}
        -\Vert \bar{\mathbf{w}}_{(s+1)^2}(S_s)-\mathbf{w}_s\Vert^2_{\mathbf{A}_{(s+1)^2-1}}}{2}}_{:=\mathrm{Reg}_1}+\\
        &\underbrace{\left(\sum^{s_0}_{s=1}+\sum^{s_1}_{s=s_0+1}+\sum^{s_2}_{s=s_1+1}+\sum^{s_T}_{s=s_2+1}\right)
        \sum^{(s+1)^2-1}_{t=s^2+1}\frac{\mathbf{g}^\top_t\mathbf{A}^{-1}_t\mathbf{g}_t}{2}}_{:=\mathrm{Reg}_2}+\\
        &\underbrace{\left(\sum^{s_0}_{s=1}+\sum^{s_1}_{s=s_0+1}+\sum^{s_2}_{s=s_1+1}+\sum^{s_T}_{s=s_2+1}\right)
        \sum^{(s+1)^2-1}_{t=s^2+1}
        (\langle\mathbf{w}_s-\mathbf{w}^\ast,{\mathbf{x}}_t\rangle)^2
        +\sum^{s_T}_{s=1}(\langle \hat{\mathbf{w}}_{s-1}(B_s)-\mathbf{w}^\ast,{\mathbf{x}}_{s^2}\rangle)^2}_{:=\mathrm{Reg}_3}+\\
        &\underbrace{-2\left(\sum^{s_0}_{s=1}+\sum^{s_1}_{s=s_0+1}+\sum^{s_2}_{s=s_1+1}+\sum^{s_T}_{s=s_2+1}\right)\sum^{(s+1)^2-1}_{t=s^2+1}
        \eta_t\langle \mathbf{w}_s-\mathbf{w}^\ast,{\mathbf{x}}_t\rangle
        -\sum^{s_T}_{s=1}2\eta_{s^2}\langle \hat{\mathbf{w}}_{s-1}(B_s)-\mathbf{w}^\ast,{\mathbf{x}}_{s^2}\rangle
        }_{:=\mathrm{Reg}_4}.
    \end{align*}

    We first analyze $\mathrm{Reg}_4$.
    By the concentration inequality of Gaussian variables,
    with probability at least $1-\delta$,
    \begin{align*}
        \mathrm{Reg}_4
        \leq2\sqrt{2\mathrm{Reg}_3\cdot\ln\frac{1}{\delta}}.
    \end{align*}
    Next we analyze $\mathrm{Reg}_3$.
    \begin{align*}
        \sum^{s_T}_{s=1}(\langle \hat{\mathbf{w}}_{s-1}(B_s)-\mathbf{w}^\ast,\mathbf{x}_{s^2}\rangle)^2
        \leq& s_T\cdot \Vert \hat{\mathbf{w}}_{s-1}(B_s)-\mathbf{w}^\ast\Vert^2_1\cdot \Vert \mathbf{x}_{s^2}\Vert^2_{\infty}
        \leq 4s_T,\\
        \sum^{s_0}_{s=1}\sum^{(s+1)^2-1}_{t=s^2+1}
        (\langle\mathbf{w}_s-\mathbf{w}^\ast,{\mathbf{x}}_t\rangle)^2\leq& 4s_0(s_0+1).
    \end{align*}
    Next we analyze the regret in the intervals
    $[(s_0+1)^2+1,(s_1+1)^2-1]$, $[s^2_1+1,(s_2+1)^2-1]$ and $[(s_2+1)^2+1,T]$.
    For any $s\geq 2$,
    if $S_s\neq S_{s-1}$,
    then we call such a round ``breakpoint''.
    Assuming that there are $n$ breakpoints in the set $\{2,3,\ldots, s_T\}$,
    denoted by $b_1,b_2,\ldots,b_n$.
    Specifically, for each $j=1,2,\ldots,n$,
    we have $S_{b_j}\neq S_{b_j-1}$.
    Without loss of generality,
    assuming that there are two breakpoints $b_r,b_m\in\{b_1,b_2,\ldots,b_n\}$ such that
    $$
        b_{r-1}< s_0+1< b_r,\quad b_{m-1}< s_1+1 < b_m.
    $$
    By Lemma \ref{lemma:ICML2025:convergence_S_s_to_S},
    under the condition that $\Delta_s(S)$ satisfies
    Lemma \ref{lemma:estimator_error:DS-OSLRC} for any $s\geq 1$,
    it must be
    $$
        \forall s\geq s_2+1,\quad S_s=S_{s+1}=\ldots=S_{s_T}=S,
    $$
    which means there is no breakpoints in the interval $[s_2+2,s_T]$
    and
    \begin{equation}
    \label{eq:ICML2025:upper_bound_restart}
        b_n\leq s_2+1,\quad n\leq s_2.
    \end{equation}
    Next we define the competitor $\mathbf{w}_s$ as follows.
    \begin{equation}
    \label{eq:ICML2025:definition:w_s}
    \left\{
    \begin{split}
        \forall s\in \{s_0+1,\ldots,b_r-1\}, \quad
        \mathbf{w}_s=&\mathbf{w}^\ast(S_{b_r-1}\cap S):=\mathbf{w}^{\ast}_{b_r-1},\\
        \forall s\in \{b_{\tau},b_{\tau+1}-1\},r\leq\tau\leq n-2,\quad
        \mathbf{w}_s=&\mathbf{w}^\ast(S_{b_{\tau+1}-1}\cap S):=\mathbf{w}^{\ast}_{b_{\tau+1}-1},\\
        \forall s\in \{b_{n-1},\ldots,s_2\},\quad
        \mathbf{w}_s=&\mathbf{w}^\ast(S_{s_2}\cap S):=\mathbf{w}^{\ast}_{s_2},\\
        \forall s\in \{b_n,\ldots,s_T\},\quad\mathbf{w}_s=&\mathbf{w}^\ast.
    \end{split}
    \right.
    \end{equation}
    By Lemma \ref{lemma:COLT2025:approximation_error_w_s},
    \eqref{eq:ICML2025:bounding_Delta(S_s)_by_Delta_s:ours}
    and \eqref{eq:ICML2025:upper_bound_restart},
    \begin{align*}
        &\mathrm{Reg}_3-\sum^{s_T}_{s=1}(\langle \hat{\mathbf{w}}_{s-1}(B_s)-\mathbf{w}^\ast,\mathbf{x}_{s^2}\rangle)^2
            -\sum^{s_0}_{s=1}\sum_{t\in\mathcal{T}_s}(\langle\mathbf{w}_s-\mathbf{w}^\ast,{\mathbf{x}}_t\rangle)^2\\
        \leq&\left(\sum^{b_r-1}_{s=s_0+1}+
            \sum^{n-1}_{\tau=r}\sum^{b_{\tau+1}-1}_{s=b_{\tau}}+
            \sum^{s_T}_{s=b_n}\right)\sum^{(s+1)^2-1}_{t=s^2+1}
            \Vert\mathbf{w}_s-\mathbf{w}^\ast\Vert^2_1\cdot \Vert {\mathbf{x}}_t\Vert^2_{\infty}\\
        \leq&\left(\sum^{b_r-1}_{s=s_0+1}+
            \sum^{n-2}_{\tau=r}\sum^{b_{\tau+1}-1}_{s=b_{\tau}}+
            \sum^{s_2}_{s=b_{n-1}}\right)\sum^{(s+1)^2-1}_{t=s^2+1}
            \Vert\mathbf{w}_s-\mathbf{w}^\ast\Vert^2_1\cdot \Vert {\mathbf{x}}_t\Vert^2_{\infty}\\
        \leq&2\left(\sum^{b_r-1}_{s=s_0+1}
            \Vert \mathbf{w}^\ast_{b_r-1}-\mathbf{w}^\ast\Vert^2_1+
            \sum^{n-2}_{\tau=r}\sum^{b_{\tau+1}-1}_{s=b_{\tau}}
            \Vert \mathbf{w}^\ast_{b_{\tau+1}-1}-\mathbf{w}^\ast\Vert^2_1+
            \sum^{s_2}_{s=b_{n-1}}\Vert \mathbf{w}^\ast_{s_2}-\mathbf{w}^\ast\Vert^2_1\right)\cdot s\\
        \leq&2\left(\sum^{b_r-1}_{s=s_0+1}
            \Vert \hat{\mathbf{w}}_{b_r-1}-\mathbf{w}^\ast\Vert^2_1+
            \sum^{n-2}_{\tau=r}\sum^{b_{\tau+1}-1}_{s=b_{\tau}}
            \Vert \hat{\mathbf{w}}_{b_{\tau+1}-1}-\mathbf{w}^\ast\Vert^2_1+
            \sum^{s_2}_{s=b_{n-1}}
            \Vert \hat{\mathbf{w}}_{s_2}-\mathbf{w}^\ast\Vert^2_1\right)\cdot s\\
        \leq&8\left(\sum^{b_r-1}_{s=s_0+1}\Vert \Delta_{b_r-1}(S)\Vert^2_1+
            \sum^{n-2}_{\tau=r}\sum^{b_{\tau+1}-1}_{s=b_{\tau}}\Vert \Delta_{b_{\tau+1}-1}(S)\Vert^2_1+
            \sum^{s_2}_{s=b_{n-1}}\Vert \Delta_{s_2}(S)\Vert^2_1\right)\cdot s\\
        \leq&8\sum^{s_1}_{s=s_0+1}\Vert \Delta_s(S)\Vert^2_1\cdot s +
            8\sum^{s_2}_{s=s_1+1}\Vert \Delta_s(S)\Vert^2_1\cdot s
            \qquad\qquad\qquad\qquad\qquad\qquad(\mathrm{by}~\mathrm{Lemma}~\ref{lemma:estimator_error:DS-OSLRC})\\
        \leq&22\frac{a^2_4}{\delta^8_S}k^4g^2_{d,k}\ln\frac{s_1}{s_0}
            +8\frac{(\mu_2a_5)^2}{\delta^4_S}\frac{k^2(d-1)}{k-1}(s_2-s_1).
    \end{align*}
    Summing all results gives
    \begin{align*}
        \mathrm{Reg}_3
        \leq& 4s_T+4s_0(s_0+1)
        +22\frac{a^2_4}{\delta^8_S}k^4g^2_{d,k}\ln\frac{s_1}{s_0}
        +8\cdot\frac{(\mu_2a_5)^2}{\delta^4_S}\frac{k^2(d-1)}{k-1}(s_2-s_1)\\
        \leq&4\sqrt{T}+22\frac{a^2_4}{\delta^8_S}k^4g^2_{d,k}\ln\frac{(d-2)\ln\frac{d}{\delta}}{k-2}
        +2\frac{(2\mu_2a_5)^4}{\delta^8_S}
        \cdot\frac{k^4(d-1)^2}{\min_{i\in S}\vert w^\ast_i\vert^2(k-1)^2}.
    \end{align*}
    Now we begin to analyze $\mathrm{Reg}_1$.
    For $s\in[1,b_1-1]$,
    we define
    $$
        \mathbf{w}_s=\mathbf{w}^{\ast}(S_{b_1-1}\cap S):=\mathbf{w}^{\ast}_{b_1-1}.
    $$
    Recalling \eqref{eq:ICML2025:definition:w_s},
    for $\tau\in[b_{\tau},b_{\tau+1}-1]$,
    in which $1\leq \tau\leq n-1$,
    we define $\mathbf{w}_s=\mathbf{w}^\ast_{b_{\tau+1}-1}$.
    For $s\in [b_n,s_T]$, $\mathbf{w}_s=\mathbf{w}^\ast$.
    \begin{align*}
        \mathrm{Reg}_1
        =&\sum^{b_1-1}_{s=1}
        \left[
        \frac{\Vert \bar{\mathbf{w}}_{s^2+1}(S_s)-\mathbf{w}^{\ast}_{b_1-1}\Vert^2_{\mathbf{A}_{s^2}}}{2}
        -\frac{\Vert \bar{\mathbf{w}}_{(s+1)^2}(S_s)-\mathbf{w}^{\ast}_{b_1-1}\Vert^2_{\mathbf{A}_{(s+1)^2-1}}}{2}\right]+\\
        &\sum^{n-1}_{r=1}\sum^{b_{r+1}-1}_{s=b_r}
        \frac{\Vert \bar{\mathbf{w}}_{s^2+1}(S_s)-\mathbf{w}^{\ast}_{b_{r+1}-1}\Vert^2_{\mathbf{A}_{s^2}}}{2}
        -\frac{\Vert \bar{\mathbf{w}}_{(s+1)^2}(S_s)-\mathbf{w}^{\ast}_{b_{r+1}-1}
        \Vert^2_{\mathbf{A}_{(s+1)^2-1}}}{2}+\\
        &\sum^{s_T}_{s=b_n}
        \left[
        \frac{\Vert \bar{\mathbf{w}}_{s^2+1}(S_s)-\mathbf{w}^{\ast}\Vert^2_{\mathbf{A}_{s^2}}}{2}
        -\frac{\Vert \bar{\mathbf{w}}_{(s+1)^2}(S_s)-\mathbf{w}^{\ast}\Vert^2_{\mathbf{A}_{(s+1)^2-1}}}{2}\right]+\\
        \leq&\frac{\Vert \bar{\mathbf{w}}_2(S_1)-\hat{\mathbf{w}}^\ast_{b_1-1}\Vert^2_{\mathbf{A}_1}}{2}+
        \sum^{n-1}_{r=1}
        \frac{\Vert \bar{\mathbf{w}}_{(b_r)^2+1}(S_{b_r})-\hat{\mathbf{w}}^{\ast}_{b_{r+1}-1}
        \Vert^2_{\mathbf{A}_{(b_r)^2}}}{2}+\\
        &\frac{\Vert \bar{\mathbf{w}}_{(b_n)^2+1}(S_s)-\mathbf{w}^\ast\Vert^2_{\mathbf{A}_{(b_n)^2}}}{2}\\
        \leq&4\varepsilon+2(n-1)\varepsilon,
    \end{align*}
    in which
    $\bar{\mathbf{w}}_{s^2+1}(S_s)=\bar{\mathbf{w}}_{s^2}(S_{s-1})$ and $\mathbf{A}_{s^2}=\mathbf{A}_{s^2-1}$
    for all $s$ such that $S_s=S_{s-1}$,
    and $\mathbf{A}_{s^2}=\varepsilon\cdot\mathbf{I}_{k\times k}$
    for $s=1,b_1,b_2,\ldots,b_n$.

    Finally, we analyze $\mathrm{Reg}_2$.
    By Lemma \ref{lemma:ICML2025:technical_lemma:ONS},
    \begin{align*}
        \mathrm{Reg}_2
        =&\sum^{b_1-1}_{s=1}\sum^{(s+1)^2-1}_{t=s^2+1}\frac{\mathbf{g}^\top_t\mathbf{A}^{-1}_t\mathbf{g}_t}{2}
        +\sum^{n-1}_{r=1}\sum^{b_{r+1}-1}_{s=b_r}\sum^{(s+1)^2-1}_{t=s^2+1}\frac{\mathbf{g}^\top_t\mathbf{A}^{-1}_t\mathbf{g}_t}{2}
        +\sum^{s_T}_{s=b_n}\sum^{(s+1)^2-1}_{t=s^2+1}\frac{\mathbf{g}^\top_t\mathbf{A}^{-1}_t\mathbf{g}_t}{2}\\
        \leq&\frac{nk}{2}\ln\left(\frac{4(1+Y_{\delta})^2kb^2_n}{\varepsilon}+1\right)
        +\frac{k}{2}\ln\left(\frac{4(1+Y_{\delta})^2kT}{\varepsilon}+1\right),
    \end{align*}
    in which $\Vert \mathbf{g}_t\Vert_2\leq
    2\left(1+Y_{\delta}\right)\Vert {\mathbf{x}}_t(S_s)\Vert_2\leq
    2\left(1+Y_{\delta}\right)\sqrt{k}$.
    By \eqref{eq:ICML2025:upper_bound_restart},
    $$
        \mathrm{Reg}_1+\mathrm{Reg}_2
        \leq 2s_2\varepsilon
        +\frac{s_2}{2}k\ln\left(\frac{4(1+Y_{\delta})^2k(s_2+1)^2}{\varepsilon}+1\right)
        +\frac{k}{2}\ln\left(\frac{4(1+Y_{\delta})^2kT}{\varepsilon}+1\right).
    $$
    Let $\varepsilon=k$.
    Summing all results gives,
    with probability at least $1-(T+\sqrt{T}(6+\log_{1.5}\frac{2(d-2)}{3(k-2)})+1)\delta$,
    we have
    \begin{align*}
        \mathrm{Reg}(\mathbf{w}^\ast)
        \leq&4\sqrt{T}+2ks_2
        +\frac{s_2k}{2}\ln\left(4(1+Y_{\delta})^2s^2_2+1\right)
        +\frac{k}{2}\ln\left(4(1+Y_{\delta})^2T+1\right)\\
        &\frac{22a^2_4}{\delta^8_S}k^4g^2_{d,k}\ln\frac{(d-2)\ln\frac{d}{\delta}}{k-2}
        +\frac{2(2\mu_2a_5)^4}{\delta^8_S}
        \frac{k^4(d-1)^2}{\min_{i\in S}\vert w^\ast_i\vert^2(k-1)^2}
        +2\sqrt{2\mathrm{Reg}_3\ln\frac{1}{\delta}}.
    \end{align*}
    Omitting the lower order terms concludes the proof.
\end{proof}

\section{Proof of Theorem \ref{thm:ICML2025:regret_bound_DS-POSLRC}}

    We first prove a technical lemma similar to Lemma \ref{lemma:COLT2025:approximation_error_w_s}.
    \begin{Mylemma}
    \label{lemma:COLT2025:approximation_error_hat_w(S_s)}
        For any $s\geq 1$,
        let $\hat{\mathbf{w}}_s$ be the solution of $\mathrm{DS}(\hat{\gamma}_s)$.
        Let $S_s\subseteq[d]$ satisfy
        $\vert S_s\vert=k$ and for any $i\in S_s$
        and $j\in [d]\setminus S_s$, $\vert \hat{w}_{s,i}\vert\geq \vert \hat{w}_{s,j}\vert$.
        If $\hat{\gamma}_{\tau}\geq \gamma_{\tau}$ for all $\tau\leq s$,
        then w.p. at least $1-s\left(5+\log_{1.5}\frac{2(d'-2)}{3(k_0-2)}\right)\delta$,
        $$
            \forall \tau\leq s,\quad
            \Vert \hat{\mathbf{w}}_{\tau}(S_{\tau}) - \mathbf{w}^\ast\Vert_1
            \leq 3\Vert \hat{\mathbf{w}}_{\tau}(S) - \mathbf{w}^\ast\Vert_1.
        $$
    \end{Mylemma}

    \begin{proof}[of Lemma \ref{lemma:COLT2025:approximation_error_hat_w(S_s)}]
        Unfolding $\Vert \mathbf{w}^{\ast}_{\tau}(S_{\tau}) - \mathbf{w}^\ast\Vert_1$ gives
        \begin{align*}
            \Vert \hat{\mathbf{w}}_{\tau}(S_{\tau}) - \mathbf{w}^\ast\Vert_1
            =&\sum_{i\in S\cap S_{\tau}}\vert \hat{w}_{\tau,i}-w^\ast_i\vert
            +\sum_{i\in S_{\tau}\setminus S}\vert \hat{w}_{\tau,i}\vert+
            \sum_{i\in S\setminus S_{\tau}} \vert w^{\ast}_i\vert\\
            =&\sum_{i\in S\cap S_{\tau}}\vert \hat{w}_{\tau,i}-w^\ast_i\vert+
            \sum_{i\in S_{\tau}\setminus S}\vert \hat{w}_{\tau,i}\vert+
            \sum_{i\in S\setminus S_{\tau}} \vert w^{\ast}_i-\hat{w}_{\tau,i}+\hat{w}_{\tau,i}\vert\\
            \leq&\sum_{i\in S\cap S_{\tau}}\vert \hat{w}_{\tau,i}-w^\ast_i\vert+
            \sum_{i\in S_{\tau}\setminus S}\vert \hat{w}_{\tau,i}\vert+
            \sum_{i\in S\setminus S_{\tau}} \vert w^{\ast}_i-\hat{w}_{\tau,i}\vert
            +\sum_{i\in S\setminus S_{\tau}}\vert\hat{w}_{\tau,i}\vert\\
            \leq&\Vert \hat{\mathbf{w}}_{\tau}(S) - \mathbf{w}^\ast\Vert_1+
            2\sum_{i\in S_{\tau}\setminus S}\vert \hat{w}_{\tau,i}\vert\\
            \leq& 3\Vert \hat{\mathbf{w}}_{\tau}(S) - \mathbf{w}^\ast\Vert_1,
        \end{align*}
        where the last inequality comes from Lemma \ref{lemma:ICML25:Dantzig2005}.
        We conclude the proof.
    \end{proof}

\begin{proof}[of Theorem \ref{thm:ICML2025:regret_bound_DS-POSLRC}]
    For simplicity,
    we will alternately use the notation
    $\hat{y}_s=\langle\hat{\mathbf{w}}_{s-1}(S_{s-1}),{\mathbf{x}}_s(S_{s-1})\rangle$.
    With probability at least $1-(T(6+\log_{1.5}\frac{d+k}{k-2})+1)\delta$,
    the regret of DS-POSLRC satisfies
    \begin{align*}
        &\mathrm{Reg}(\mathbf{w}^\ast)\\
        =&\left(\sum^{s_0}_{s=1}+\sum^{s_1}_{s=s_0+1}+
        \sum^T_{s=s_1+1}\right)
        \left[\ell\left(\langle \hat{\mathbf{w}}_{s-1}(S_{s-1}),\mathbf{x}_s(S_{s-1})\rangle,y_s\right)
        -\ell\left(\langle\mathbf{w}^\ast,{\mathbf{x}}_s\rangle,y_s\right)\right]\\
        =&\left(\sum^{s_0}_{s=1}+\sum^{s_1}_{s=s_0+1}+
        \sum^T_{s=s_1+1}\right)
        \left[(\langle \hat{\mathbf{w}}_{s-1}(S_{s-1})-\mathbf{w}^\ast,{\mathbf{x}}_s\rangle)^2
        -2\eta_t\langle \hat{\mathbf{w}}_{s-1}(S_{s-1})-\mathbf{w}^\ast,{\mathbf{x}}_s\rangle\right]\\
        \leq&\left(\sum^{s_0}_{s=1}+\sum^{s_1}_{s=s_0+1}+\sum^T_{s=s_1+1}\right)
        \Vert\hat{\mathbf{w}}_{s-1}(S_{s-1})-\mathbf{w}^\ast\Vert^2_1+
        2\sqrt{2\sum^T_{s=1}\Vert\hat{\mathbf{w}}_{s-1}(S_{s-1})-\mathbf{w}^\ast\Vert^2_1\cdot\ln\frac{1}{\delta}}\\
        \leq&\frac{48^2k^2g'_{d',k_0}}{\delta^4_S}\ln\frac{d^2}{\delta}
        +\frac{2.7a^2_4k^2g_{d',k_0}}{64\delta^4_S\ln\frac{d^2}{\delta}}
        +9\mu^2_2a^2_5\frac{k^2(d'-1)}{\delta^4_S(k_0-1)}\ln\frac{T}{s_1+1}+\\
        &2\sqrt{2\left(\frac{48^2k^2g'_{d',k_0}}{\delta^4_S}\ln\frac{d^2}{\delta}
        +\frac{2.7a^2_4k^2g_{d',k_0}}{64\delta^4_S\ln\frac{d^2}{\delta}}
        +9\mu^2_2a^2_5\frac{k^2(d'-1)}{\delta^4_S(k_0-1)}\ln\frac{T}{s_1+1}\right)\ln\frac{1}{\delta}},
    \end{align*}
    in which by
    Lemma \ref{lemma:estimator_error:DS-POSLRC} and
    Lemma \ref{lemma:COLT2025:approximation_error_hat_w(S_s)},
    we have
    \begin{align*}
        &\left(\sum^{s_0}_{s=1}+\sum^{s_1}_{s=s_0+1}+\sum^T_{s=s_1+1}\right)
        \Vert\hat{\mathbf{w}}_{s-1}(S_{s-1})-\mathbf{w}^\ast\Vert^2_1\\
        \leq&4s_0+9\left(\sum^{s_1}_{s=s_0+1}+\sum^T_{s=s_1+1}\right)
        \Vert\hat{\mathbf{w}}_{s-1}(S)-\mathbf{w}^\ast\Vert^2_1\\
        \leq&4s_0+
        9\left(\frac{9}{9-2\sqrt{3}}\right)^2\frac{a^2_4k^4g^2_{d',k_0}}{\delta^8_S}
        \left(\frac{1}{s_0}-\frac{1}{s_1}\right)+
        9\mu^2_2a^2_5\frac{k^2(d'-1)}{\delta^4_S(k_0-1)}\ln\frac{T}{s_1+1}\\
        \leq&4s_0+\frac{2.7a^2_4k^2g_{d',k_0}}{64\delta^4_S\ln\frac{d^2}{\delta}}
        +9\mu^2_2a^2_5\frac{k^2(d'-1)}{\delta^4_S(k_0-1)}\ln\frac{T}{s_1+1}.
    \end{align*}
    Omitting the lower order terms concludes the proof.
\end{proof}

\end{document}